\documentclass[11pt]{article}
\usepackage{titlesec}

\usepackage[authoryear,round]{natbib}
\usepackage[T1]{fontenc}
\usepackage{amsmath,amsfonts,amsthm}
\usepackage{newpxtext}
\usepackage{newpxmath}

\usepackage[utf8]{inputenc}
\usepackage[T1]{fontenc}
\usepackage{textcomp}
\usepackage{microtype}

\usepackage{amsmath}
\usepackage{amsfonts}
\usepackage{amsthm}
\usepackage{mathtools}
\usepackage{bm}
\usepackage{nicefrac}
\usepackage{cases}
\usepackage{siunitx}

\PassOptionsToPackage{table}{xcolor}
\usepackage{xcolor}
\usepackage{graphicx}
\usepackage{svg}
\usepackage{wrapfig}
\usepackage{subcaption}
\usepackage{booktabs}
\usepackage{tabularx}
\usepackage{multirow}
\usepackage{makecell}
\usepackage{threeparttable}
\usepackage{colortbl}

\usepackage[ruled,boxed,commentsnumbered]{algorithm2e}
\usepackage{algorithmicx}
\usepackage{algpseudocode}
\usepackage[noend]{algcompatible}
\usepackage[edges]{forest}
\usepackage{enumitem}
\usepackage{fancyvrb}
\usepackage{verbatim}
\usepackage{tcolorbox}

\usepackage{xurl}
\usepackage{hyperref}
\usepackage{bibentry}
\usepackage{xspace}
\usepackage{etoolbox}
\usepackage{ragged2e}
\usepackage{pifont}
\usepackage{utfsym}
\usepackage{fontawesome5}
\usepackage{tikzsymbols}
\usepackage{lineno}

\usepackage{titletoc}

\definecolor{hiddendraw}{RGB}{20,68,106}
\definecolor{hidden-pink}{RGB}{255,245,247}
\definecolor{hidden-red}{RGB}{180,0,0}
\definecolor{output-black}{RGB}{0,0,0}
\definecolor{output-white}{RGB}{255,255,255}
\definecolor{myorange}{RGB}{255,208,153}
\definecolor{mygreen}{RGB}{166,207,152}
\definecolor{forestgreen}{RGB}{34,139,34}
\definecolor{darkgreen}{RGB}{24,119,24}
\definecolor{melon}{HTML}{F89E7B}
\definecolor{lightorange}{HTML}{FBB982}
\definecolor{userblack}{HTML}{41B0E4}
\definecolor{cerulean}{HTML}{00A2E3}
\definecolor{mygrey}{RGB}{128,128,128}
\definecolor{verylightgray}{RGB}{240,240,240}
\definecolor{titlebgcolor}{RGB}{92,54,117}
\definecolor{textbgcolor}{RGB}{180,147,69}
\definecolor{mypink}{rgb}{1,0.928,0.928}
\definecolor{mydarkpink}{rgb}{1,0.7,0.7}
\definecolor{lightblue}{RGB}{57,139,250}

\newtheorem{proposition}{Proposition}

\newtheorem{lemma}{Lemma}

\newtcolorbox{warningbox}[2][]{
  colframe=red!25,
  colback=red!10,
  coltitle=red!20!black,
  title=#2,
  #1,
}

\newtcolorbox{hintbox}[2][]{
  colframe=green!25,
  colback=green!10,
  coltitle=green!20!black,
  #1,
}

\newtcolorbox{citebox}[2][]{
  colframe=mygrey!55,
  colback=mygrey!10,
  #1,
}

\newtcolorbox{infobox}[2][]{
  colframe=titlebgcolor!28,
  colback=textbgcolor!7,
  boxrule=0.45pt,
  arc=2pt,
  top=7pt,
  bottom=7pt,
  left=8pt,
  right=8pt,
  #1,
}

\newtcolorbox{methbox}[2][]{
  colframe=titlebgcolor!35,
  colback=textbgcolor!7,
  fonttitle=\sffamily\bfseries\small,
  title={#2},
  boxrule=0.45pt,
  arc=2pt,
  top=8pt,
  bottom=8pt,
  left=9pt,
  right=9pt,
  colbacktitle=titlebgcolor!8,
  coltitle=black,
  #1,
}

\usepackage{placeins}
\usepackage{float}
\definecolor{cordCiteBlue}{HTML}{37627E}
\definecolor{cordLinkRed}{HTML}{B82E38}
\hypersetup{
  colorlinks=true,
  citecolor=cordCiteBlue,
  linkcolor=cordLinkRed,
  urlcolor=cordLinkRed
}
\usepackage[
  letterpaper,
  left=1in,
  right=1in,
  top=0.85in,
  bottom=0.9in,
  headheight=14pt,
  headsep=14pt,
  footskip=24pt
]{geometry}
\usepackage{fancyhdr}
\tcbuselibrary{skins,breakable}

\colorlet{cordPurple}{titlebgcolor}
\colorlet{cordGold}{textbgcolor}

\titleformat{\section}
  {\sffamily\bfseries\fontsize{14}{17}\selectfont}
  {\thesection}{0.65em}{}
\titleformat{\subsection}
  {\sffamily\bfseries\fontsize{11.5}{14}\selectfont}
  {\thesubsection}{0.65em}{}
\titleformat{\subsubsection}
  {\sffamily\bfseries\normalsize}
  {\thesubsubsection}{0.65em}{}
\titleformat{\paragraph}[runin]
  {\normalfont\bfseries}{\theparagraph}{0.65em}{}
\titlespacing*{\section}{0pt}{14pt plus 3pt minus 2pt}{5pt}
\titlespacing*{\subsection}{0pt}{10pt plus 2pt minus 1pt}{4pt}
\titlespacing*{\subsubsection}{0pt}{8pt plus 2pt minus 1pt}{3pt}
\titlespacing*{\paragraph}{0pt}{6pt plus 1pt minus 1pt}{0.65em}

\setlist{topsep=4pt,itemsep=2pt,parsep=0pt}
\setcitestyle{authoryear,round,semicolon}
\apptocmd{\thebibliography}{\raggedright}{}{%
  \PackageWarning{cord-report}{Could not set bibliography alignment}%
}

\fancypagestyle{plain}{%
  \fancyhf{}
  \fancyfoot[C]{\sffamily\footnotesize\thepage}

}

\makeatletter
\renewcommand{\@maketitle}{%
  \cordaffiliationlogos
  \vskip 12pt
  \noindent
  {\sffamily\bfseries\fontsize{23}{27}\selectfont\@title\par}%
  \vskip 13pt
  \noindent
  {\sffamily\bfseries\fontsize{11.5}{16}\selectfont\@author\par}%
  \vskip 16pt
}
\makeatother

\newtcolorbox{cordabstractbox}{
  enhanced,
  breakable,
  colback=cordGold!4!white,
  colframe=cordPurple!24!white,
  coltext=black,
  boxrule=0.5pt,
  arc=6pt,
  outer arc=6.5pt,
  boxsep=0pt,
  left=14pt,
  right=14pt,
  top=11pt,
  bottom=11pt,
  before skip=0pt,
  after skip=12pt
}
\renewenvironment{abstract}{%
  \begin{cordabstractbox}%
  \small
  \setlength{\parskip}{0pt}%
  \noindent{\sffamily\bfseries\normalsize Abstract}\par\vspace{5pt}
  \ignorespaces
}{%
  \par\end{cordabstractbox}%
}

\providecommand{\R}{\mathbb{R}}
\definecolor{AppendixTOCAccent}{HTML}{1F6278}
\newcommand{\cordappendixcontents}{%
  \begingroup
  \thispagestyle{plain}
  \setlength{\parskip}{0pt}
  \renewcommand{\baselinestretch}{1}
  \vspace*{12pt}
  \phantomsection
  \pdfbookmark[0]{Appendix Contents}{appendix-contents}
  {\centering\sffamily\bfseries\fontsize{18}{22}\selectfont
    Appendix Contents\par}
  \vspace{14pt}
  \hypersetup{linkcolor=cordLinkRed}
  \titlecontents{section}[1.8em]
    {\addvspace{14pt}\sffamily\bfseries\fontsize{12}{16}\selectfont}
    {\contentslabel[\textcolor{AppendixTOCAccent}{\thecontentslabel}]{1.8em}}{}
    {\titlerule*[0.65pc]{.}\contentspage}
    [\addvspace{4pt}]
  \titlecontents{subsection}[4em]
    {\addvspace{6pt}\normalfont\fontsize{11}{17}\selectfont}
    {\contentslabel[\textcolor{AppendixTOCAccent}{\thecontentslabel}]{2.4em}}{}
    {\titlerule*[0.6pc]{.}\contentspage}
  \setcounter{tocdepth}{2}
  \printcontents[appendices]{}{1}{}
  \endgroup
}

\newcommand{\Gobj}[1]{\bm{G}_{#1}} %
\newcommand{\gobj}[2]{\bm{g}_{#1}^{(#2)}} %
\newcommand{\Gr}{\Gobj{r}}         %
\newcommand{\Go}{\Gobj{o}}         %
\newcommand{\dsum}{\bm{d}_{\mathrm{sum}}}
\newcommand{\dcord}{\bm{d}_{\mathrm{CORD}}}
\newcommand{\gfinal}{\bm{g}_{\mathrm{final}}}
\newcommand{\Lgrpo}{\mathcal{L}_{\mathrm{GRPO}}}
\newcommand{\Lopd}{\mathcal{L}_{\mathrm{OPD}}}

\newcommand{\inner}[2]{\langle #1, #2 \rangle}
\newcommand{\norm}[1]{\lVert #1 \rVert}

\newcommand{\stdv}[1]{\ensuremath{\,\pm #1}}
\newcommand{\best}[1]{{\bfseries\boldmath #1}}

\newcommand{\corddeltaup}[1]{\textcolor{red!70!black}{\ensuremath{+#1}}}
\newcommand{\corddeltadown}[1]{\textcolor{darkgreen}{\ensuremath{-#1}}}

\newcommand{\cordaffiliationlogos}{%
  \noindent
  \raisebox{-0.5\height}{\cordcuhklogo{28pt}}%
  \hspace{22pt}%
  \raisebox{-0.5\height}{\cordtencentlogo{16pt}}%
  \hspace{22pt}%
  \raisebox{-0.5\height}{\cordpkulogo{28pt}}%
  \par
}

\title{CORD: Conflict-Aware Optimization\\
for Joint RL and Distillation in Agents}

\author{%
\begin{tabular}{@{}l@{}}
Wenhao Yu\textsuperscript{1},\enspace
Zhengcheng Wang\textsuperscript{2},\enspace
Jiahong Liu\textsuperscript{1},\enspace
Zexuan Qiu\textsuperscript{1},\enspace
Minda Hu\textsuperscript{1},\\[1pt]
Shaohang Wei\textsuperscript{3},\enspace
Jiafei Lyu\textsuperscript{2},\enspace
Junyan Lyu\textsuperscript{2},\enspace
Haobo Fu\textsuperscript{2},\enspace
Irwin King\textsuperscript{1} \\[6pt]
\normalfont\fontsize{10.5}{14}\selectfont
\textsuperscript{1}The Chinese University of Hong Kong
\quad \textsuperscript{2}Tencent \quad \textsuperscript{3}Peking University \\[3pt]
\normalfont\fontsize{10.5}{14}\selectfont
Email: \texttt{yuwenhao117@gmail.com}, \texttt{jensencwang@tencent.com}
\end{tabular}%
}

\usepackage{pgf}
\makeatletter
\providecommand{\cord@vectorlogo}[4]{%
  \resizebox{!}{#1}{%
    \begin{pgfpicture}%
      \pgfpathrectangle{\pgfpointorigin}{\pgfpoint{#2bp}{#3bp}}%
      \pgfusepath{use as bounding box}%
      \pgfsys@invoke{q 0 0 #2 #3 re W n #4 Q}%
    \end{pgfpicture}%
  }%
}
\providecommand{\cordcuhklogo}[1]{\cord@vectorlogo{#1}{247}{165}{\cord@cuhkpath}}
\providecommand{\cordtencentlogo}[1]{\cord@vectorlogo{#1}{109}{24}{\cord@tencentpath}}
\providecommand{\cordpkulogo}[1]{\cord@vectorlogo{#1}{80}{80}{\cord@pkupath}}

\def\cord@cuhkpath{%
q 0.1 0 0 0.1 0 0 cm 0.470703 0.141113 0.443115 rg 1206.99 1650 m 910.941 1650 694.879 1601.64
684.262 1599.26 c 669.57 1597.03 661.289 1583.75 661.289 1569.06 c 661.289 1548.95 667.109 1255
668.551 1159.22 c 671.129 991.41 688.629 853.32 740.039 736.41 c 694.879 742.422 627.422 688.949
632.539 647.699 c 634.34 633.52 649.262 627.891 659.449 619.879 c 684.18 600.699 711.641 575.27
730.309 547.031 c 724.648 485.082 708.871 430.082 675.621 376.559 c 671.289 369.57 667.969 364.922
665.898 362.27 c 624.18 368.91 583.441 373.551 547.621 373.551 c 486.879 373.551 433.59 362.809
395.512 352.148 c 357.191 341.371 329.02 328.121 y 332.77 296.641 333.051 262.73 v 314.141 266.879
268.59 274.648 192.699 274.648 c 91.4492 274.648 14.3398 250.98 5.85938 241.449 c -1.98828 232.578
0.308594 224.691 0.308594 214.34 c 0.308594 204.961 5.35156 56.4805 6.08984 54.5703 c 6.08984 8.19922
207.699 3.23828 224.691 3.23828 c 390.352 3.23828 560.699 39.1406 703.051 126.09 c 722.852 138.199
735.352 147.301 740.309 151.051 c 791.129 144.801 839.73 137.969 880.621 131.289 c 1000.74 111.879
1122.93 95.0781 1238.98 95.0781 c 1419.38 95.0781 1627.66 145.199 1695.74 161.879 c 1725.31 168.98 l
1768.32 129.02 1834.92 80.5508 1888.44 60.9805 c 1950.9 38.0898 2039.69 15.3516 2150.94 15.3516 c
2262.42 15.3516 2359.65 46.6406 2390.16 58.4805 c 2412.23 67.1094 2461.09 90.4297 2462.3 119.18 c
2466.17 139.531 2454.22 167.969 2448.44 187.539 c 2442.23 208.711 2442.73 240.781 2443.75 259.141 c
2444.45 271.641 2448.05 294.102 2440.59 305 c 2432.07 317.301 2388.95 320.699 2321.37 313.281 c
2285.9 309.41 2233.83 296.449 2190.59 284.34 c 2159.18 275.512 2132.46 267.23 2120.35 263.441 c
2116.45 287.34 2114.65 311.719 2120.2 338.199 c 2123.63 354.41 2123.79 365.898 2111.17 374.73 c
2100.94 381.879 2065.04 398.59 1973.4 398.59 c 1921.02 398.59 1865.39 391.211 1807.3 380.039 c
1794.77 377.621 l 1746.45 422.23 1721.41 489.219 1698.28 548.98 c 1724.77 596.559 1755.04 631.449
1800.74 672.578 c 1786.91 710.82 1740.39 726.68 1704.18 731.68 c 1694.34 733.051 1684.18 731.762
1675.31 729.691 c 1736.41 860.512 1762.03 1030.74 1768.12 1242.19 c 1770.86 1337.89 1760.47 1553.91
1760.47 1574.02 c 1760.47 1588.67 1752.42 1601.56 1737.93 1604.18 c 1726.8 1607.7 1592.34 1647.89
1206.99 1650 c h 732.578 721.328 m 738.789 721.371 744.379 720 748.949 716.988 c 774.77 663.441
808.16 614.449 851.129 569.02 c 843.32 570.352 835.352 570.898 827.578 570.66 c 801.988 569.879
775.43 545.039 758.98 526.559 c 741.719 568.398 704.059 605.16 669.059 632.301 c 667.77 633.32 651.25
644.73 648.441 647.852 c 641.801 676.129 697.031 721.129 732.578 721.328 c h 1695.23 715.66 m 1727.81
716.02 1765.59 701.328 1781.64 676.801 c 1737.34 635.859 1702.27 592.301 1674.77 538.359 c 1654.26
556.02 1626.84 572.301 1599.02 572.301 c 1583.95 572.301 1570.27 569.488 1558.67 563.48 c 1576.64
581.879 1594.06 601.602 1610.43 622.699 c 1631.37 649.488 1649.77 678.789 1665.9 710.43 c 1669.65
711.762 l 1677.38 714.301 1686.09 715.59 1695.23 715.66 c h 1599.02 556.602 m 1630.47 556.602 1659.02
532.891 1679.45 511.801 c 1682.5 518.359 1685.51 524.57 1688.55 530.66 c 1709.45 477.148 1732.54
422.66 1771.17 379.609 c 1773.05 377.5 1774.84 375.621 1776.48 373.91 c 1722.46 362.699 1668.79
349.609 1615.08 336.801 c 1599.18 355.16 1537.19 430 1537.19 486.801 c 1537.19 539.059 1563.09
556.602 1599.02 556.602 c h 831.602 555.039 m 849.488 555 865.941 550.551 877.73 542.539 c 881.328
539.219 884.922 535.898 888.551 532.621 c 891.25 529.141 893.281 525.309 894.531 521.172 c 910.43
468.789 887.578 400.199 865.23 352.191 c 860.078 340.98 855.352 332.148 852.148 326.41 c 849.059
327.031 l 799.922 336.449 741.172 349.57 683.199 359.449 c 691.449 371.41 705.23 394.379 721.129
432.539 c 730.941 456.051 737.23 485.82 741.02 509.301 c 741.879 514.379 742.621 519.18 743.199
523.629 c 744.301 521.129 l 747.301 513.871 753.16 494.141 y 769.262 515.16 800.66 554.141 828.051
554.961 c 829.219 555 830.43 555.039 831.602 555.039 c h 1532.85 538.281 m 1525.59 525.309 1521.48
508.281 1521.48 486.801 c 1521.48 449.922 1543.67 408.32 1564.65 376.988 c 1577.19 358.238 1589.73
342.5 1597.89 332.73 c 1482.34 305.469 1366.37 280.781 1246.21 280.781 c 1113.91 280.781 984.18
301.48 868.59 323.32 c 886.48 355.859 901.25 396.762 908.711 430.781 c 917.23 470 918.199 489.57
913.398 511.25 c 956.68 475.781 1007.3 442.582 1066.29 411.09 c 1148.87 367.031 1214.73 346.91 y
1277.15 367.109 1356.29 411.211 v 1410.35 441.332 1473.59 483.09 1532.85 538.281 c h 1973.4 382.891 m
2063.67 382.891 2095.82 366.328 2102.23 361.801 c 2106.29 358.949 2108.09 356.801 2104.84 341.41 c
2101.37 324.961 2098.71 288.871 2103.16 267.301 c 2106.72 250.121 2117.27 206.949 2118.63 178.398 c
2097.66 194.879 2051.41 223.16 1982.03 223.16 c 1942.19 223.16 1887.85 215.309 1836.56 206.41 c
1837.07 207.301 1837.15 208.012 y 1832.5 225.23 l 1830.12 234.219 1827.85 248.281 1826.84 260.121 c
1825.9 271.02 1825.43 276.719 1823.16 286.68 c 1826.68 287.66 1830.2 288.281 1833.95 288.52 c 1843.36
288.98 1848.05 286.129 1851.37 287.539 c 1854.65 288.98 1855.62 297.422 1847.58 299.801 c 1839.61
302.148 1836.8 300.699 1827.81 300.27 c 1822.19 299.961 1822.27 299.879 1820.08 299.18 c 1818.24
306.25 1816.48 312.66 1814.53 321.52 c 1811.68 334.262 1806.52 338.98 1802.73 337.621 c 1798.95
336.172 1797.58 333.789 1798.48 329.102 c 1799.41 324.379 1803.71 318.32 1805.55 308.32 c 1808.32
296.172 l 1798.32 293.672 1795.08 291.559 1789.06 289.379 c 1782.38 287.07 1778.63 282.379 1778.63
278.09 c 1778.63 273.871 1785.78 270.59 1790 271.949 c 1794.22 273.359 1795.66 276.211 1806.02
280.469 c 1811.37 282.621 l 1812.85 275.512 1814.1 268.238 1814.53 262.031 c 1815.51 246.91 1818.75
232.34 1819.73 223.789 c 1822.19 206.449 l 1823.2 205.43 1824.88 204.531 v 1811.8 202.191 1801.02
200.199 1788.91 197.809 c 1787.73 208.441 1785.51 226.68 1783.4 234.609 c 1780 247.379 1777.7 241.719
1776.25 243.121 c 1774.84 244.57 1773.48 246.48 1772.03 251.172 c 1770.62 255.898 1770.16 260.121
1764.02 260.121 c 1757.89 260.121 1753.63 260.59 1750.78 261.051 c 1749.41 261.289 1749.3 262.77
1749.45 264.18 c 1755.74 266.02 1761.8 268.551 1764.96 271.949 c 1771.56 279.02 1772.5 292.66 1768.71
301.211 c 1764.96 309.73 1758.83 316.371 1750.31 316.84 c 1741.84 317.301 1738.48 317.77 1738.52
318.48 c 1747.23 323.238 1753.16 329.77 1755.08 336.641 c 1757.38 345.121 1755.08 346.949 1752.7
348.441 c 1750.9 349.57 1750.2 351.68 1748.91 352.422 c 1823.05 369.262 1898.32 382.891 1973.4
382.891 c h 547.621 357.852 m 598.441 357.852 674.691 343.441 759.801 326.879 c 759.18 323.32 758.871
318.672 760.121 315.391 c 762.5 309.301 770.941 304.02 773.32 294.609 c 775.621 285.16 774.02 253.711
772.379 247.379 c 766.719 225.23 765.738 217.621 765.738 203.012 c 765.738 188.359 771.371 172.77
779.449 173.711 c 787.539 174.648 787.578 181.289 785.621 189.34 c 783.672 197.379 782.77 218.629
786.129 231.371 c 789.41 244.102 793.199 252.191 793.672 268.199 c 794.141 284.219 795.59 294.57
789.379 304.488 c 783.199 314.449 777.539 317.77 778.012 321.52 c 778.051 323.398 l 861.68 307.109
953.051 289.57 1053.12 277.621 c 1052.23 275.98 1051.6 273.91 1051.02 271.52 c 1048.63 266.289
1035.66 261.02 1034.45 247.379 c 1032.54 226.641 1042.46 223.238 1041.56 218.551 c 1040.59 213.871
1028.32 210.512 1030.2 194.488 c 1031.25 185.738 1034.69 180.512 1039.22 177.379 c 1037.73 173.98
1035.16 168.949 1031.6 165.699 c 1025.94 160.512 1019.34 153.441 1016.52 149.219 c 1013.67 144.922
1013.2 138.789 1017.97 136.949 c 1022.66 135.039 1025.94 137.422 1029.77 142.621 c 1033.52 147.77
1034.92 148.75 1038.2 154.379 c 1039.69 156.91 1041.95 159.57 1043.98 161.68 c 1043.4 153.441 l
1042.34 145.621 1039.69 136.449 1039.61 127.93 c 863.98 150.66 651.371 188.16 462.461 188.16 c
402.621 188.16 355.039 163.789 350.59 161.48 c 348.121 160.352 345.199 158.949 342.23 157.191 c
343.012 173.949 345.039 197.809 347.621 230.039 c 350.551 265.16 347.301 304.102 345.859 318.52 c
407.109 344.301 481.91 357.852 547.621 357.852 c h 1743.52 351.172 m 1743.55 350.898 1743.55 350.66 y
1742.54 349.141 1742.66 346.68 1742.34 343.238 c 1742.07 340.469 1739.77 336.25 1736.17 332.809 c
1735.2 334.73 l 1732.81 338.711 1729.1 343.711 1726.52 347.23 c 1742.85 351.02 l h 1716.95 344.961 m
1717.73 339.449 l 1718.63 336.762 1719.06 332.109 1719.06 327.23 c 1713.05 327.73 1710.27 329.57
1706.91 331.48 c 1702.62 333.789 1697.38 337.148 1696.45 331.48 c 1695.51 325.781 1695.12 325.352
1699.84 321.52 c 1702.77 319.141 1709.92 315.469 1717.85 314.219 c 1717.3 311.949 1716.64 310.309
1715.86 309.73 c 1712.58 307.422 1701.72 302.621 1699.84 296.48 c 1697.93 290.352 1700.31 271.09
1706.91 266.762 c 1713.16 262.699 1723.12 259.961 1728.09 258.281 c 1723.87 255.898 1718.2 252.621
1718.2 249.77 c 1718.2 246.91 1720.08 246.02 1726.21 246.91 c 1732.34 247.852 1738.09 250.23 1744.65
250.699 c 1751.25 251.172 1758.36 252.109 1754.57 249.262 c 1750.78 246.48 1743.63 242.66 1735.66
241.25 c 1727.62 239.879 1721.02 239.84 1717.27 237.461 c 1713.52 235.121 1710.66 233.711 1712.11
228.98 c 1713.52 224.262 1715.86 224.77 1723.4 226.641 c 1730.98 228.551 1737.62 231.84 1743.28
232.77 c 1748.95 233.711 1756.91 234.219 y 1751.8 226.172 1744.18 223.32 v 1736.6 220.512 1728.59
219.531 1725.74 219.02 c 1722.93 218.551 1722.46 214.34 1725.74 212.422 c 1729.06 210.551 1735.7
210.551 1743.28 213.398 c 1750.78 216.25 1756.45 218.629 1762.58 221.48 c 1768.71 224.262 1769.65
228.121 1772.97 220.98 c 1775 216.602 1777.38 203.59 1778.12 195.66 c 1736.29 187.23 1700.9 179.301
1691.99 177.109 c 1623.83 160.391 1417.27 110.781 1238.98 110.781 c 1187.34 110.781 1124.06 117.148
1055.74 125.738 c 1056.17 133.121 l 1056.88 142.07 1057.3 145.859 1060 150.121 c 1067.03 147.301
1066.56 139.262 1073.16 137.891 c 1079.77 136.48 1082.15 138.789 1081.17 146.371 c 1080.23 153.91
1075.55 158.121 1067.03 162.34 c 1063.63 164.102 1061.33 165.66 1059.73 166.91 c 1059.77 172.77 l
1064.1 172.73 1068.55 173.199 1072.19 174.879 c 1074.34 173.98 1083.16 170.23 1089.77 168.98 c
1097.27 167.621 1110.98 167.109 1117.07 169.449 c 1123.2 171.84 1129.88 181.762 1132.73 199.219 c
1135.55 216.68 1139.3 234.691 1137.42 246.91 c 1135.55 259.18 1132.62 263.012 1126.48 263.48 c
1120.43 263.98 1118.05 261.52 1118.52 256.371 c 1118.98 251.172 1120.39 247.379 1119.92 238.441 c
1119.02 223.98 l 1116.91 223.359 1114.34 223.012 1111.41 223.32 c 1103.87 224.262 1098.2 226.172
1095.39 224.77 c 1092.54 223.32 1091.13 216.641 1093.05 212.891 c 1094.92 209.141 1099.14 206.762
1107.66 206.289 c 1111.09 206.129 1114.53 206.719 1117.38 207.461 c 1116.45 199.961 1115.12 193.012
1113.32 188.828 c 1108.55 178.012 1103.4 178.98 1096.8 180.82 c 1091.33 182.422 1084.57 185.59
1080.04 186.641 c 1081.76 199.379 1077.34 206.988 1074.57 211.949 c 1071.8 217.109 1068.44 218.121
1073.63 223.789 c 1078.79 229.449 1082.62 239.879 1081.68 250.23 c 1080.7 260.59 1077.42 262.539
1074.57 267.73 c 1073.01 270.59 1072.3 273.441 1072.03 275.43 c 1129.14 269.18 1187.58 265.078
1246.21 265.078 c 1302.27 265.078 1363.83 271.949 1421.99 281.129 c 1421.64 280.121 1421.56 279.34 y
1422.15 275.898 1422.5 271.988 1421.99 269.41 c 1411.99 266.371 1400.94 260.621 1394.3 252.109 c
1386.25 241.719 1381.48 228.949 1381.02 217.621 c 1380.59 206.289 1381.02 198.75 1384.34 195 c
1387.62 191.25 1394.3 191.211 1395.7 195.469 c 1396.84 198.75 1395.66 203.672 1395.51 210.121 c
1400.08 209.301 1406.45 208.09 1411.25 206.762 c 1417.5 205.078 1424.22 201.051 1430.62 194.531 c
1422.62 178.98 1421.21 169.922 1420.7 159.57 c 1420.23 149.219 1421.21 141.719 1422.62 136.949 c
1423.98 132.148 1425.9 128.871 1430.12 130.27 c 1434.41 131.641 1436.37 135.512 1435.82 141.172 c
1435.31 146.84 1435.43 162.77 1437.27 173.711 c 1438.12 178.711 1439.92 183.551 1442.27 188.16 c
1451.21 180.352 1457.11 168.281 1461.8 161.02 c 1466.99 152.969 1469.84 138.32 1472.62 136.949 c
1475.51 135.512 1481.17 137.891 1481.17 142.109 c 1481.17 146.371 1479.77 149.219 1475.98 161.48 c
1472.46 172.621 1460.82 191.289 1453.32 203.012 c 1464.18 213.871 1475.55 224.77 1481.17 229.449 c
1484.69 232.621 l 1485.86 229.059 1486.64 225.66 1486.84 222.422 c 1487.3 214.34 1487.81 206.289
1494.41 206.762 c 1501.02 207.27 1502.07 214.77 1501.48 222.422 c 1500.55 234.219 1493.05 252.07
1485.9 260.121 c 1478.75 268.199 1460.35 272.891 1453.32 273.398 c 1448.24 273.711 1443.44 273.59
1437.07 272.621 c 1435.62 277.23 1435 276.48 1435.35 281.129 c 1435.23 282.148 1434.77 283.281 v
1519.14 297.23 1594.65 315.43 1631.37 324.219 c 1714.73 344.41 l h 676.09 328.16 m 675.121 328.199
674.18 328.16 673.238 328.121 c 658.59 327.66 658.121 321.988 658.121 317.301 c 658.121 312.578
661.879 310.23 668.012 309.73 c 674.141 309.262 696.879 308.871 709.609 307.422 c 722.301 305.941
727.969 303.551 732.23 303.09 c 736.48 302.621 738.871 307.422 739.34 311.172 c 739.77 314.922 734.57
321.52 722.301 322.969 c 710.82 324.301 690.59 328.09 676.09 328.16 c h 1745.08 307.031 m 1748.36
306.91 1752.81 305.121 1753.63 302.621 c 1754.57 299.801 1750.31 296.949 1745.12 296.48 c 1739.96
296.02 1737.54 297.422 1737.11 301.211 c 1736.64 305 1740.43 306.48 1743.75 306.949 c 1744.14 306.988
1744.61 307.031 1745.08 307.031 c h 696.211 303.16 m 692.809 303.051 691.129 300.898 691.129 291.762
c 691.129 280.941 686.91 253.551 685.039 236.051 c 683.121 218.551 681.719 204.879 681.719 194.02 c
681.719 183.199 682.66 178.48 688.789 178.48 c 694.922 178.48 698.711 184.141 698.711 196.879 c
698.711 209.648 701.09 239.801 703.91 256.371 c 706.762 272.891 708.672 282.852 708.199 290.352 c
707.73 297.891 703.09 302.621 698.238 303.09 c 697.5 303.16 696.84 303.199 696.211 303.16 c h 2383.63
301.41 m 2413.2 301.41 2426.41 297.809 2427.62 296.051 c 2428.67 294.531 2429.53 291.129 2429.45
284.18 c 2429.14 250.27 2423.52 216.68 2433.4 183.16 c 2436.33 173.16 2442.34 157.301 2445.39 143.09
c 2428.12 151.02 2401.84 150.59 2388.98 150.59 c 2266.13 150.59 2159.02 75.3516 2036.45 75.3516 c
1988.05 75.3516 1972.23 89.5703 1969.57 97.5 c 1969.1 98.9492 1968.52 102.148 v 1969.38 111.68 1976.6
122.27 1994.8 128.711 c 2027.03 140.121 2083.87 138.012 2120.35 144.801 c 2132.89 147.148 2135.31
159.449 2134.92 170.551 c 2134.1 197.539 2128.05 222.621 2123.2 247.852 c 2184.69 267.73 2260.39
290.898 2323.12 297.699 c 2348.2 300.43 2368.16 301.41 2383.63 301.41 c h 1723.59 300.98 m 1726.21
300.941 1727.27 299.379 1727.62 297.422 c 1728.12 295.121 1726.68 293.199 1722.93 291.762 c 1719.18
290.352 1715.35 288.012 1713.95 292.73 c 1712.58 297.422 1717.27 299.801 1721.48 300.699 c 1722.3
300.898 1722.97 300.98 1723.59 300.98 c h 670.352 291.371 m 670.039 291.371 669.73 291.328 669.449
291.289 c 664.73 290.82 660.98 291.289 659.102 278.551 c 657.191 265.82 650.039 240.859 647.73
229.449 c 645.352 218.09 642.969 208.672 644.41 203.012 c 645.82 197.34 649.609 194.961 654.84 197.34
c 660.039 199.691 660.512 204.488 663.789 220.512 c 667.07 236.52 671.02 251.988 675.121 266.762 c
681.289 289.34 675.039 291.559 670.352 291.371 c h 728.789 287.699 m 728.281 287.699 727.699 287.621
727.07 287.539 c 720.941 287.539 720 286.602 719.02 275.27 c 718.09 263.98 718.09 220 719.02 208.672
c 720 197.379 721.84 187.422 729.879 186.949 c 737.891 186.48 739.77 193.09 738.871 203.48 c 737.891
213.871 735.039 244.488 735.512 261.52 c 735.98 277.5 736.84 287.66 728.789 287.699 c h 1755.51
287.621 m 1757.19 287.66 1758.32 287.301 1758.32 285.621 c 1758.32 282.379 1758.36 280.469 1753.63
279.02 c 1748.95 277.621 1745.12 277.621 1744.65 280.469 c 1744.18 283.238 1746.09 287.07 1749.88
287.07 c 1751.76 287.07 1753.87 287.539 1755.51 287.621 c h 1718.24 279.449 m 1719.84 279.531 1721.99
279.02 1723.87 279.02 c 1727.66 279.02 1732.34 277.621 1732.34 275.738 c 1732.34 273.871 1729.53
270.98 1723.4 270.59 c 1717.27 270.078 1716.33 272.422 1715.86 276.719 c 1715.59 278.828 1716.68
279.41 1718.24 279.449 c h 1446.41 260.199 m 1462.07 260.352 1474.14 255.551 1477.81 248.32 c 1478.91
246.129 l 1475.74 242.621 1469.77 236.449 1461.33 229.449 c 1452.7 222.379 1445.66 216.211 1439.57
208.551 c 1426.09 214.41 1404.06 220.309 1396.45 222.191 c 1397.77 230.391 1400.35 239.379 1405.59
245.039 c 1410 249.73 1416.95 253.98 1425 256.719 c 1427.3 251.559 1428.24 244.102 1432.5 242.191 c
1436.8 240.352 1439.18 241.25 1439.18 246.48 c 1439.18 250.078 1439.92 254.18 1439.49 259.84 c
1441.88 260.078 1444.18 260.199 1446.41 260.199 c h 1057.03 259.648 m 1057.23 259.691 1057.42 259.648
1057.62 259.609 c 1060.43 259.141 1067.5 253.051 1067.03 246.91 c 1066.56 240.781 1064.22 232.77
1058.55 231.84 c 1052.93 230.859 1048.63 240.352 1048.63 247.379 c 1048.63 254.02 1054.02 259.41
1057.03 259.648 c h 192.699 258.949 m 272.5 258.949 317.699 250.27 332.809 246.719 c 332.66 241.48
332.422 236.328 332.031 231.289 c 328.441 187.539 326.129 159.02 326.129 142.621 c 326.129 118.121
353.59 111.371 373.828 111.371 c 394.879 111.371 426.559 112.5 452.461 112.5 c 459.84 112.5 497.5
112.422 497.5 99.1406 c 497.5 98.0898 497.34 97.1484 497.031 96.25 c 492.301 80.8984 450.781 80.0781
441.879 80.0781 c 401.949 80.0781 266.91 111.719 158.949 111.719 c 100.941 111.719 45.5898 97.1484
20.3516 77.4609 c 18.9805 118.398 16.0195 207.191 16.0195 214.34 c 16.0195 219.102 14.1406 227.07
17.6211 231.02 c 20.6992 234.41 90.0391 258.949 192.699 258.949 c h 1452.23 249.77 m 1449.49 249.648
1448.28 232.301 1433.95 231.371 c 1419.3 230.391 1417.38 238.441 1416.48 239.879 c 1415.55 241.25
1408.48 244.57 1408.48 238.441 c 1408.48 232.301 1417.34 220.512 1433.95 222.422 c 1450.47 224.262
1455.08 233.871 1457.11 239.379 c 1460.35 248.32 1455.2 249.301 1452.38 249.77 c 1452.34 249.77
1452.3 249.77 1452.23 249.77 c h 1055.23 209.141 m 1059.02 209.141 1063.75 202.539 1066.13 198.32 c
1068.44 194.02 1065.16 185.078 1059.02 184.141 c 1052.89 183.199 1046.29 187.422 1045.82 195.941 c
1045.35 204.379 1051.48 209.141 1055.23 209.141 c h 1982.03 207.461 m 2028.12 207.461 2075.82 193.359
2111.68 163.91 c 2113.48 162.422 2114.96 161.09 2116.17 160 c 2094.77 156.129 2071.8 155.551 2050.35
153.281 c 2027.27 150.898 1986.41 148.629 1965.62 129.77 c 1965.35 134.879 1965.23 140.66 1965.23
147.191 c 1965.23 164.922 1969.53 180.121 1973.79 190.98 c 1977.5 200.27 1981.21 206.211 1982.03
207.461 c h 1963.67 207.148 m 1954.57 189.102 1949.53 165.551 1949.53 147.191 c 1949.53 124.262
1951.25 109.57 1952.73 101.25 c 1952.73 100.078 l 1952.73 77.7305 1982.81 59.6484 2036.45 59.6484 c
2159.45 59.6484 2266.33 134.879 2388.98 134.879 c 2397.85 134.879 2446.41 135.699 2446.41 120.078 c
2446.41 113.949 2440.35 105.391 2427.34 95.8594 c 2365 50.1602 2215.51 31.0508 2150.94 31.0508 c
2041.88 31.0508 1954.96 53.3203 1893.83 75.6602 c 1833.91 97.6602 1768.87 151.52 1744.02 173.32 c
1808.2 187.93 1893.01 204.57 1963.67 207.148 c h 462.461 172.461 m 477.578 172.461 493.91 172.109
511.172 171.371 c 510.352 157.031 508.199 134.57 504.379 117.191 c 490.98 127.578 465.781 128.199
452.461 128.199 c 426.52 128.199 394.879 126.988 373.828 126.988 c 355.98 126.988 345.941 131.91
343.398 138.441 c 346.051 141.641 350.551 143.98 354.18 145.82 c 355.859 146.68 357.34 147.301 y
386.719 163.121 429.691 172.461 462.461 172.461 c h 526.84 170.66 m 585.078 167.699 652.23 161.41
717.191 153.871 c 711.328 149.879 703.871 145.039 694.879 139.531 c 554.961 54.0586 387.422 18.9102
224.691 18.9102 c 166.641 18.9102 114.488 24.4492 76.9102 32.5781 c 68.6719 34.3789 27.4609 44.0586
21.7188 54.5703 c 21.7188 66.4805 73.6289 96.0195 158.949 96.0195 c 266.91 96.0195 401.949 64.3789
441.879 64.3789 c 461.051 64.3789 503.441 66.6406 511.762 90.3516 c 518.48 102.93 522.109 124.219
524.18 141.211 c 525.66 153.238 526.449 164.18 526.84 170.66 c h 526.84 170.66 m f 0.835938 0.611328
0.141113 rg 1207.93 1633.05 m 905.551 1633.05 691.328 1583.52 y 680.039 1580.43 678.238 1574.84
678.238 1569.06 c 678.238 1548.95 684.02 1255.2 685.512 1159.49 c 690.859 809.922 759.961 593.711
1074.26 426.051 c 1144.14 388.789 1199.96 369.531 1213.83 365 c 1213.83 558.051 l 1182.77 558.051 l
1187.77 550.27 1187.77 539.02 v 1187.77 527.73 1174.61 509.691 1163.32 501.211 c 1152.03 492.73
1132.3 493.441 1124.1 495.551 c 1115.86 497.66 1119.18 502.66 1128.32 506.879 c 1137.5 511.09 1147.11
513.91 1147.11 530.551 c 1147.11 547.27 1118.44 556.602 1099.41 556.602 c 1080.39 556.602 1079.22
545.352 1079.22 523.051 c 1079.22 500.738 1086.99 476.801 1093.83 466.91 c 1100.59 457.07 1102.97
452.148 1095.2 454.262 c 1087.46 456.371 1071.52 465.551 1058.83 482.191 c 1046.13 498.832 1038.63
529.84 1038.63 546.801 c 1038.63 563.711 1037.93 563.711 1030.86 564.379 c 1023.83 565.082 1017.5
563.711 1009.02 563.711 c 1000.59 563.711 995.859 555.941 995.859 536.211 c 995.859 516.48 1005.51
508.238 1011.13 504.02 c 1016.8 499.77 1016.8 493.441 1004.1 494.879 c 991.449 496.289 982.5 509.691
973.359 521.68 c 964.219 533.629 959.219 559.449 971.25 583.91 c 983.16 608.32 1032.3 608.52 1039.34
609.961 c 1046.37 611.371 1043.55 621.02 1043.55 627.34 c 1043.55 646.09 1063.75 672.148 1092.38
677.809 c 1105.94 680.469 1120.59 685.551 1120.59 679.922 c 1120.59 674.301 1114.96 670.738 1104.38
664.41 c 1093.83 658.051 1086.33 647.07 1081.33 636.48 c 1076.45 625.898 1078.55 609.961 1084.22
605.699 c 1089.8 601.48 1118.44 597.93 1142.89 616.762 c 1167.3 635.551 1177.19 700.82 1177.19 711.41
c 1177.19 721.949 1176.45 724.77 1169.69 724.77 c 1162.85 724.77 1126.21 716.289 1075 716.289 c
959.488 716.289 914.41 801.09 913.949 810.23 c 913.48 819.379 936.91 831.129 951.051 836.09 c 965.121
841.02 988.16 836.09 997.301 832.539 c 1006.45 829.02 1001.99 839.57 y 990.941 849.449 966.988
849.449 v 933.871 849.449 884.34 818.012 855.66 799.691 c 827.031 781.371 801.719 747.77 802.379
736.02 c 803.09 724.301 803.75 714.18 798.16 714.18 c 792.5 714.18 775.16 742.148 775.16 775.98 c
775.16 809.77 789.02 831.129 805.941 838.91 c 822.809 846.68 872.379 864.961 900.551 873.871 c 928.75
882.809 947.5 892.93 973.359 893.59 c 999.141 894.301 1007.62 883.711 1013.95 880.898 c 1020.27
878.121 1025.94 876.68 1025.94 884.449 c 1025.94 892.191 1016.8 897.852 996.602 902.07 c 976.41
906.289 936.91 901.371 915.078 896.41 c 893.281 891.48 895.59 887.27 887.891 886.559 c 880.121 885.82
874.488 887.93 872.379 901.371 c 870.27 914.77 859.219 984.73 859.219 1018.32 c 859.219 1094.65
888.551 1096.02 888.551 1112.93 c 888.551 1129.84 870.469 1126.29 844.41 1133.12 c 818.359 1139.92
808.75 1158.48 804.488 1167.66 c 800.27 1176.8 793.949 1176.56 784.059 1173.28 c 774.219 1170 760.352
1164.14 750.469 1149.34 c 746.988 1144.1 744.879 1135.94 740.199 1135.94 c 735.469 1135.94 733.828
1143.71 738.75 1160.62 c 743.672 1177.54 762.5 1195.59 772.301 1206.88 c 782.191 1218.12 784.059
1217.42 782.891 1234.1 c 781.719 1250.78 768.828 1452.73 761.051 1496.88 c 758.359 1512.19 755.391
1509.53 746.949 1508.87 c 738.551 1508.12 727.5 1502.5 727.5 1489.14 c 727.5 1475.7 733.828 1475.27
738.75 1471.76 c 743.672 1468.2 753.281 1473.16 753.281 1467.54 c 753.281 1461.88 741.328 1454.84
729.609 1454.84 c 717.852 1454.84 694.801 1465.39 694.801 1484.84 c 694.801 1509.06 698.75 1522.23
711.949 1533.4 c 725.121 1544.53 760.43 1553.67 782.73 1553.67 c 799.961 1553.67 809.648 1540.86
813.32 1503.91 c 817.031 1466.95 826.172 1259.84 828.051 1230.27 c 829.879 1200.74 839.262 1193.28
859.531 1184.96 c 879.84 1176.64 890.039 1177.58 889.102 1192.38 c 888.199 1207.19 855.469 1213.01
855.469 1247.15 c 855.469 1268.67 863.199 1271.09 867.621 1271.72 c 872.07 1272.3 868.98 1259.92
872.699 1251.41 c 876.371 1242.85 883.051 1233.12 894.02 1233.12 c 905 1233.12 905 1242.7 904.379
1260.59 c 903.789 1278.4 892.191 1402.46 894.648 1431.25 c 897.07 1460.12 913.52 1470.27 943.551
1477.15 c 973.629 1484.02 1021.72 1492.58 1041.45 1490.74 c 1061.09 1488.95 1060.47 1478.98 1061.09
1466.21 c 1061.68 1453.4 1058.05 1443.63 1064.77 1443.63 c 1071.45 1443.63 1078.16 1444.1 1078.75
1434.96 c 1079.38 1425.82 1080.12 1416.02 1092.34 1413.59 c 1104.49 1411.17 1105.12 1420.16 1105.78
1430.66 c 1106.37 1441.21 1106.37 1447.89 1104.49 1458.87 c 1102.7 1469.84 1095.98 1497.42 1110.62
1509.65 c 1125.23 1521.8 1136.84 1531.13 1157.73 1534.18 c 1178.59 1537.23 1184.69 1537.97 1184.69
1528.67 c 1184.69 1519.38 1186.52 1495.59 1188.98 1490.16 c 1191.41 1484.65 1198.12 1483.44 1202.97
1484.65 c 1207.85 1485.86 1206.64 1494.38 1207.27 1507.81 c 1207.85 1521.21 1207.27 1555.55 y 1196.48
1558.55 1181.05 1555.55 v 1165.66 1552.46 1141.88 1540.27 1137.58 1541.48 c 1133.36 1542.7 1144.92
1552.46 1151.02 1556.13 c 1157.11 1559.77 1184.1 1576.41 1194.45 1583.12 c 1204.84 1589.8 1207.27
1597.15 1207.85 1602.62 c 1208.36 1606.95 1208.05 1625.47 1207.93 1633.05 c 1209.77 630 m 1202.5
623.281 1195.78 619.379 v 1189.06 615.43 1182.38 607.852 1181.8 603.91 c 1181.21 600 1186.29 601.129
1193.01 600.551 c 1199.69 600 1212.54 599.102 1213.67 608.398 c 1214.8 617.73 1209.77 630 y 1208.63
875 m 1176.17 875.59 1150.12 875 v 1124.02 874.41 1100.47 875.59 1083.36 872.191 c 1066.21 868.828
1050.9 856.172 1050.9 828.199 c 1050.9 800.199 1062.62 777.699 1096.05 777.699 c 1129.41 777.699
1148.44 790.898 y 1126.8 791.48 1109.45 795.391 v 1092.15 799.301 1085.59 807.109 1085.59 819.449 c
1085.59 831.762 1094.34 839.379 1102.19 839.379 c 1110 839.379 1208.63 842.73 y h 935.121 1100.04 m
917.23 1100.04 888.359 1081.56 888.359 1053.83 c 888.359 1026.02 923.398 1016.52 936.25 1016.52 c
949.102 1016.52 983.59 1032.54 983.59 1056.56 c 983.59 1080.66 956.949 1100.04 935.121 1100.04 c
1202.5 1176.29 m 1196.91 1175.2 1171.91 1154.49 1166.88 1133.44 c 1161.84 1112.34 1163.55 1080.12
1165.78 1069.49 c 1167.97 1058.87 1176.17 1057.15 1186.29 1056.05 c 1196.33 1054.92 1208.63 1056.05
1209.22 1061.05 c 1209.77 1066.09 1209.22 1111.8 y 1199.14 1110.66 1192.42 1101.13 v 1185.7 1091.64
1180.66 1082.15 1180.08 1089.41 c 1179.57 1096.68 1177.85 1111.45 1184.57 1126.88 c 1191.29 1142.38
1197.42 1152.46 1204.73 1160.82 c 1212.03 1169.22 1208.12 1177.42 1202.5 1176.29 c 956.91 1430.08 m
948.199 1429.3 941.02 1429.3 941.02 1414.22 c 941.02 1399.1 944.18 1341.33 944.18 1316.95 c 944.18
1292.66 942.07 1280.7 952.148 1279.14 c 962.191 1277.54 976.52 1282.3 976.52 1289.41 c 976.52 1296.6
967.73 1313.01 977.301 1324.96 c 986.84 1336.84 998.52 1340.2 1006.33 1340.2 c 1014.14 1340.2 1016.45
1337.97 1015.86 1332.34 c 1015.27 1326.76 1008.91 1331.25 1004.1 1325.62 c 999.262 1320.04 995.121
1304.02 1010.23 1301.8 c 1025.35 1299.57 1024.77 1306.6 1027.07 1312.81 c 1029.26 1318.95 1029.26
1332.34 y 1048.48 1318.55 1046.41 1295.62 v 1044.41 1272.66 1030.39 1262.85 1024.77 1260.04 c 1019.22
1257.23 1018.09 1249.41 1020.35 1243.83 c 1022.58 1238.2 1027.58 1233.71 1033.79 1237.11 c 1039.88
1240.43 1062.62 1257.03 1067.7 1263.75 c 1072.73 1270.47 1072.15 1276.05 1076.64 1272.66 c 1081.13
1269.34 1083.79 1255.55 1080.55 1232.62 c 1075.55 1196.45 1028.16 1146.88 991.449 1146.88 c 962.34
1146.88 954.18 1184.73 954.18 1200.94 c 954.18 1217.19 961.172 1234.84 y 981.02 1189.14 993.121
1189.14 v 1005.23 1189.14 1001.52 1207.97 997.621 1219.77 c 993.629 1231.48 979.488 1249.96 965.699
1249.96 c 951.879 1249.96 925.039 1223.09 925.039 1185.82 c 925.039 1148.52 926.719 1121.88 940.738
1121.88 c 954.691 1121.88 986.211 1130.86 1005.23 1127.46 c 1024.26 1124.06 1043.83 1117.93 1043.28
1111.25 c 1042.7 1104.53 1008.01 1094.65 1008.01 1037.77 c 1008.01 980.941 1045.86 919.57 1147.3
919.57 c 1209.77 919.57 l 1209.77 1005.31 l 1151.76 1005.31 l 1133.91 1005.31 1103.87 1019.3 1103.87
1067.77 c 1103.87 1116.25 1098.91 1134.49 1106.09 1154.69 c 1120.62 1195.35 1162.73 1214.69 1177.85
1225.9 c 1193.01 1237.11 1200.82 1238.79 1205.31 1247.15 c 1209.77 1255.55 1208.63 1272.15 1208.63
1283.36 c 1208.63 1391.45 l 1194.61 1388.12 1188.48 1383.63 v 1182.38 1379.18 1176.17 1373.59 y
1175.66 1384.57 1180.66 1392.93 v 1185.7 1401.33 1186.8 1407.54 1186.8 1416.45 c 1186.8 1425.39
1185.39 1443.09 1173.05 1443.09 c 1165.86 1443.09 1160.39 1443.52 1157.34 1438.59 c 1155.2 1435.08
1154.73 1427.5 1154.02 1420.94 c 1152.34 1405.23 1156.84 1383.44 1156.25 1365.7 c 1155.66 1348.05
1149.53 1315 1111.09 1315 c 1072.73 1315 1063.75 1329.22 1062.11 1339.65 c 1060.43 1350.08 1056.88
1349.69 1050.31 1349.18 c 1043.83 1348.59 1034.88 1346.72 1034.3 1358.48 c 1033.79 1370.2 1031.48
1396.33 1029.84 1409.73 c 1028.16 1423.16 1025.35 1435.47 1018.09 1435.47 c 1010.82 1435.47 1006.33
1431.56 1006.33 1421.48 c 1006.33 1411.45 1009.92 1389.8 998.121 1384.22 c 986.41 1378.59 980.82
1380.31 971.879 1385.31 c 962.891 1390.35 951.879 1402.46 y 962.34 1401.91 970.199 1407.54 v 978.012
1413.09 979.691 1419.26 975.781 1424.84 c 971.879 1430.47 965.621 1430.9 956.91 1430.08 c 837.07
980.941 m 846.02 980.941 849.34 968.629 849.922 958.52 c 850.512 948.48 852.73 938.012 847.699
931.879 c 842.621 925.699 835.391 930.199 817.109 926.289 c 798.871 922.34 776.449 921.25 769.73
916.172 c 763.012 911.172 758.012 900.16 753.52 900.16 c 749.059 900.16 747.379 919.57 749.57 936.371
c 751.879 953.121 755.199 964.102 770.898 970.82 c 786.52 977.539 828.09 980.941 837.07 980.941 c
935.121 1081.8 m 952.461 1081.8 960.078 1071.17 960.078 1059.96 c 960.078 1052.66 946.879 1037.77
935.121 1037.77 c 923.398 1037.77 911.25 1042.27 911.25 1057.15 c 911.25 1072.07 923.398 1081.8
935.121 1081.8 c 1529.73 1023.2 m 1511.21 1005.98 1473.95 972.621 1466.41 965.512 c 1458.83 958.32
1457.97 957.77 1457.97 953.16 c 1457.97 948.52 1460.82 944.77 1466.72 944.77 c 1472.58 944.77 1479.14
946.48 1479.14 952.77 c 1479.14 959.059 1473.95 954.879 1473.95 959.141 c 1473.95 963.48 1475.23
965.512 1479.57 965.512 c 1483.83 965.512 1488.4 961.25 1488.4 953.16 c 1488.4 945.078 1474.8 939.262
1465.55 939.262 c 1451.21 939.262 1446.6 944.34 1446.6 952.34 c 1446.6 960.309 1454.18 969.691 1460.9
978.551 c 1467.66 987.379 1475.23 995.352 1475.23 1005.98 c 1475.23 1021.64 1457.46 1035.55 1442.77
1035.55 c 1428.01 1035.55 1428.44 1026.6 y 1444.49 1025.74 1444.49 1004.73 v 1444.49 995.469 1436.95
986.09 1429.3 986.09 c 1421.56 986.09 1410.74 992.809 1409.45 999.18 c 1408.24 1005.47 1407.23
1015.12 1414.8 1026.02 c 1422.42 1036.95 1438.63 1049.45 1449.57 1052.77 c 1460.51 1056.17 1469.77
1050.7 1479.57 1043.01 c 1489.34 1035.31 1494.26 1029.14 y 1497.23 1040.08 1496.8 1049.84 v 1496.37
1059.65 1483.71 1071.76 y 1488.79 1073.79 1495.08 1073.4 v 1501.45 1072.97 1508.16 1064.14 1514.45
1062.03 c 1520.78 1059.96 1533.52 1060.39 y 1529.73 1075.51 1525.51 1081.02 v 1521.33 1086.45 1514.88
1087.7 1514.88 1091.21 c 1514.88 1094.65 1518.67 1099.18 1523.36 1099.61 c 1528.12 1100.04 1530.55
1093.67 1536.02 1092.89 c 1541.52 1092.03 1554.49 1081.02 y 1555.78 1094.1 1552.46 1102.15 v 1549.06
1110.12 1545.27 1111.41 1544.84 1115.59 c 1544.41 1119.8 1545.27 1122.34 1548.2 1122.34 c 1551.17
1122.34 1558.75 1121.48 1565.04 1116.84 c 1571.37 1112.23 1573.16 1100.47 y 1583.28 1111.8 1582.85
1119.38 v 1582.42 1126.95 1574.45 1129.88 1570.2 1134.65 c 1565.98 1139.38 1567.66 1141.8 1571.48
1146.8 c 1575.23 1151.88 1581.56 1147.23 1586.17 1145.55 c 1590.82 1143.91 1597.54 1133.79 y 1600.08
1143.91 1598.79 1151.05 v 1597.54 1158.16 1592.11 1156.48 1591.25 1164.53 c 1590.39 1172.5 1597.11
1175.98 y 1600.08 1169.53 1606.8 1165.78 v 1613.55 1161.99 1621.64 1163.24 y 1620.39 1175.59 1616.56
1182.7 v 1612.81 1189.84 1608.91 1189.84 1608.91 1194.49 c 1608.91 1199.1 1615.35 1201.21 1620.39
1197.38 c 1625.47 1193.63 1632.58 1185.66 y 1635.98 1198.24 1634.69 1206.72 v 1633.44 1215.12 1626.29
1219.8 1625.86 1224.45 c 1625.47 1229.06 1630.9 1229.92 1635.98 1228.2 c 1640.98 1226.56 1647.34
1220.66 y 1648.98 1228.2 1646.48 1236.25 v 1643.95 1244.22 1635.55 1248.05 y 1639.3 1255.2 1643.12
1263.67 v 1646.91 1272.19 1646.48 1272.11 1648.55 1278.44 c 1650.66 1284.73 1655 1292.3 1656.25
1299.88 c 1657.54 1307.46 1663.4 1322.27 y 1667.62 1285.16 1662.58 1247.19 v 1657.54 1209.18 1632.5
1148.91 1607.23 1114.3 c 1581.99 1079.73 1548.2 1040.51 1529.73 1023.2 c 1395.66 669.531 m 1389.69
669.531 1387.46 663.59 1384.1 651.801 c 1380.74 640.078 1369.77 615.66 1345.12 608.949 c 1320.51
602.23 1281.02 598.91 1261.02 597.23 c 1241.09 595.512 1223.55 594.922 y 1223.55 561.559 l 1266.09
555.98 1293.28 555.98 v 1320.51 555.98 1378.28 560.82 1396.76 584.301 c 1415.23 607.852 1414.69
630.199 1412.42 646.211 c 1410.2 662.27 1401.56 669.531 1395.66 669.531 c 1223.16 1584.14 m 1219.3
1583.71 1216.68 1580.7 1216.68 1574.69 c 1216.68 1568.63 1216.72 1567.34 1225.27 1563.01 c 1233.91
1558.71 1245.47 1555.31 1247.66 1538.95 c 1249.77 1522.62 1248.91 1508.79 1236.48 1494.65 c 1224.02
1480.47 1213.67 1477.81 y 1213.67 1403.44 l 1240.31 1416.72 1253.2 1430.94 v 1266.13 1445.2 1285.08
1483.01 1285.08 1510.55 c 1285.08 1538.09 1266.09 1567.81 1251.91 1575.51 c 1237.77 1583.28 1227.03
1584.57 1223.16 1584.14 c 1596.05 1594.02 m 1565.51 1594.02 1474.8 1558.32 1436.95 1532.54 c 1399.1
1506.68 1383.55 1477.89 1371.99 1440.47 c 1360.35 1403.05 1350.9 1356.6 1324.26 1326.91 c 1297.54
1297.19 1263.98 1270.08 1245.9 1259.34 c 1227.89 1248.63 1213.24 1242.58 y 1213.24 1188.4 l 1298.83
1214.69 1333.67 1224.57 v 1368.55 1234.45 1441.17 1255.23 1469.57 1267.54 c 1519.02 1289.06 1540.12
1320.9 y 1493.24 1305.35 1464.41 1297.62 v 1435.59 1289.92 1406.88 1286.48 1398.63 1286.05 c 1390.43
1285.62 1388.24 1289.49 y 1404.65 1328.59 l 1407.23 1334.61 1409.84 1335.47 1417.97 1336.37 c 1426.09
1337.23 1470.82 1348.91 1502.27 1363.91 c 1559.06 1390.98 1578.4 1424.06 y 1524.92 1406.52 1495
1394.45 v 1446.84 1375.08 1416.29 1372.46 y 1415.31 1390.2 1426.99 1403.44 v 1438.63 1416.76 1450.2
1425.74 1466.13 1430.94 c 1482.07 1436.17 1528.55 1458.52 1548.32 1469.69 c 1568.05 1480.86 1587.46
1495.55 1596.05 1505.86 c 1604.61 1516.17 1609.38 1526.52 y 1574.96 1506.25 1535.43 1493.79 v 1495.82
1481.33 1468.71 1472.7 1459.69 1468.44 c 1450.7 1464.1 1445.08 1458.12 y 1454.53 1499.88 1495.82
1523.91 v 1537.11 1547.97 1557.3 1561.76 1584.41 1573.4 c 1611.56 1584.96 1636.48 1589.69 y 1613.67
1594.02 1596.05 1594.02 v 1272.81 1171.8 m 1254.3 1171.8 1224.18 1168.75 1219.06 1165.74 c 1213.91
1162.66 1214.57 1159.02 1213.91 1152.93 c 1213.36 1146.88 1213.91 1121.64 y 1216.6 1120.43 1223.32
1120.43 v 1230 1120.43 1237.03 1122.85 1252.5 1121.05 c 1267.93 1119.22 1281.95 1121.91 1281.95
1096.76 c 1281.95 1071.52 1270.16 1067.89 1253.71 1061.17 c 1237.3 1054.53 1224.18 1053.28 1220.27
1053.28 c 1216.37 1053.28 1216.99 1050.51 1216.99 1045.35 c 1216.99 1015.94 l 1216.99 1012.93 1218.48
1010.47 1225.78 1010.47 c 1233.05 1010.47 1282.54 1011.68 1291.68 1015.35 c 1300.82 1018.98 1319.34
1043.91 1319.34 1084.92 c 1319.34 1125.94 1315.94 1141.99 1307.19 1154.8 c 1298.4 1167.54 1291.33
1171.8 1272.81 1171.8 c 1698.4 839.73 m 1693.36 841.949 1674.02 844.531 1660.23 848.359 c 1644.02
852.809 1638.05 865.121 y 1661.33 931.879 1664.14 940.82 v 1666.95 949.801 1667.46 946.988 1676.45
950.309 c 1685.39 953.711 1702.54 967.461 1721.02 987.66 c 1726.25 993.359 1730 998.789 1732.7
1003.67 c 1734.96 1020.43 l 1729.02 1015.39 1718.05 1006.8 1708.12 1001.95 c 1693.59 994.879 1675.35
993.199 y 1684.3 1040.04 1686.52 1066.09 v 1688.79 1092.19 1691.52 1110.12 y 1708.71 1129.14 1717.11
1141.8 v 1725.47 1154.49 1732.19 1170.9 y 1720.43 1166.99 1711.52 1160.27 v 1702.54 1153.59 1700.31
1149.65 1694.34 1150.2 c 1688.4 1150.78 1694.34 1165.31 1694.34 1183.55 c 1694.34 1201.84 1696.6
1248.83 1698.09 1260.59 c 1699.49 1272.34 1702.54 1273.24 1715.43 1287.27 c 1728.28 1301.21 1733.87
1315 1738.36 1329.57 c 1742.81 1344.14 1743.75 1352.27 1740.35 1353.44 c 1737.03 1354.57 1734.96
1346.37 1725.47 1337.42 c 1715.94 1328.44 1705.9 1325.12 y 1704.22 1357.89 1705.9 1378.01 v 1707.58
1398.2 1721.02 1398.01 1731.64 1416.99 c 1742.27 1436.02 1737.77 1463.24 y 1719.92 1440.9 1712.03
1440.31 v 1704.22 1439.77 1705.9 1446.99 1706.45 1457.62 c 1707.03 1468.32 1707.58 1514.53 1705.9
1542.3 c 1704.22 1570.08 1679.26 1585.78 y 1682.07 1535.04 1678.09 1501.68 v 1674.18 1468.32 1672.54
1452.62 1662.46 1453.2 c 1652.34 1453.71 1615.66 1489.34 y 1606.37 1475.55 1613.4 1433.79 v 1617.15
1411.72 1630.2 1401.33 1636.33 1401.91 c 1642.54 1402.46 1652.58 1423.16 1657.97 1421.48 c 1663.36
1419.8 1662.46 1399.65 1656.84 1376.41 c 1651.25 1353.09 1638.59 1334.61 1633.55 1332.93 c 1628.52
1331.25 1623.48 1335.16 1612.85 1347.46 c 1602.23 1359.8 1587.34 1392.03 y 1578.95 1375.82 1578.95
1338.55 v 1578.95 1320.23 1588.44 1309.41 1592.93 1305.12 c 1597.38 1300.9 1600.78 1312.23 1610.59
1312.81 c 1620.47 1313.32 1621.8 1297.3 1617.3 1286.68 c 1612.85 1276.05 1601.29 1254.45 1596.84
1253.32 c 1592.34 1252.23 1588.98 1255.55 1577.23 1267.11 c 1565.51 1278.63 1551.72 1318.4 y 1537.19
1281.09 1537.19 1251.05 v 1537.19 1237.77 1542.77 1225.31 1547.81 1223.67 c 1552.81 1221.99 1555.08
1232.07 1562.34 1230.9 c 1569.61 1229.8 1568.32 1221.41 1564.02 1210.98 c 1559.73 1200.59 1529.34
1165.9 1523.16 1165.31 c 1517.03 1164.73 1513.32 1172.03 1506.05 1189.14 c 1498.75 1206.33 1493.12
1228.12 1492.62 1237.11 c 1492.03 1246.02 1488.71 1248.32 1481.99 1244.34 c 1475.27 1240.43 1456.45
1202.03 1454.18 1190.27 c 1451.95 1178.52 1455.27 1160.82 1460.9 1154.1 c 1466.48 1147.38 1473.79
1150.78 1479.38 1151.88 c 1485 1153.01 1489.22 1145.16 1487.58 1135.08 c 1485.9 1125.04 1442.46
1096.68 y 1440.23 1098.32 1432.54 1110.66 v 1424.96 1122.97 1413.52 1143.48 1409.61 1166.41 c 1405.7
1189.38 1411.88 1216.02 y 1401.21 1208.75 1393.4 1196.45 v 1385.55 1184.14 1379.06 1167.54 1375.66
1150.2 c 1372.34 1132.85 1374.57 1128.05 1376.25 1116.84 c 1377.97 1105.62 1376.84 1103.4 1381.29
1102.3 c 1385.78 1101.13 1393.4 1105.12 1401.8 1103.95 c 1410.2 1102.81 1409.61 1080.12 1403.48
1073.4 c 1397.3 1066.68 1388.59 1061.05 1371.76 1049.84 c 1355 1038.71 1359.45 1018.79 1361.13
991.559 c 1362.85 964.34 1358.91 952.578 1330.59 937.461 c 1302.27 922.34 1238.67 917.34 1227.46
915.078 c 1216.25 912.852 1218.75 911.949 1217.62 902.422 c 1216.45 892.93 1217.03 882.809 1219.65
880.59 c 1222.23 878.32 1244.26 877.809 1272.03 880.59 c 1299.8 883.398 1333.4 889.02 1349.92 899.059
c 1366.52 909.141 1376.25 916.172 1381.29 915.66 c 1386.37 915.078 1374.57 899.648 1356.64 883.949 c
1338.79 868.281 1322.19 855.59 1294.96 847.23 c 1267.77 838.828 1219.65 837.73 y 1219.65 649.609 l
1230.27 653.52 1239.26 661.719 v 1248.16 669.879 1260.51 695.27 1265.31 709.059 c 1270.12 722.852
1274.8 726.988 1291.64 729.801 c 1308.4 732.578 1331.68 739.84 1353.32 755.859 c 1374.92 771.879
1388.01 794.801 1397.3 817.23 c 1406.6 839.57 1410.2 855.078 1415.82 865.121 c 1421.41 875.199
1420.82 872.77 1432.03 882.301 c 1443.24 891.801 1453.63 898.16 1473.2 914.18 c 1492.81 930.199
1509.96 945.898 y 1513.32 941.41 1517.23 933.551 v 1521.17 925.699 1526.56 916.172 y 1535.47 926.84
1536.05 939.102 v 1536.6 951.449 1530.47 967.461 y 1538.87 965.23 1543.32 959.102 v 1547.81 952.93
1554.49 942.5 y 1558.98 954.219 1558.98 963.551 v 1558.98 972.852 1559.22 974.219 1558.67 979.77 c
1558.09 985.352 1556.72 991.559 1559.53 992.07 c 1562.34 992.66 1568.32 989.301 1572.23 984.84 c
1576.13 980.352 1577.81 973.629 y 1585.66 980.551 1586.37 995.512 v 1586.72 1002.46 1582.27 1008.09
1582.27 1016.52 c 1582.27 1024.92 1588.98 1021.02 1593.48 1018.2 c 1597.97 1015.39 1602.46 1006.45 y
1607.46 1015.98 1607.46 1023.79 v 1607.46 1031.64 1600.78 1038.32 1600.78 1043.36 c 1600.78 1048.44
1606.91 1046.52 1611.76 1041.13 c 1616.56 1035.74 1620.7 1026.02 y 1625.7 1042.27 1623.48 1049.34 v
1621.29 1056.41 1613.4 1059.96 y 1615.08 1062.77 1626.84 1086.8 v 1638.59 1110.86 1640.82 1119.06 y
1641.41 1085.7 1638.59 1063.36 v 1635.82 1040.94 1632.42 1033.83 1625.7 1021.52 c 1618.98 1009.26
1615.08 1005.31 y 1623.4 997.539 1623.91 993.871 v 1624.3 991.41 1617.3 987.07 1609.18 985.352 c
1600.98 983.75 1597.38 980.941 y 1604.1 972.539 l 1608.87 968.121 1610.55 962.539 v 1612.23 956.91
1609.06 957.66 1601.99 951.48 c 1594.88 945.352 1584.53 942.5 y 1587.34 936.91 1591.25 932.969 v
1595.16 929.059 1602.46 926.289 y 1599.06 917.852 1594.06 914.531 v 1588.98 911.172 1575.94 911.559 y
1582.27 899.961 l 1585.59 893.828 1588.98 892.34 y 1590.16 890.121 1580.62 863.48 v 1571.09 836.762
1548.91 790.352 1534.38 773.199 c 1519.84 756.051 1498.75 747.422 1490.9 748.012 c 1483.09 748.59
1467.07 763.672 1458.09 772.66 c 1449.14 781.602 1448.05 791.48 1451.37 798.199 c 1454.77 804.879
1467.07 812.73 1474.34 818.32 c 1481.6 823.949 1481.41 825.621 1494.3 826.719 c 1507.15 827.852
1507.73 823.359 1507.73 814.961 c 1507.73 806.602 1505.47 803.789 1499.34 799.879 c 1493.12 795.98
1482.19 794.801 1478.28 791.48 c 1474.34 788.09 1475.78 778.238 1486.41 777.699 c 1497.11 777.148
1513.87 777.148 1524.84 791.988 c 1531.84 801.48 1540.51 820.551 1540.51 831.559 c 1540.51 848.672
1520.04 864.57 1511.64 871.602 c 1503.24 878.711 1493.71 881.129 1485.31 878.32 c 1476.91 875.59
1452.54 856.172 1429.22 836.559 c 1405.9 816.988 1401.8 813.828 1401.8 783.281 c 1401.8 752.73 1469.3
696.371 1489.8 696.371 c 1510.31 696.371 1523.16 711.879 1528.2 716.371 c 1533.24 720.781 1537.19
725.27 1543.87 720.27 c 1550.59 715.199 1568.32 705.16 1592.34 703.48 c 1616.41 701.762 1630.2
714.102 y 1612.85 714.648 1599.65 722.5 v 1590.51 727.891 1580.04 735.352 1575.04 738.711 c 1570
742.109 1570 744.648 1571.64 747.422 c 1573.32 750.23 1621.29 816.641 y 1624.8 816.09 1649.22 818.91
v 1669.77 821.25 1691.84 835.309 1698.4 839.73 c f Q%
}

\def\cord@tencentpath{%
q 1 0 0 -1 0 24 cm q q 0 0.321569 0.85098 rg 23.587 7.91 m 28.714 7.91 29.759 11.517 29.247 14.47 c
29.137 15.105 28.985 15.614 28.857 15.904 c 18.307 15.904 l 17.989 18.352 20.018 19.418 22.572 19.418
c 24.403 19.418 26.566 18.735 28.265 17.514 c 27.725 20.598 l 25.985 21.568 23.845 22 21.556 22 c
16.435 22 13.905 18.973 14.619 14.839 c 15.264 11.116 18.512 7.91 23.587 7.91 c h 71.474 7.91 m
76.604 7.91 77.648 11.517 77.139 14.47 c 77.061 14.961 76.93 15.442 76.749 15.904 c 66.198 15.904 l
65.881 18.352 67.908 19.418 70.461 19.418 c 72.294 19.418 74.456 18.735 76.155 17.514 c 75.615 20.598
l 73.875 21.568 71.734 22 69.446 22 c 64.326 22 61.794 18.973 62.509 14.839 c 63.154 11.116 66.401
7.91 71.474 7.91 c h 57.742 7.917 m 59.202 7.907 60.652 8.18 62.009 8.72 c 61.39 12.39 l 60.33 10.99
58.65 10.573 57.409 10.573 c 54.366 10.573 52.032 12.475 51.618 14.946 c 51.138 17.713 53.196 19.3
55.913 19.3 c 57.245 19.3 59.126 18.99 60.513 17.494 c 59.867 21.099 l 58.979 21.429 56.63 21.997
54.664 21.997 c 49.717 21.997 46.797 18.967 47.519 14.778 c 48.219 10.741 52.193 7.918 57.742 7.918 c
h 41.218 7.912 m 43.968 7.912 46.582 9.427 45.843 13.692 c 44.437 21.707 l 40.447 21.707 l 41.781
14.093 l 42.114 12.158 41.522 10.683 39.342 10.683 c 37.756 10.683 36.506 11.643 35.91 12.685 c
35.738 12.965 35.606 13.378 35.54 13.767 c 34.151 21.707 l 30.161 21.707 l 31.799 12.227 l 32.003
11.055 32.151 9.866 32.267 8.851 c 32.372 7.911 l 35.887 9.995 l 35.98 9.995 l 36.859 8.948 38.726
7.912 41.218 7.912 c h 89.107 7.912 m 91.857 7.912 94.471 9.427 93.732 13.692 c 92.328 21.707 l
88.336 21.707 l 89.67 14.093 l 90.005 12.158 89.413 10.683 87.231 10.683 c 85.646 10.683 84.395
11.643 83.799 12.685 c 83.627 12.965 83.494 13.378 83.429 13.767 c 82.04 21.707 l 78.048 21.707 l
79.687 12.227 l 79.916 10.908 80.076 9.568 80.198 8.479 c 80.261 7.912 l 83.776 9.995 l 83.868 9.995
l 84.718 8.984 86.486 7.983 88.85 7.915 c 89.107 7.912 l h 20.633 2 m 20.139 4.866 l 13.034 4.866 l
10.085 21.706 l 6.095 21.706 l 9.043 4.866 l 2 4.866 l 2.498 2 l 20.633 2 l h 99.323 4.414 m 107.184
10.758 l 102.206 10.758 l 100.286 21.707 l 96.294 21.707 l 98.209 10.758 l 95.93 10.758 l 94.162
8.215 l 98.657 8.215 l 99.323 4.414 l h 22.934 10.239 m 20.478 10.239 19.129 12.145 18.713 13.526 c
25.656 13.526 l 25.889 12.278 25.606 10.239 22.934 10.239 c h 70.822 10.239 m 68.369 10.239 67.019
12.145 66.603 13.526 c 73.546 13.526 l 73.574 13.351 l 73.749 12.096 73.375 10.239 70.822 10.239 c h
f* Q Q Q%
}

\def\cord@pkupath{%
q 1 0 0 -1 0 80 cm 0.571198 0.060394 0.077667 rg 26.246094 73.957031 m 26.191406 73.980469 l
25.863281 73.804688 25.460938 73.609375 24.980469 73.394531 c 24.507812 73.183594 24.007812 72.964844
23.480469 72.75 c 23.457031 72.691406 l 23.554688 72.476562 l 23.613281 72.457031 l 23.742188 72.5
23.867188 72.542969 23.984375 72.582031 c 24.199219 72.660156 24.335938 72.699219 24.40625 72.703125
c 24.457031 72.703125 24.511719 72.671875 24.570312 72.617188 c 24.667969 72.523438 24.898438
72.082031 25.253906 71.285156 c 25.660156 70.394531 l 25.78125 70.125 25.929688 69.785156 26.105469
69.371094 c 26.222656 69.097656 l 26.277344 68.980469 26.277344 68.910156 26.226562 68.886719 c
26.207031 68.878906 26.121094 68.871094 25.976562 68.871094 c 25.582031 68.875 25.242188 68.878906
24.957031 68.886719 c 24.925781 68.84375 l 24.972656 68.703125 25.015625 68.539062 25.054688
68.351562 c 25.070312 68.328125 l 25.816406 68.441406 26.667969 68.527344 27.621094 68.582031 c
27.671875 68.667969 l 27.601562 68.804688 27.492188 69.011719 27.351562 69.28125 c 27.292969
69.378906 27.160156 69.65625 26.945312 70.109375 c 26.296875 71.515625 l 25.839844 72.523438
25.613281 73.070312 25.613281 73.160156 c 25.609375 73.222656 25.628906 73.277344 25.675781 73.304688
c 25.746094 73.371094 25.960938 73.496094 26.316406 73.679688 c 26.34375 73.734375 l 26.246094
73.957031 l f 37.097656 74.789062 m 37.054688 75.09375 36.960938 75.359375 36.804688 75.574219 c
36.511719 75.996094 36.09375 76.253906 35.546875 76.347656 c 35.324219 76.390625 35.058594 76.386719
34.75 76.347656 c 34.089844 76.257812 33.613281 75.964844 33.320312 75.472656 c 33.1875 75.253906
33.144531 74.980469 33.1875 74.65625 c 33.242188 74.261719 33.425781 73.957031 33.742188 73.746094 c
33.847656 73.679688 34.066406 73.585938 34.402344 73.476562 c 33.859375 73.101562 33.625 72.644531
33.695312 72.101562 c 33.777344 71.492188 34.144531 71.089844 34.796875 70.886719 c 35.082031
70.800781 35.394531 70.78125 35.738281 70.824219 c 36.234375 70.890625 36.617188 71.058594 36.875
71.320312 c 37.121094 71.570312 37.21875 71.871094 37.175781 72.214844 c 37.109375 72.683594 36.78125
73.03125 36.1875 73.25 c 36.574219 73.460938 36.835938 73.691406 36.972656 73.945312 c 37.101562
74.167969 37.140625 74.453125 37.097656 74.789062 c h 36.0625 75.109375 m 36.113281 74.734375
35.960938 74.421875 35.613281 74.183594 c 35.386719 74.027344 35.140625 73.878906 34.871094 73.734375
c 34.757812 73.679688 l 34.308594 73.90625 34.054688 74.238281 33.996094 74.671875 c 33.960938
74.921875 33.984375 75.125 34.070312 75.292969 c 34.25 75.667969 34.546875 75.886719 34.960938
75.941406 c 35.25 75.980469 35.496094 75.917969 35.695312 75.769531 c 35.902344 75.609375 36.027344
75.390625 36.0625 75.109375 c h 36.417969 72.226562 m 36.449219 71.96875 36.394531 71.753906
36.246094 71.570312 c 36.09375 71.394531 35.875 71.285156 35.582031 71.25 c 35.320312 71.210938
35.101562 71.257812 34.929688 71.382812 c 34.765625 71.5 34.675781 71.660156 34.644531 71.863281 c
34.621094 72.058594 34.6875 72.257812 34.851562 72.441406 c 34.941406 72.542969 35.078125 72.644531
35.257812 72.753906 c 35.320312 72.800781 35.417969 72.851562 35.542969 72.929688 c 35.769531
73.054688 l 36.148438 72.894531 36.363281 72.617188 36.417969 72.226562 c f 46.347656 72.785156 m
46.445312 73.484375 46.382812 74.074219 46.171875 74.566406 c 45.949219 75.085938 45.617188 75.523438
45.175781 75.871094 c 44.722656 76.230469 44.199219 76.453125 43.605469 76.535156 c 43.226562
76.324219 l 43.3125 76.160156 l 43.402344 76.171875 43.507812 76.171875 43.640625 76.15625 c 44.09375
76.09375 44.46875 75.886719 44.769531 75.535156 c 44.886719 75.394531 44.992188 75.214844 45.078125
74.996094 c 45.222656 74.652344 45.285156 74.160156 45.269531 73.53125 c 44.886719 73.949219
44.667969 74.167969 44.621094 74.179688 c 44.4375 74.269531 44.234375 74.324219 44.011719 74.359375 c
43.621094 74.410156 43.292969 74.324219 43.015625 74.105469 c 42.714844 73.863281 42.527344 73.507812
42.464844 73.03125 c 42.355469 72.265625 42.578125 71.660156 43.125 71.210938 c 43.382812 71
43.714844 70.863281 44.132812 70.808594 c 44.910156 70.699219 45.507812 70.953125 45.929688 71.554688
c 46.136719 71.855469 46.273438 72.265625 46.347656 72.785156 c h 45.15625 72.539062 m 45.085938
72.054688 44.957031 71.707031 44.769531 71.496094 c 44.578125 71.285156 44.347656 71.199219 44.078125
71.230469 c 43.746094 71.28125 43.539062 71.472656 43.46875 71.820312 c 43.425781 72.003906 43.421875
72.214844 43.457031 72.457031 c 43.515625 72.859375 43.640625 73.160156 43.835938 73.351562 c
44.019531 73.539062 44.246094 73.609375 44.511719 73.570312 c 44.859375 73.527344 45.066406 73.355469
45.148438 73.066406 c 45.183594 72.933594 45.183594 72.757812 45.15625 72.539062 c f 55.597656
71.152344 m 55.722656 71.433594 55.777344 71.707031 55.757812 71.972656 c 55.722656 72.484375
55.496094 72.921875 55.078125 73.285156 c 54.90625 73.4375 54.679688 73.574219 54.394531 73.695312 c
53.785156 73.960938 53.230469 73.960938 52.726562 73.691406 c 52.496094 73.570312 52.316406 73.359375
52.1875 73.066406 c 52.027344 72.699219 52.027344 72.335938 52.1875 71.996094 c 52.242188 71.882812
52.382812 71.691406 52.613281 71.421875 c 51.949219 71.382812 51.511719 71.113281 51.296875 70.617188
c 51.046875 70.050781 51.152344 69.511719 51.609375 69.003906 c 51.808594 68.78125 52.0625 68.601562
52.378906 68.460938 c 52.839844 68.261719 53.253906 68.203125 53.609375 68.296875 c 53.949219
68.386719 54.191406 68.589844 54.328125 68.910156 c 54.519531 69.34375 54.414062 69.8125 54.019531
70.308594 c 54.460938 70.285156 54.804688 70.34375 55.050781 70.488281 c 55.277344 70.617188
55.457031 70.835938 55.597656 71.152344 c h 53.6875 69.308594 m 53.582031 69.074219 53.417969
68.917969 53.199219 68.84375 c 52.976562 68.769531 52.734375 68.789062 52.46875 68.910156 c 52.226562
69.011719 52.0625 69.164062 51.980469 69.363281 c 51.902344 69.546875 51.902344 69.730469 51.984375
69.914062 c 52.066406 70.101562 52.226562 70.226562 52.460938 70.308594 c 52.589844 70.34375
52.757812 70.367188 52.96875 70.367188 c 53.042969 70.367188 53.15625 70.367188 53.304688 70.367188 c
53.558594 70.351562 l 53.804688 70.023438 53.84375 69.671875 53.6875 69.308594 c h 54.878906
71.960938 m 54.730469 71.613281 54.4375 71.425781 54.015625 71.398438 c 53.738281 71.382812 53.449219
71.386719 53.148438 71.40625 c 53.023438 71.410156 l 52.753906 71.835938 52.707031 72.257812
52.882812 72.652344 c 52.984375 72.882812 53.113281 73.050781 53.269531 73.152344 c 53.621094
73.382812 53.984375 73.410156 54.367188 73.242188 c 54.632812 73.125 54.816406 72.945312 54.910156
72.707031 c 55.003906 72.46875 54.992188 72.21875 54.878906 71.960938 c f 24.328125 3.207031 m
19.582031 5.210938 15.324219 8.085938 11.667969 11.742188 c 8.007812 15.398438 5.136719 19.660156
3.132812 24.40625 c 1.054688 29.316406 0 34.53125 0 39.910156 c 0 45.289062 1.054688 50.507812
3.132812 55.414062 c 5.136719 60.160156 8.007812 64.414062 11.667969 68.074219 c 15.324219 71.734375
19.582031 74.605469 24.328125 76.609375 c 29.238281 78.6875 34.453125 79.738281 39.832031 79.738281 c
45.207031 79.738281 50.425781 78.6875 55.335938 76.609375 c 60.078125 74.605469 64.339844 71.734375
67.996094 68.074219 c 71.65625 64.414062 74.527344 60.160156 76.53125 55.414062 c 78.609375 50.507812
79.664062 45.289062 79.664062 39.910156 c 79.664062 34.53125 78.609375 29.316406 76.53125 24.40625 c
74.527344 19.660156 71.65625 15.398438 67.996094 11.742188 c 64.339844 8.085938 60.078125 5.210938
55.335938 3.207031 c 50.425781 1.132812 45.207031 0.074219 39.832031 0.074219 c 34.453125 0.074219
29.238281 1.132812 24.328125 3.207031 c h 24.839844 75.394531 m 20.253906 73.457031 16.136719
70.679688 12.601562 67.140625 c 9.058594 63.609375 6.28125 59.484375 4.34375 54.902344 c 2.335938
50.15625 1.320312 45.109375 1.320312 39.910156 c 1.320312 34.710938 2.335938 29.667969 4.34375
24.914062 c 6.28125 20.328125 9.058594 16.210938 12.601562 12.675781 c 16.136719 9.136719 20.253906
6.355469 24.839844 4.417969 c 29.585938 2.410156 34.628906 1.390625 39.832031 1.390625 c 45.03125
1.390625 50.074219 2.410156 54.820312 4.417969 c 59.410156 6.355469 63.527344 9.136719 67.066406
12.675781 c 70.605469 16.210938 73.378906 20.328125 75.320312 24.914062 c 77.328125 29.667969
78.347656 34.710938 78.347656 39.910156 c 78.347656 45.109375 77.328125 50.15625 75.320312 54.902344
c 73.378906 59.484375 70.605469 63.609375 67.066406 67.140625 c 63.527344 70.679688 59.410156
73.457031 54.820312 75.394531 c 50.074219 77.40625 45.03125 78.417969 39.832031 78.417969 c 34.628906
78.417969 29.585938 77.40625 24.839844 75.394531 c f 28.363281 12.765625 m 24.859375 14.246094
21.707031 16.371094 19.007812 19.078125 c 16.300781 21.785156 14.175781 24.933594 12.691406 28.441406
c 11.15625 32.074219 10.378906 35.929688 10.378906 39.90625 c 10.378906 43.882812 11.15625 47.738281
12.691406 51.371094 c 14.175781 54.878906 16.300781 58.027344 19.007812 60.734375 c 21.707031 63.4375
24.859375 65.5625 28.363281 67.042969 c 32 68.585938 35.851562 69.363281 39.828125 69.363281 c
43.804688 69.363281 47.664062 68.585938 51.296875 67.042969 c 54.804688 65.5625 57.957031 63.4375
60.660156 60.734375 c 63.367188 58.027344 65.484375 54.878906 66.972656 51.371094 c 68.507812
47.738281 69.289062 43.882812 69.289062 39.90625 c 69.289062 35.929688 68.507812 32.074219 66.972656
28.441406 c 65.484375 24.933594 63.367188 21.785156 60.660156 19.078125 c 57.957031 16.371094
54.804688 14.246094 51.296875 12.765625 c 47.664062 11.226562 43.804688 10.453125 39.828125 10.453125
c 35.851562 10.453125 32 11.226562 28.363281 12.765625 c h 28.714844 66.226562 m 25.3125 64.789062
22.257812 62.730469 19.632812 60.105469 c 17.007812 57.480469 14.949219 54.429688 13.511719 51.027344
c 12.023438 47.503906 11.265625 43.765625 11.265625 39.90625 c 11.265625 36.050781 12.023438
32.308594 13.511719 28.785156 c 14.949219 25.382812 17.007812 22.328125 19.632812 19.707031 c
22.257812 17.082031 25.3125 15.023438 28.714844 13.585938 c 32.234375 12.09375 35.972656 11.339844
39.828125 11.339844 c 43.683594 11.339844 47.429688 12.09375 50.953125 13.585938 c 54.351562
15.023438 57.40625 17.082031 60.03125 19.707031 c 62.65625 22.328125 64.710938 25.382812 66.152344
28.785156 c 67.644531 32.308594 68.394531 36.050781 68.394531 39.90625 c 68.394531 43.765625
67.644531 47.503906 66.152344 51.027344 c 64.710938 54.429688 62.65625 57.480469 60.03125 60.105469 c
57.40625 62.730469 54.351562 64.789062 50.953125 66.226562 c 47.429688 67.71875 43.683594 68.472656
39.828125 68.472656 c 35.972656 68.472656 32.234375 67.71875 28.714844 66.226562 c f 6.046875
50.109375 m 6.328125 50.003906 6.585938 49.96875 6.828125 50.007812 c 7.070312 50.039062 7.292969
50.136719 7.496094 50.285156 c 7.703125 50.445312 7.886719 50.648438 8.042969 50.890625 c 8.203125
51.132812 8.347656 51.417969 8.472656 51.746094 c 8.625 52.140625 l 10 51.621094 l 10.082031
51.589844 10.148438 51.546875 10.195312 51.488281 c 10.246094 51.433594 10.273438 51.355469 10.273438
51.253906 c 10.277344 51.199219 10.269531 51.113281 10.246094 51.003906 c 10.226562 50.890625
10.207031 50.789062 10.183594 50.71875 c 10.457031 50.613281 l 11.480469 53.296875 l 11.210938
53.398438 l 11.175781 53.335938 11.136719 53.25 11.078125 53.144531 c 11.023438 53.035156 10.980469
52.96875 10.941406 52.925781 c 10.871094 52.851562 10.804688 52.808594 10.738281 52.800781 c
10.671875 52.792969 10.597656 52.808594 10.5 52.839844 c 7.039062 54.160156 l 6.957031 54.195312
6.894531 54.234375 6.847656 54.289062 c 6.796875 54.335938 6.769531 54.414062 6.769531 54.523438 c
6.765625 54.597656 6.773438 54.691406 6.792969 54.808594 c 6.8125 54.925781 6.828125 55.011719
6.84375 55.066406 c 6.570312 55.167969 l 5.542969 52.464844 l 5.296875 51.835938 5.21875 51.320312
5.300781 50.929688 c 5.386719 50.546875 5.632812 50.269531 6.046875 50.109375 c h 6.730469 51.382812
m 6.367188 51.519531 6.117188 51.699219 6 51.917969 c 5.875 52.136719 5.886719 52.429688 6.03125
52.808594 c 6.140625 53.097656 l 8.292969 52.273438 l 8.234375 52.125 l 8.089844 51.753906 7.898438
51.503906 7.648438 51.371094 c 7.398438 51.238281 7.09375 51.242188 6.730469 51.382812 c f 4.652344
41.863281 m 4.683594 42.152344 l 4.5 42.242188 4.3125 42.386719 4.125 42.582031 c 3.941406 42.785156
3.84375 42.964844 3.828125 43.144531 c 3.820312 43.214844 3.820312 43.316406 3.828125 43.445312 c
3.839844 43.578125 3.851562 43.6875 3.863281 43.765625 c 3.949219 44.519531 l 5.960938 44.285156 l
5.898438 43.753906 l 5.882812 43.621094 5.855469 43.503906 5.808594 43.398438 c 5.765625 43.292969
5.691406 43.199219 5.589844 43.121094 c 5.53125 43.070312 5.421875 43.027344 5.273438 42.984375 c
5.128906 42.949219 5 42.925781 4.894531 42.917969 c 4.859375 42.632812 l 7.027344 42.378906 l
7.066406 42.667969 l 6.960938 42.691406 6.835938 42.738281 6.6875 42.824219 c 6.542969 42.902344
6.441406 42.964844 6.386719 43.027344 c 6.308594 43.117188 6.257812 43.21875 6.238281 43.328125 c
6.214844 43.445312 6.214844 43.570312 6.234375 43.714844 c 6.292969 44.25 l 7.800781 44.066406 l
7.957031 44.050781 8.082031 44.027344 8.171875 43.992188 c 8.257812 43.957031 8.324219 43.90625
8.367188 43.835938 c 8.40625 43.773438 8.425781 43.683594 8.429688 43.570312 c 8.433594 43.464844
8.421875 43.320312 8.402344 43.140625 c 8.390625 43.054688 8.378906 42.953125 8.359375 42.832031 c
8.347656 42.710938 8.324219 42.605469 8.304688 42.523438 c 8.28125 42.421875 8.25 42.328125 8.214844
42.238281 c 8.179688 42.15625 8.136719 42.089844 8.082031 42.054688 c 7.949219 41.949219 7.753906
41.839844 7.488281 41.71875 c 7.226562 41.601562 7.035156 41.527344 6.902344 41.496094 c 6.871094
41.230469 l 8.503906 41.140625 l 9.078125 46.03125 l 8.789062 46.066406 l 8.777344 45.996094 8.753906
45.902344 8.726562 45.789062 c 8.699219 45.667969 8.671875 45.585938 8.648438 45.539062 c 8.597656
45.449219 8.539062 45.398438 8.476562 45.367188 c 8.414062 45.34375 8.332031 45.335938 8.238281
45.351562 c 4.558594 45.777344 l 4.472656 45.789062 4.398438 45.804688 4.34375 45.84375 c 4.289062
45.878906 4.246094 45.949219 4.214844 46.054688 c 4.191406 46.121094 4.175781 46.214844 4.164062
46.335938 c 4.152344 46.457031 4.148438 46.542969 4.148438 46.605469 c 3.859375 46.640625 l 3.316406
42.019531 l 4.652344 41.863281 l f 9.363281 32.152344 m 9.070312 34.207031 l 8.546875 34.445312
8.109375 34.667969 7.753906 34.871094 c 7.398438 35.074219 6.949219 35.332031 6.410156 35.652344 c
6.382812 35.839844 l 8.019531 36.074219 l 8.109375 36.085938 8.183594 36.082031 8.25 36.0625 c
8.316406 36.035156 8.375 35.980469 8.433594 35.890625 c 8.457031 35.851562 8.484375 35.789062
8.515625 35.707031 c 8.550781 35.625 8.578125 35.539062 8.597656 35.453125 c 8.886719 35.496094 l
8.492188 38.257812 l 8.203125 38.214844 l 8.210938 38.152344 8.214844 38.058594 8.222656 37.933594 c
8.230469 37.8125 8.226562 37.726562 8.214844 37.671875 c 8.191406 37.585938 8.152344 37.511719
8.097656 37.46875 c 8.042969 37.425781 7.96875 37.402344 7.875 37.378906 c 4.1875 36.859375 l
4.101562 36.847656 4.027344 36.847656 3.964844 36.871094 c 3.902344 36.890625 3.84375 36.949219
3.78125 37.039062 c 3.734375 37.105469 3.695312 37.199219 3.660156 37.300781 c 3.621094 37.402344
3.59375 37.488281 3.578125 37.554688 c 3.292969 37.511719 l 3.699219 34.660156 l 3.984375 34.703125 l
3.980469 34.777344 3.980469 34.871094 3.976562 34.976562 c 3.976562 35.085938 3.980469 35.179688
3.996094 35.253906 c 4.011719 35.355469 4.054688 35.425781 4.121094 35.464844 c 4.183594 35.507812
4.257812 35.535156 4.339844 35.546875 c 6.097656 35.796875 l 6.117188 35.644531 l 5.96875 35.433594
5.792969 35.179688 5.585938 34.890625 c 5.375 34.605469 5.203125 34.378906 5.066406 34.214844 c 4.875
33.996094 4.722656 33.824219 4.609375 33.714844 c 4.496094 33.605469 4.390625 33.542969 4.300781
33.53125 c 4.257812 33.523438 4.21875 33.566406 4.183594 33.644531 c 4.148438 33.730469 4.113281
33.867188 4.078125 34.050781 c 3.789062 34.007812 l 4.113281 31.75 l 4.398438 31.792969 l 4.390625
32.066406 4.410156 32.304688 4.453125 32.496094 c 4.496094 32.691406 4.585938 32.886719 4.722656
33.09375 c 4.921875 33.386719 5.09375 33.628906 5.242188 33.820312 c 5.386719 34.015625 5.578125
34.261719 5.808594 34.5625 c 6.128906 34.375 6.519531 34.140625 6.980469 33.875 c 7.441406 33.605469
7.890625 33.335938 8.332031 33.066406 c 8.441406 32.996094 8.554688 32.917969 8.660156 32.824219 c
8.765625 32.730469 8.859375 32.621094 8.929688 32.503906 c 8.96875 32.445312 8.996094 32.378906
9.023438 32.308594 c 9.050781 32.230469 9.066406 32.167969 9.078125 32.113281 c 9.363281 32.152344 l
f 11.347656 26.941406 m 10.261719 29.648438 l 9.992188 29.542969 l 10.015625 29.476562 10.042969
29.382812 10.082031 29.257812 c 10.121094 29.136719 10.140625 29.050781 10.140625 28.992188 c
10.140625 28.894531 10.121094 28.820312 10.074219 28.769531 c 10.03125 28.714844 9.964844 28.667969
9.878906 28.632812 c 6.421875 27.25 l 6.34375 27.222656 6.273438 27.203125 6.199219 27.214844 c
6.128906 27.222656 6.058594 27.253906 5.988281 27.328125 c 5.9375 27.375 5.882812 27.445312 5.816406
27.546875 c 5.75 27.644531 5.695312 27.730469 5.65625 27.800781 c 5.386719 27.695312 l 6.46875
24.988281 l 6.738281 25.09375 l 6.71875 25.167969 6.691406 25.257812 6.664062 25.367188 c 6.636719
25.472656 6.621094 25.566406 6.613281 25.648438 c 6.609375 25.75 6.628906 25.824219 6.679688
25.882812 c 6.730469 25.9375 6.796875 25.980469 6.878906 26.007812 c 10.332031 27.390625 l 10.414062
27.429688 10.492188 27.441406 10.558594 27.429688 c 10.632812 27.425781 10.703125 27.386719 10.773438
27.3125 c 10.8125 27.269531 10.867188 27.203125 10.929688 27.101562 c 10.992188 26.996094 11.042969
26.902344 11.078125 26.832031 c 11.347656 26.941406 l f 12.195312 16.285156 m 12.160156 16.347656
12.113281 16.429688 12.054688 16.550781 c 11.996094 16.664062 11.96875 16.78125 11.972656 16.910156 c
11.972656 17.027344 12.027344 17.144531 12.136719 17.269531 c 12.246094 17.398438 12.429688 17.550781
12.683594 17.730469 c 15.742188 19.925781 l 15.355469 20.460938 l 10.152344 20.746094 l 12.152344
22.175781 l 12.402344 22.359375 12.605469 22.476562 12.753906 22.539062 c 12.90625 22.597656
13.042969 22.609375 13.152344 22.578125 c 13.238281 22.554688 13.339844 22.5 13.457031 22.40625 c
13.578125 22.308594 13.667969 22.234375 13.726562 22.171875 c 13.964844 22.339844 l 12.648438
24.171875 l 12.414062 24.003906 l 12.460938 23.921875 12.511719 23.820312 12.5625 23.707031 c
12.613281 23.589844 12.636719 23.476562 12.636719 23.378906 c 12.632812 23.253906 12.582031 23.136719
12.488281 23.027344 c 12.390625 22.914062 12.203125 22.757812 11.925781 22.558594 c 9.90625 21.105469
l 9.804688 21.035156 9.695312 20.980469 9.585938 20.949219 c 9.472656 20.914062 9.359375 20.910156
9.25 20.925781 c 9.132812 20.949219 9.023438 20.992188 8.917969 21.054688 c 8.8125 21.117188 8.710938
21.199219 8.605469 21.300781 c 8.367188 21.128906 l 9.519531 19.523438 l 14.085938 19.277344 l
12.460938 18.113281 l 12.203125 17.929688 11.996094 17.808594 11.847656 17.753906 c 11.699219
17.691406 11.570312 17.679688 11.453125 17.707031 c 11.363281 17.734375 11.265625 17.789062 11.152344
17.882812 c 11.035156 17.972656 10.945312 18.042969 10.882812 18.105469 c 10.65625 17.945312 l
11.960938 16.121094 l 12.195312 16.285156 l f 21.292969 11.960938 m 21.238281 12.015625 21.191406
12.066406 21.148438 12.109375 c 21.109375 12.152344 21.066406 12.214844 21.015625 12.292969 c
20.953125 12.382812 20.9375 12.460938 20.960938 12.542969 c 20.980469 12.617188 21.015625 12.6875
21.066406 12.753906 c 21.449219 13.230469 l 21.527344 13.324219 21.59375 13.402344 21.640625
13.457031 c 21.691406 13.515625 21.75 13.585938 21.824219 13.671875 c 21.722656 13.9375 21.527344
14.28125 21.234375 14.679688 c 20.941406 15.082031 20.636719 15.414062 20.320312 15.664062 c
20.023438 15.90625 19.710938 16.082031 19.378906 16.199219 c 19.046875 16.320312 18.710938 16.367188
18.371094 16.34375 c 18.03125 16.320312 17.699219 16.222656 17.371094 16.046875 c 17.046875 15.871094
16.746094 15.613281 16.472656 15.277344 c 16.21875 14.953125 16.03125 14.613281 15.921875 14.261719 c
15.804688 13.90625 15.773438 13.550781 15.816406 13.191406 c 15.855469 12.839844 15.972656 12.5
16.167969 12.167969 c 16.367188 11.835938 16.636719 11.535156 16.976562 11.257812 c 17.207031
11.074219 17.445312 10.925781 17.6875 10.824219 c 17.933594 10.722656 18.144531 10.652344 18.324219
10.617188 c 18.378906 10.3125 l 18.628906 10.105469 l 19.859375 11.582031 l 19.605469 11.785156 l
19.097656 11.449219 18.65625 11.261719 18.277344 11.207031 c 17.902344 11.15625 17.558594 11.253906
17.242188 11.503906 c 16.875 11.796875 16.71875 12.1875 16.777344 12.671875 c 16.839844 13.15625
17.09375 13.679688 17.546875 14.246094 c 18.027344 14.847656 18.507812 15.238281 18.976562 15.417969
c 19.445312 15.597656 19.859375 15.546875 20.21875 15.257812 c 20.386719 15.121094 20.515625
14.976562 20.609375 14.820312 c 20.695312 14.664062 20.75 14.523438 20.769531 14.402344 c 20.105469
13.570312 l 20.050781 13.503906 19.992188 13.449219 19.929688 13.429688 c 19.867188 13.398438
19.785156 13.398438 19.683594 13.429688 c 19.601562 13.449219 19.511719 13.484375 19.414062 13.535156
c 19.3125 13.585938 19.226562 13.632812 19.164062 13.671875 c 18.984375 13.449219 l 21.113281
11.734375 l 21.292969 11.960938 l f 33.523438 4.101562 m 33.457031 4.128906 33.359375 4.171875
33.234375 4.234375 c 33.109375 4.292969 33.011719 4.367188 32.941406 4.453125 c 32.863281 4.535156
32.828125 4.660156 32.832031 4.824219 c 32.832031 4.984375 32.878906 5.222656 32.960938 5.535156 c
33.464844 7.417969 l 33.613281 7.972656 33.539062 8.464844 33.253906 8.902344 c 32.964844 9.34375
32.515625 9.644531 31.894531 9.816406 c 31.167969 10.003906 30.566406 9.992188 30.082031 9.773438 c
29.597656 9.550781 29.277344 9.152344 29.125 8.570312 c 28.429688 5.980469 l 28.410156 5.894531
28.375 5.828125 28.328125 5.777344 c 28.28125 5.726562 28.207031 5.691406 28.101562 5.664062 c
28.03125 5.648438 27.945312 5.648438 27.851562 5.648438 c 27.753906 5.660156 27.679688 5.664062
27.621094 5.671875 c 27.546875 5.390625 l 30.207031 4.679688 l 30.28125 4.960938 l 30.207031 4.984375
30.128906 5.015625 30.054688 5.050781 c 29.980469 5.085938 29.90625 5.125 29.832031 5.171875 c
29.746094 5.222656 29.699219 5.289062 29.683594 5.363281 c 29.671875 5.441406 29.675781 5.519531
29.699219 5.605469 c 30.378906 8.15625 l 30.515625 8.660156 30.722656 8.996094 31.003906 9.164062 c
31.289062 9.335938 31.644531 9.359375 32.070312 9.25 c 32.480469 9.136719 32.777344 8.925781
32.945312 8.609375 c 33.113281 8.296875 33.132812 7.890625 33.003906 7.402344 c 32.546875 5.683594 l
32.464844 5.378906 32.382812 5.160156 32.300781 5.015625 c 32.21875 4.878906 32.125 4.78125 32.019531
4.738281 c 31.945312 4.707031 31.824219 4.6875 31.660156 4.6875 c 31.492188 4.683594 31.375 4.6875
31.3125 4.695312 c 31.234375 4.40625 l 33.449219 3.816406 l 33.523438 4.101562 l f 43.445312 3.730469
m 43.375 3.734375 43.277344 3.746094 43.148438 3.761719 c 43.019531 3.78125 42.902344 3.828125
42.800781 3.902344 c 42.707031 3.96875 42.640625 4.082031 42.601562 4.246094 c 42.558594 4.402344
42.539062 4.640625 42.535156 4.957031 c 42.484375 8.71875 l 41.824219 8.714844 l 38.617188 4.601562 l
38.582031 7.066406 l 38.578125 7.371094 38.597656 7.605469 38.632812 7.761719 c 38.671875 7.925781
38.738281 8.039062 38.828125 8.113281 c 38.894531 8.171875 39 8.21875 39.144531 8.261719 c 39.289062
8.308594 39.402344 8.339844 39.492188 8.355469 c 39.488281 8.644531 l 37.234375 8.609375 l 37.234375
8.328125 l 37.332031 8.316406 37.4375 8.296875 37.5625 8.273438 c 37.691406 8.246094 37.792969
8.203125 37.878906 8.148438 c 37.976562 8.074219 38.042969 7.972656 38.082031 7.824219 c 38.121094
7.683594 38.140625 7.441406 38.140625 7.097656 c 38.175781 4.605469 l 38.179688 4.480469 38.164062
4.367188 38.125 4.253906 c 38.089844 4.144531 38.03125 4.050781 37.953125 3.96875 c 37.867188
3.882812 37.769531 3.816406 37.65625 3.765625 c 37.546875 3.71875 37.417969 3.675781 37.277344
3.648438 c 37.277344 3.355469 l 39.253906 3.386719 l 42.066406 6.988281 l 42.09375 4.992188 l
42.097656 4.675781 42.082031 4.441406 42.042969 4.28125 c 42.003906 4.125 41.9375 4.011719 41.847656
3.933594 c 41.78125 3.878906 41.671875 3.828125 41.535156 3.785156 c 41.394531 3.738281 41.28125
3.714844 41.195312 3.691406 c 41.203125 3.414062 l 43.449219 3.441406 l 43.445312 3.730469 l f
49.101562 10.011719 m 46.277344 9.285156 l 46.351562 9 l 46.414062 9.015625 46.511719 9.035156
46.640625 9.054688 c 46.765625 9.078125 46.855469 9.085938 46.910156 9.078125 c 47.007812 9.066406
47.078125 9.035156 47.128906 8.980469 c 47.175781 8.933594 47.210938 8.859375 47.234375 8.773438 c
48.167969 5.167969 l 48.1875 5.085938 48.191406 5.007812 48.175781 4.9375 c 48.160156 4.871094
48.113281 4.804688 48.03125 4.742188 c 47.984375 4.703125 47.90625 4.660156 47.796875 4.601562 c
47.695312 4.550781 47.601562 4.503906 47.519531 4.476562 c 47.59375 4.199219 l 50.417969 4.929688 l
50.34375 5.207031 l 50.269531 5.195312 50.179688 5.179688 50.070312 5.167969 c 49.957031 5.152344
49.863281 5.148438 49.78125 5.152344 c 49.679688 5.160156 49.609375 5.195312 49.558594 5.25 c
49.511719 5.308594 49.476562 5.378906 49.453125 5.460938 c 48.519531 9.0625 l 48.5 9.148438 48.5
9.226562 48.511719 9.292969 c 48.527344 9.363281 48.578125 9.429688 48.660156 9.492188 c 48.703125
9.527344 48.78125 9.570312 48.890625 9.621094 c 49 9.667969 49.097656 9.707031 49.171875 9.734375 c
49.101562 10.011719 l f 59.441406 9.308594 m 59.367188 9.277344 59.292969 9.253906 59.214844 9.242188
c 59.136719 9.222656 59.054688 9.21875 58.964844 9.222656 c 58.84375 9.226562 58.738281 9.246094
58.65625 9.285156 c 58.570312 9.320312 58.472656 9.371094 58.367188 9.457031 c 58.21875 9.574219
58.007812 9.742188 57.734375 9.96875 c 57.457031 10.195312 57.144531 10.441406 56.792969 10.730469 c
56.480469 10.988281 56.136719 11.273438 55.765625 11.582031 c 55.390625 11.886719 55.035156 12.1875
54.695312 12.464844 c 54.042969 12.121094 l 54.132812 11.210938 54.226562 10.289062 54.316406
9.347656 c 54.414062 8.410156 54.476562 7.757812 54.503906 7.386719 c 54.511719 7.253906 54.507812
7.144531 54.488281 7.066406 c 54.46875 6.984375 54.429688 6.886719 54.355469 6.773438 c 54.304688
6.695312 54.253906 6.628906 54.207031 6.582031 c 54.160156 6.539062 54.089844 6.484375 54.003906
6.421875 c 54.140625 6.164062 l 56.65625 7.496094 l 56.523438 7.746094 l 56.269531 7.648438 56.089844
7.585938 55.984375 7.574219 c 55.882812 7.5625 55.820312 7.578125 55.792969 7.625 c 55.78125 7.652344
55.773438 7.679688 55.765625 7.707031 c 55.757812 7.738281 55.75 7.78125 55.738281 7.839844 c
55.707031 8.070312 55.667969 8.449219 55.617188 8.988281 c 55.570312 9.523438 55.492188 10.34375
55.375 11.429688 c 55.890625 11 56.320312 10.632812 56.671875 10.332031 c 57.015625 10.03125
57.324219 9.765625 57.585938 9.535156 c 57.71875 9.421875 57.835938 9.316406 57.933594 9.210938 c
58.035156 9.113281 58.101562 9.035156 58.132812 8.96875 c 58.171875 8.902344 58.148438 8.816406
58.070312 8.714844 c 57.996094 8.617188 57.824219 8.476562 57.5625 8.3125 c 57.699219 8.050781 l
59.582031 9.046875 l 59.441406 9.308594 l f 65.152344 15.414062 m 64.941406 15.214844 l 65.027344
15.019531 65.074219 14.789062 65.082031 14.519531 c 65.089844 14.253906 65.046875 14.046875 64.945312
13.90625 c 64.90625 13.839844 64.832031 13.765625 64.738281 13.679688 c 64.648438 13.582031 64.566406
13.511719 64.503906 13.449219 c 63.949219 12.941406 l 62.582031 14.433594 l 62.976562 14.800781 l
63.074219 14.886719 63.175781 14.957031 63.277344 15.007812 c 63.378906 15.058594 63.496094 15.082031
63.625 15.074219 c 63.703125 15.070312 63.8125 15.035156 63.949219 14.96875 c 64.085938 14.902344
64.199219 14.835938 64.285156 14.769531 c 64.496094 14.960938 l 63.023438 16.578125 l 62.808594
16.382812 l 62.867188 16.300781 62.929688 16.175781 62.988281 16.019531 c 63.050781 15.863281
63.082031 15.742188 63.082031 15.664062 c 63.082031 15.546875 63.054688 15.4375 62.996094 15.335938 c
62.9375 15.238281 62.855469 15.144531 62.746094 15.042969 c 62.351562 14.683594 l 61.332031 15.804688
l 61.222656 15.917969 61.152344 16.019531 61.101562 16.105469 c 61.058594 16.1875 61.042969 16.269531
61.058594 16.351562 c 61.066406 16.421875 61.113281 16.507812 61.179688 16.585938 c 61.25 16.671875
61.351562 16.777344 61.488281 16.902344 c 61.550781 16.957031 61.628906 17.027344 61.71875 17.109375
c 61.808594 17.1875 61.890625 17.25 61.96875 17.304688 c 62.050781 17.367188 62.136719 17.417969
62.214844 17.460938 c 62.304688 17.503906 62.375 17.519531 62.441406 17.515625 c 62.605469 17.507812
62.828125 17.460938 63.105469 17.375 c 63.378906 17.289062 63.578125 17.222656 63.695312 17.160156 c
63.890625 17.339844 l 62.714844 18.476562 l 59.085938 15.160156 l 59.28125 14.949219 l 59.335938
14.988281 59.410156 15.042969 59.507812 15.117188 c 59.609375 15.183594 59.679688 15.226562 59.730469
15.246094 c 59.828125 15.285156 59.90625 15.289062 59.96875 15.265625 c 60.035156 15.246094 60.101562
15.195312 60.164062 15.125 c 62.664062 12.394531 l 62.722656 12.328125 62.761719 12.265625 62.78125
12.203125 c 62.800781 12.140625 62.785156 12.058594 62.742188 11.949219 c 62.714844 11.886719
62.664062 11.808594 62.59375 11.710938 c 62.527344 11.613281 62.472656 11.542969 62.433594 11.492188
c 62.625 11.285156 l 66.058594 14.417969 l 65.152344 15.414062 l f 67.878906 25.722656 m 66.789062
24.027344 l 67.046875 23.449219 67.246094 22.976562 67.382812 22.613281 c 67.523438 22.25 67.675781
21.808594 67.835938 21.300781 c 67.730469 21.136719 l 66.398438 22 l 66.320312 22.042969 66.265625
22.097656 66.234375 22.160156 c 66.199219 22.222656 66.191406 22.304688 66.207031 22.40625 c 66.21875
22.453125 66.238281 22.527344 66.277344 22.625 c 66.3125 22.71875 66.347656 22.796875 66.378906
22.859375 c 66.132812 23.015625 l 64.609375 20.65625 l 64.855469 20.496094 l 64.898438 20.554688
64.960938 20.628906 65.035156 20.722656 c 65.113281 20.8125 65.171875 20.875 65.214844 20.902344 c
65.300781 20.964844 65.375 20.988281 65.441406 20.984375 c 65.511719 20.980469 65.585938 20.949219
65.664062 20.898438 c 68.78125 18.890625 l 68.851562 18.84375 68.90625 18.789062 68.9375 18.734375 c
68.980469 18.679688 68.984375 18.59375 68.96875 18.484375 c 68.949219 18.398438 68.921875 18.308594
68.878906 18.203125 c 68.839844 18.101562 68.808594 18.019531 68.78125 17.964844 c 69.027344
17.808594 l 70.625 20.28125 l 70.980469 20.839844 71.164062 21.316406 71.179688 21.714844 c 71.191406
22.113281 71.019531 22.425781 70.667969 22.65625 c 70.359375 22.855469 70.050781 22.925781 69.75
22.871094 c 69.445312 22.8125 69.125 22.636719 68.785156 22.34375 c 68.691406 22.601562 68.589844
22.902344 68.472656 23.246094 c 68.347656 23.585938 68.230469 23.9375 68.113281 24.300781 c 68.070312
24.414062 68.03125 24.570312 67.988281 24.757812 c 67.941406 24.941406 67.933594 25.085938 67.960938
25.191406 c 67.980469 25.246094 68 25.3125 68.042969 25.394531 c 68.078125 25.476562 68.105469
25.53125 68.121094 25.5625 c h 69.800781 21.496094 m 70.097656 21.304688 70.257812 21.082031
70.285156 20.828125 c 70.316406 20.578125 70.234375 20.300781 70.035156 19.996094 c 69.878906 19.75 l
68.035156 20.941406 l 68.183594 21.179688 l 68.378906 21.484375 68.613281 21.667969 68.878906
21.730469 c 69.148438 21.792969 69.457031 21.714844 69.800781 21.496094 c f 69.351562 30.035156 m
69.246094 29.730469 69.191406 29.429688 69.175781 29.144531 c 69.160156 28.855469 69.175781 28.601562
69.21875 28.386719 c 68.890625 28.222656 l 68.792969 27.925781 l 70.578125 27.28125 l 70.679688
27.585938 l 70.53125 27.714844 70.382812 27.859375 70.230469 28.023438 c 70.074219 28.191406
69.949219 28.375 69.847656 28.5625 c 69.738281 28.753906 69.664062 28.957031 69.625 29.171875 c
69.589844 29.378906 69.605469 29.601562 69.6875 29.835938 c 69.792969 30.15625 69.949219 30.371094
70.160156 30.492188 c 70.371094 30.609375 70.597656 30.628906 70.832031 30.550781 c 71.027344
30.488281 71.164062 30.359375 71.246094 30.171875 c 71.328125 29.984375 71.363281 29.722656 71.359375
29.390625 c 71.359375 29.171875 71.359375 28.96875 71.355469 28.785156 c 71.355469 28.601562
71.359375 28.429688 71.363281 28.265625 c 71.386719 27.886719 71.496094 27.570312 71.691406 27.328125
c 71.882812 27.082031 72.132812 26.902344 72.441406 26.808594 c 72.632812 26.742188 72.832031
26.71875 73.039062 26.746094 c 73.246094 26.765625 73.453125 26.839844 73.648438 26.949219 c
73.832031 27.0625 74.007812 27.226562 74.171875 27.433594 c 74.339844 27.648438 74.464844 27.894531
74.5625 28.179688 c 74.652344 28.457031 74.703125 28.726562 74.707031 28.992188 c 74.714844 29.265625
74.699219 29.476562 74.671875 29.640625 c 74.949219 29.785156 l 75.054688 30.089844 l 73.363281
30.691406 l 73.261719 30.386719 l 73.417969 30.257812 73.578125 30.109375 73.734375 29.953125 c
73.890625 29.792969 74.015625 29.640625 74.105469 29.492188 c 74.203125 29.332031 74.261719 29.164062
74.296875 28.988281 c 74.324219 28.8125 74.308594 28.621094 74.238281 28.417969 c 74.152344 28.164062
74.011719 27.984375 73.820312 27.867188 c 73.621094 27.753906 73.417969 27.730469 73.210938 27.804688
c 73.015625 27.867188 72.875 27.992188 72.796875 28.171875 c 72.71875 28.355469 72.679688 28.613281
72.683594 28.9375 c 72.683594 29.132812 72.683594 29.328125 72.6875 29.53125 c 72.6875 29.730469
72.683594 29.90625 72.675781 30.066406 c 72.664062 30.449219 72.566406 30.769531 72.398438 31.023438
c 72.226562 31.28125 71.972656 31.464844 71.628906 31.578125 c 71.414062 31.648438 71.1875 31.667969
70.953125 31.636719 c 70.71875 31.601562 70.503906 31.519531 70.3125 31.394531 c 70.101562 31.253906
69.917969 31.078125 69.757812 30.855469 c 69.597656 30.636719 69.460938 30.363281 69.351562 30.035156
c f 71.070312 38.109375 m 70.773438 35.207031 l 71.058594 35.179688 l 71.070312 35.242188 71.089844
35.339844 71.109375 35.46875 c 71.136719 35.597656 71.15625 35.683594 71.183594 35.730469 c 71.226562
35.820312 71.28125 35.878906 71.347656 35.902344 c 71.410156 35.929688 71.488281 35.9375 71.582031
35.929688 c 75.285156 35.546875 l 75.367188 35.539062 75.441406 35.515625 75.503906 35.476562 c
75.5625 35.4375 75.609375 35.371094 75.636719 35.273438 c 75.65625 35.214844 75.671875 35.121094
75.683594 35 c 75.699219 34.882812 75.707031 34.78125 75.707031 34.699219 c 75.996094 34.671875 l
76.296875 37.570312 l 76.007812 37.601562 l 75.992188 37.527344 75.972656 37.4375 75.949219 37.332031
c 75.925781 37.222656 75.894531 37.128906 75.859375 37.054688 c 75.820312 36.964844 75.761719
36.90625 75.695312 36.878906 c 75.625 36.855469 75.542969 36.84375 75.457031 36.855469 c 71.761719
37.238281 l 71.667969 37.246094 71.597656 37.269531 71.535156 37.308594 c 71.476562 37.347656
71.433594 37.417969 71.402344 37.511719 c 71.386719 37.570312 71.375 37.65625 71.363281 37.777344 c
71.355469 37.898438 71.355469 38 71.359375 38.078125 c 71.070312 38.109375 l f 74.289062 46.59375 m
74.324219 46.308594 l 74.558594 46.261719 74.808594 46.167969 75.082031 46.027344 c 75.355469
45.890625 75.511719 45.753906 75.546875 45.613281 c 75.566406 45.539062 75.585938 45.441406 75.609375
45.328125 c 75.628906 45.210938 75.644531 45.113281 75.652344 45.039062 c 75.671875 44.859375 l
71.589844 44.375 l 71.5 44.367188 71.417969 44.371094 71.347656 44.394531 c 71.277344 44.421875
71.21875 44.476562 71.164062 44.566406 c 71.140625 44.617188 71.109375 44.703125 71.070312 44.820312
c 71.03125 44.941406 71.003906 45.039062 70.984375 45.117188 c 70.695312 45.085938 l 71.042969
42.167969 l 71.328125 42.199219 l 71.328125 42.273438 71.328125 42.378906 71.324219 42.503906 c
71.320312 42.636719 71.328125 42.722656 71.34375 42.773438 c 71.371094 42.871094 71.414062 42.949219
71.476562 42.984375 c 71.539062 43.027344 71.617188 43.046875 71.710938 43.0625 c 75.828125 43.550781
l 75.847656 43.375 l 75.859375 43.296875 75.867188 43.195312 75.871094 43.078125 c 75.882812
42.964844 75.882812 42.867188 75.882812 42.789062 c 75.882812 42.644531 75.761719 42.472656 75.527344
42.277344 c 75.292969 42.078125 75.070312 41.929688 74.855469 41.828125 c 74.890625 41.539062 l
76.386719 41.71875 l 75.785156 46.773438 l 74.289062 46.59375 l f 72.582031 55.6875 m 72.601562
55.601562 72.617188 55.515625 72.621094 55.445312 c 72.628906 55.367188 72.625 55.289062 72.609375
55.214844 c 72.59375 55.09375 72.554688 54.976562 72.503906 54.871094 c 72.445312 54.757812 72.375
54.644531 72.300781 54.527344 c 72.136719 54.296875 71.957031 54.050781 71.765625 53.789062 c
71.570312 53.523438 71.355469 53.238281 71.125 52.925781 c 70.972656 52.726562 70.855469 52.59375
70.761719 52.519531 c 70.664062 52.453125 70.488281 52.371094 70.230469 52.273438 c 69.539062
52.003906 l 69.457031 51.972656 69.375 51.957031 69.300781 51.96875 c 69.226562 51.976562 69.152344
52.015625 69.082031 52.089844 c 69.050781 52.125 69 52.195312 68.921875 52.316406 c 68.847656
52.429688 68.792969 52.519531 68.769531 52.578125 c 68.496094 52.472656 l 69.550781 49.726562 l
69.820312 49.832031 l 69.796875 49.898438 69.773438 49.996094 69.742188 50.121094 c 69.707031 50.25
69.691406 50.339844 69.6875 50.394531 c 69.695312 50.496094 69.722656 50.570312 69.769531 50.628906 c
69.820312 50.679688 69.890625 50.726562 69.972656 50.761719 c 70.839844 51.089844 l 70.914062
51.121094 70.980469 51.144531 71.042969 51.164062 c 71.097656 51.179688 71.152344 51.191406 71.195312
51.191406 c 71.246094 51.191406 71.308594 51.191406 71.394531 51.179688 c 71.476562 51.164062
71.585938 51.144531 71.714844 51.125 c 72.097656 51.070312 72.457031 51.015625 72.796875 50.949219 c
73.140625 50.894531 73.460938 50.835938 73.765625 50.78125 c 73.917969 50.753906 74.050781 50.710938
74.160156 50.660156 c 74.269531 50.613281 74.363281 50.554688 74.441406 50.492188 c 74.5 50.441406
74.554688 50.378906 74.609375 50.292969 c 74.660156 50.210938 74.703125 50.136719 74.734375 50.070312
c 75.003906 50.175781 l 73.929688 52.984375 l 73.65625 52.878906 l 73.75 52.617188 73.796875
52.433594 73.804688 52.335938 c 73.8125 52.234375 73.789062 52.167969 73.738281 52.152344 c 73.710938
52.140625 73.652344 52.140625 73.566406 52.140625 c 73.480469 52.148438 73.335938 52.164062 73.144531
52.183594 c 72.964844 52.203125 72.714844 52.242188 72.394531 52.289062 c 72.074219 52.34375
71.695312 52.40625 71.261719 52.484375 c 71.789062 53.1875 72.160156 53.667969 72.371094 53.9375 c
72.585938 54.207031 72.722656 54.355469 72.785156 54.378906 c 72.863281 54.40625 72.941406 54.386719
73.019531 54.304688 c 73.097656 54.226562 73.183594 54.082031 73.277344 53.871094 c 73.542969
53.972656 l 72.851562 55.789062 l 72.582031 55.6875 l f 66.925781 39.910156 m 66.925781 54.871094
54.796875 67 39.835938 67 c 24.867188 67 12.738281 54.871094 12.738281 39.910156 c 12.738281
24.941406 24.867188 12.8125 39.835938 12.8125 c 54.796875 12.8125 66.925781 24.941406 66.925781
39.910156 c h 56.671875 53.675781 m 55.773438 51.582031 54.75 49.648438 53.632812 47.9375 c 52.460938
46.132812 51.179688 44.558594 49.828125 43.269531 c 48.371094 41.878906 46.820312 40.800781 45.214844
40.0625 c 44.152344 39.574219 43.070312 39.234375 41.972656 39.042969 c 41.972656 37.785156 l
41.972656 36.601562 41.015625 35.640625 39.832031 35.640625 c 38.652344 35.640625 37.691406 36.601562
37.691406 37.785156 c 37.691406 39.042969 l 36.59375 39.234375 35.507812 39.574219 34.453125 40.0625
c 32.847656 40.800781 31.296875 41.878906 29.835938 43.269531 c 28.484375 44.558594 27.199219
46.132812 26.03125 47.9375 c 24.914062 49.648438 23.890625 51.582031 22.992188 53.675781 c 22.542969
54.730469 23.03125 55.945312 24.078125 56.402344 c 25.128906 56.847656 26.34375 56.359375 26.796875
55.3125 c 28.417969 51.519531 30.457031 48.394531 32.695312 46.261719 c 34.304688 44.726562 35.976562
43.726562 37.691406 43.273438 c 37.691406 46.640625 l 37.691406 47.800781 36.839844 48.253906
36.445312 48.492188 c 36.320312 48.566406 36.050781 48.730469 35.921875 48.808594 c 35.128906
49.324219 34.214844 50.121094 33.363281 51.214844 c 31.824219 53.191406 30.625 55.878906 29.992188
58.792969 c 29.746094 59.914062 30.453125 61.011719 31.570312 61.253906 c 32.691406 61.496094
33.789062 60.792969 34.03125 59.671875 c 34.539062 57.355469 35.457031 55.253906 36.625 53.761719 c
37.21875 52.996094 38.429688 51.714844 39.832031 51.714844 c 41.234375 51.714844 42.445312 52.996094
43.039062 53.761719 c 44.207031 55.253906 45.128906 57.355469 45.632812 59.671875 c 45.875 60.792969
46.976562 61.496094 48.09375 61.253906 c 49.207031 61.011719 49.921875 59.914062 49.675781 58.792969
c 49.042969 55.878906 47.84375 53.191406 46.300781 51.214844 c 45.449219 50.121094 44.535156
49.324219 43.746094 48.808594 c 43.613281 48.730469 43.34375 48.566406 43.214844 48.492188 c
42.824219 48.253906 41.972656 47.800781 41.972656 46.640625 c 41.972656 43.273438 l 43.6875 43.726562
45.359375 44.726562 46.96875 46.261719 c 49.203125 48.394531 51.242188 51.519531 52.867188 55.3125 c
53.316406 56.359375 54.535156 56.847656 55.585938 56.402344 c 56.632812 55.945312 57.121094 54.730469
56.671875 53.675781 c h 41.667969 19.335938 m 41.667969 19.40625 l 41.667969 19.40625 41.667969
28.429688 41.667969 31.636719 c 41.667969 32.222656 41.667969 32.613281 41.667969 32.71875 c
41.667969 33.425781 42.179688 34.445312 43.285156 34.683594 c 45.769531 35.207031 48.605469 36.660156
51.167969 38.777344 c 53.679688 40.847656 56.042969 43.570312 58.1875 46.878906 c 58.808594 47.835938
60.089844 48.105469 61.050781 47.488281 c 61.660156 47.089844 61.988281 46.425781 61.988281 45.75 c
61.988281 45.359375 61.882812 44.96875 61.65625 44.621094 c 59.28125 40.96875 56.640625 37.929688
53.800781 35.582031 c 51.28125 33.5 48.644531 32.011719 45.949219 31.136719 c 45.949219 26.382812 l
49.828125 27.445312 53.875 29.652344 57.359375 32.636719 c 58.226562 33.375 59.53125 33.277344
60.273438 32.410156 c 61.019531 31.542969 60.914062 30.234375 60.046875 29.492188 c 57.824219
27.585938 55.289062 25.894531 52.71875 24.601562 c 50.453125 23.453125 48.183594 22.621094 45.949219
22.109375 c 45.949219 19.335938 l 45.949219 18.15625 44.992188 17.195312 43.808594 17.195312 c 42.625
17.195312 41.667969 18.15625 41.667969 19.335938 c h 35.855469 17.195312 m 34.675781 17.195312
33.714844 18.15625 33.714844 19.335938 c 33.714844 22.109375 l 31.484375 22.621094 29.210938
23.453125 26.945312 24.601562 c 24.375 25.894531 21.839844 27.585938 19.617188 29.492188 c 18.746094
30.234375 18.644531 31.542969 19.390625 32.410156 c 20.132812 33.277344 21.4375 33.375 22.308594
32.636719 c 25.789062 29.652344 29.832031 27.445312 33.714844 26.382812 c 33.714844 31.136719 l
31.023438 32.011719 28.386719 33.5 25.863281 35.582031 c 23.023438 37.929688 20.378906 40.96875
18.007812 44.621094 c 17.78125 44.96875 17.671875 45.359375 17.671875 45.75 c 17.671875 46.425781
18.007812 47.089844 18.617188 47.488281 c 19.574219 48.105469 20.855469 47.835938 21.476562 46.878906
c 23.621094 43.570312 25.984375 40.847656 28.496094 38.777344 c 31.058594 36.660156 33.898438
35.207031 36.382812 34.683594 c 37.488281 34.445312 37.996094 33.425781 37.996094 32.71875 c
37.996094 32.613281 37.996094 32.222656 37.996094 31.636719 c 37.996094 28.429688 37.996094 19.40625
37.996094 19.40625 c 37.996094 19.335938 l 37.996094 18.15625 37.035156 17.195312 35.855469 17.195312
c f Q%
}
\makeatother

\begin{document}

\maketitle

\begin{abstract}
Joint reinforcement learning and self-distillation combines task rewards
with dense teacher guidance for agent training, but their gradients can
partly cancel. A mixing coefficient controls relative scale without
selectively removing opposing components. We present \textbf{CORD}, a
general framework for conflict-aware optimization of two objectives.
CORD separates the choice of protected objective from the condition that
triggers projection, and combines the gradients once per optimizer
mini-batch. A nearest-point formulation yields a closed-form update and
characterizes its first-order effect on the protected objective before
optimizer transformations. Gradient diagnostics connect conflict in agent
training to lost local optimization progress and motivate selective
correction. As a gradient-combination mechanism, CORD can be added to
different joint training recipes while retaining their objectives.
Across three such recipes, it improves ALFWorld success by
1.5--15.9 percentage points. On WebShop, it improves two of the three pairs,
with SDAR rising from 60.7\% to 64.6\% success.
Our analyses explain why protection direction and projection gate must be
chosen together, and why the distillation objective and its mixing weight
remain consequential after projection.

\par\vspace{5pt}
{\footnotesize\sffamily\noindent
\textbf{Code}\enspace
\href{https://github.com/LvAoAo/CORD-release}{github.com/LvAoAo/CORD-release}\par}

\end{abstract}

\section{Introduction}
\label{sec:intro}

Language-model agents learn to complete tasks through sequences of actions,
but outcome rewards provide limited credit for individual decisions
\citep{feng2025gigpo}. Group Relative Policy Optimization (GRPO) compares
returns across sampled attempts \citep{shao2024deepseekmath}, yet supplies
little learning signal when rewards tie. On-policy distillation (OPD)
provides dense token-level feedback on student trajectories
\citep{agarwal2024opd}, but its usefulness depends on the teacher's guidance.
Joint training couples this guidance with reward-based optimization,
allowing agents to learn from both task outcomes and informed action
preferences \citep{lu2026sdar}.

Existing approaches refine how teacher guidance enters learning.
Skill-conditioned self-distillation supplies an informed contextual view
of the policy \citep{wang2026skillsd}, while SDAR weights token supervision
using detached teacher--student confidence gaps \citep{lu2026sdar}.
RLSD incorporates detached log-probability differences into token
advantages, preserving their signs without a separate distillation loss
\citep{yang2026rlvr}. Confidence gates reweight token
signals, whereas a gradient inner-product gate tests their accumulated
effect. The two operations can be composed for separate RL and OPD losses.

Joint optimization still requires a rule for combining these gradients.
Which contribution to preserve on conflict depends on the losses' roles
and the training goal. A mixing coefficient changes scale but cannot
selectively remove opposing components.
Figure~\ref{fig:motivation-triptych} diagnoses this interaction:
\begin{enumerate}[label=\textbf{(\alph*)},leftmargin=1.8em,itemsep=2pt,topsep=3pt,parsep=0pt]
\item \textbf{Conflict occurs during agent training.}
  In the WebShop \citep{yao2022webshop} log used for the left panel,
  48.2\% of mini-batch gradient records have negative RL--distillation cosine.
\item \textbf{Conflict reduces local optimization progress.}
  The middle panel shows that a conflicting distillation gradient reduces
  first-order progress on the RL surrogate relative to RL alone.
\item \textbf{The opposing component can be removed selectively.}
  On conflict, projecting one gradient onto the other's orthogonal
  complement moves the displayed cosine quantiles to zero and recovers the
  protected gradient's first-order progress. Aligned pairs remain unchanged.
\end{enumerate}
\begin{figure}[t]
\centering
\includegraphics[width=0.32\linewidth]{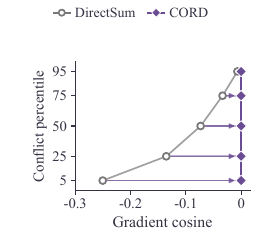}\hfill
\includegraphics[width=0.32\linewidth]{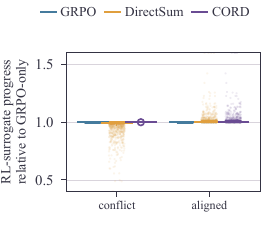}\hfill
\includegraphics[width=0.32\linewidth]{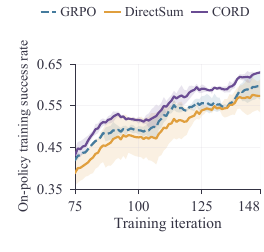}
\caption{\textbf{From gradient conflict to improved agent training.}
\textbf{Left, observations (a) and (c):} cosine quantiles of conflicting
WebShop/Qwen3-1.7B gradient records \citep{yang2025qwen3}, before and after analytic CORD projection.
\textbf{Middle, observation (b):} first-order RL-surrogate progress relative
to RL alone on the same log, evaluated at $\lambda=0.005$ with RL protection.
\textbf{Right:} preliminary WebShop/3B training success with $k_3$,
$\lambda=0.001$, and conflict-only distillation protection.
The right panel shows iterations 75--148 after 21-iteration smoothing.
Curves average two training seeds and bands show their range.
Appendix~\ref{app:gradient-diagnostics} gives measurement and smoothing details.}
\label{fig:motivation-triptych}
\end{figure}

We introduce \textbf{CORD}, a general framework for conflict-aware
optimization of two objectives. CORD accumulates both gradients under the
specified loss weights for each optimizer mini-batch. It retains the
selected protected gradient and minimally modifies the other before
combining them. Conflict-only projection removes opposing components,
while an always-projection variant also removes aligned components.
Either objective can be protected. Our analysis relates the resulting
first-order contribution on fixed sampled losses to finite step size and
optimizer transformations. The agent experiments evaluate protection,
triggering, and loss construction while retaining each host's objectives.
Appendix~\ref{app:related} discusses related work.

Our contributions are threefold:
\begin{itemize}[leftmargin=1.4em,itemsep=1pt]
\item \textbf{Conflict-aware joint optimization.}
  We connect measured gradient conflict to lost first-order progress and
  formulate CORD to preserve a selected objective's contribution while
  correcting the conflicting component of the other.
\item \textbf{Incremental gains across training methods.}
  CORD improves all three host recipes on ALFWorld by 1.5--15.9 percentage
  points, and two of three on WebShop. These paired comparisons establish
  the value of gradient combination while retaining each host's objectives.
\item \textbf{Guidance for configuring joint training.}
  We show why protection direction and triggering must be selected together,
  and how the distillation loss, mixing weight, and aggregation determine
  the gradient that projection receives. Agent cases connect these choices
  to attribute completion, repeated actions, and token probability changes.
\end{itemize}

\section{Preliminaries}
\label{sec:preliminaries}

\subsection{Group Relative Policy Optimization (GRPO)}

An agent samples response tokens $y_t$ from a rollout policy
$\pi_{\mathrm{old}}$, conditioned on interaction history $x_t$.
GRPO \citep{shao2024deepseekmath} compares returns within a group of attempts
at the same task. Each response record receives a fixed advantage
$\widehat A_t=(R_t-\mu_{g(t)})/(\sigma_{g(t)}+\epsilon)$, where
$\mu_{g(t)}$ and $\sigma_{g(t)}$ are its group's return mean and standard
deviation. The host specifies which response records enter these statistics.
With probability ratio $r_t=\pi_\theta(y_t\mid x_t)/\pi_{\mathrm{old}}(y_t\mid x_t)$,
the clipped policy loss to minimize is
\begin{equation}
\label{eq:prelim-grpo}
\mathcal L_{\mathrm{clip}}
=-\mathcal A_B\!\left[\min\!\left(r_t\widehat A_t,
\operatorname{clip}(r_t,1-\varepsilon,1+\varepsilon)\widehat A_t\right)\right].
\end{equation}
Here $\mathcal A_B$ aggregates valid response tokens in mini-batch $B$ with
the host's specified weights, with padding and observation tokens masked.
Our GRPO actor loss $\Lgrpo$ includes the host's clipping safeguards,
reference-policy KL regularization, and entropy term when enabled
(Appendix~\ref{app:implementation-evaluation}). All gradients below are
loss gradients, so a descent update subtracts the combined gradient.

\subsection{On-Policy Distillation}

On-policy distillation scores the student's sampled tokens with a teacher
rather than requiring teacher-generated trajectories
\citep{agarwal2024opd,zhao2026opd,wang2026skillsd}. In our skill-conditioned instantiation,
the teacher uses the current policy with additional skill context $S(x_t)$
and gradients disabled \citep{lu2026sdar}. Define
\begin{equation}
\label{eq:prelim-logratio}
s_t=\log\pi_\theta(y_t\mid x_t),\qquad
q_t=\operatorname{sg}\!\left[\log\pi_\theta(y_t\mid x_t\oplus S(x_t))\right],
\qquad \ell_t=s_t-q_t,
\end{equation}
where $\operatorname{sg}$ denotes stop-gradient. A token objective $\psi_t$
and confidence weight $w_t$ define a separate loss
\begin{equation}
\label{eq:prelim-opd}
\Lopd=\mathcal A_B[w_t\psi_t],\qquad
\psi_t^{k_3}=\rho_{\mathrm{on},t}(e^{-\ell_t}-1+\ell_t),\qquad
\psi_t^{\rho_{\mathrm d}}=\rho_{\mathrm{on},t}\ell_t.
\end{equation}
Equation~\ref{eq:prelim-opd} suppresses the numerical exponential cap.
Gradients pass through the capped student-to-rollout ratio
$\rho_{\mathrm{on},t}$ and $s_t$, while teacher scores and enabled
confidence weights are detached. The host $k_1$ term and numerical caps follow the inherited recipes
in Appendix~\ref{app:baseline-details}. Token weights enter before
aggregation, while the mixing coefficient below scales the resulting loss.

\subsection{Multi-Objective Optimization and Gradient Conflict}
\label{sec:setup}

Bold symbols denote parameter-space vectors. Losses, inner products,
and coefficients are scalars. We use $\Gobj{j}$ for a complete aggregated
objective gradient and $\gobj{j}{m}$ for a micro-batch gradient.
The role vectors $\bm{u},\bm{v},\bm{a},\bm{b}$ refer to these aggregates
after the stated weighting and protection choices.

For two mini-batch losses $F_1,F_2$ and parameters $\theta\in\R^d$, weighted-sum
optimization minimizes $F_1+\lambda F_2$. At fixed $\lambda>0$, write
\begin{equation}
\label{eq:oneliner}
\bm{u}=\Gobj{1}=\nabla F_1(\theta),\qquad
\bm{v}=\lambda\Gobj{2}=\lambda\nabla F_2(\theta),\qquad
\bm{g}_{\mathrm{sum}}=\bm{u}+\bm{v}.
\end{equation}
Our agent instantiation uses $F_1=\Lgrpo$ and $F_2=\Lopd$. The same
combination applies to any two differentiable losses with specified
aggregation. A small step along $-\bm{g}_{\mathrm{sum}}$ changes $F_1$ by
\begin{equation}
\label{eq:prelim-conflict}
F_1(\theta-\eta\bm{g}_{\mathrm{sum}})-F_1(\theta)
=-\eta\big(\norm{\bm{u}}^2+\inner{\bm{u}}{\bm{v}}\big)+o(\eta).
\end{equation}
For $\lambda F_2$, the corresponding expression replaces
$\norm{\bm{u}}^2$ by $\norm{\bm{v}}^2$. A negative cross term therefore
reduces each objective's first-order progress relative to its own gradient.
We call this \emph{gradient conflict}. Scaling $\lambda>0$ changes its
magnitude but not its sign. CORD addresses this interaction by choosing
which weighted gradient to preserve and minimally modifying the other.

\section{Method}
\label{sec:method}

CORD combines two objective gradients by making the protected direction and
projection gate explicit. It keeps one weighted gradient intact and makes the
smallest change to the other that satisfies the chosen geometric constraint.
The same construction accommodates either protection direction. The loss
functions and their relative weight specify the signals supplied to it.
Figure~\ref{fig:cord-main} illustrates the combination for GRPO and OPD,
with an analytic loss surface showing example update paths.

\begin{figure}[!t]
\centering
\includegraphics[width=0.90\linewidth]{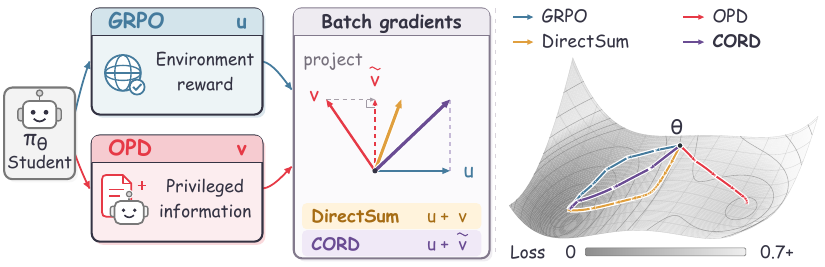}
\caption{\textbf{CORD combines two learning signals through gradient projection.}
Left: GRPO uses environment rewards, while OPD uses a teacher conditioned on
privileged skill context, with teacher gradients disabled.
Middle: accumulated gradients $\bm u=\nabla F_1$ and
$\bm v=\lambda\nabla F_2$ give the DirectSum $\bm u+\bm v$.
For the conflicting pair shown, CORD protects $\bm u$ and removes the
opposing component of $\bm v$, giving $\bm u+\tilde{\bm v}$.
Either objective can be protected.
Right: illustrative paths in an analytic two-loss construction with GRPO
protection. Height and shading encode its GRPO loss. Arrows denote updates
of step size $0.7$, stopping within $0.10$ of either local minimum.}
\label{fig:cord-main}
\end{figure}

\subsection{Choosing a protected direction}

Using the weighted gradients in Equation~\ref{eq:oneliner}, CORD selects a
protected gradient $\bm{a}$ and a candidate gradient $\bm{b}$ from
$(\bm{a},\bm{b})\in\{(\bm{u},\bm{v}),(\bm{v},\bm{u})\}$.
The protected loss $F_a$ is correspondingly $F_1$ or $\lambda F_2$, so that
$\nabla F_a=\bm{a}$. The loss definitions and training goal inform this protection choice.
Swapping roles alters the combined direction without redefining losses
or fixing priority across problems.

\subsection{Projection direction and gate}
\label{sec:template}
\label{sec:update}

Under the \emph{conflict-only} gate, CORD combines gradients as
\begin{equation}
\label{eq:gfinal}
\gfinal=\bm{a}+\tilde{\bm{b}},\qquad
\tilde{\bm{b}}=
\begin{cases}
\bm{b}-\dfrac{\inner{\bm{a}}{\bm{b}}}{\norm{\bm{a}}^2}\bm{a},
 & \inner{\bm{a}}{\bm{b}}<0,\\[5pt]
\bm{b}, & \inner{\bm{a}}{\bm{b}}\ge0.
\end{cases}
\end{equation}
A zero protected gradient leaves the candidate unchanged.

\begin{lemma}[Closest feasible replacement]
\label{lem:proj}
For $\bm{a},\bm{b}\in\R^d$, the candidate $\tilde{\bm{b}}$ in
Equation~\ref{eq:gfinal} is the unique solution of
\begin{equation}
\label{eq:lem-proj}
\min_{\bm{z}\in\R^d}\ \tfrac12\norm{\bm{z}-\bm{b}}^2
\quad\text{subject to}\quad \inner{\bm{a}}{\bm{z}}\ge0.
\end{equation}
\end{lemma}
\noindent\emph{Proof.}
For $\bm{a}=\bm{0}$ every vector is feasible, and the minimizer is $\bm{b}$.
For $\bm{a}\ne\bm{0}$, the Lagrangian is
$\tfrac12\norm{\bm{z}-\bm{b}}^2-\mu\inner{\bm{a}}{\bm{z}}$.
Stationarity gives $\bm{z}=\bm{b}+\mu\bm{a}$.
Feasibility, $\mu\ge0$, and complementary slackness give
$\mu=0$ when $\inner{\bm{a}}{\bm{b}}\ge0$, and
$\mu=-\inner{\bm{a}}{\bm{b}}/\norm{\bm{a}}^2$ otherwise.
These conditions are sufficient for this convex problem, and strict convexity
makes the minimizer unique. $\square$

The \emph{always} gate instead requires
$\inner{\bm{a}}{\tilde{\bm{b}}}=0$: for nonzero $\bm{a}$ it applies the
projection in the first branch at every update, including aligned gradients.
Disabling projection sets $\tilde{\bm{b}}=\bm{b}$ and restores weighted
summation. CORD uses the single-constraint geometry of A-GEM
\citep{chaudhry2019efficient}, with two current objective gradients and an
explicit choice of protection direction and gate. PCGrad
\citep{yu2020gradient} modifies both gradients on conflict, whereas CORD keeps the
protected gradient intact.

\noindent\begin{minipage}{\linewidth}
\begin{proposition}[First-order effect on the protected loss]
\label{prop:firstorder}
Fix $\theta$, $\lambda$, and the mini-batch losses, and suppose $F_a$ is
differentiable at $\theta$. With conflict-only projection,
\begin{align}
\inner{\bm{a}}{\gfinal}
&=\norm{\bm{a}}^2+\max\{\inner{\bm{a}}{\bm{b}},0\},
\label{eq:audit-grpo}\\
F_a(\theta-\eta\gfinal)-F_a(\theta)
&=-\eta\inner{\bm{a}}{\gfinal}+o(\eta).
\notag
\end{align}
for $\eta\to0^+$. Whenever projection is applied,
$\inner{\bm{a}}{\gfinal}=\norm{\bm{a}}^2$. This equality also holds under the
always gate.
\end{proposition}
\noindent\emph{Proof.}
The projected candidate is orthogonal to $\bm{a}$, giving the equality on
projection. Otherwise $\tilde{\bm{b}}=\bm{b}$ and the inner product is
nonnegative. Differentiability gives the stated expansion for this fixed
direction. The zero-reference case has both inner products equal to zero.
$\square$
\end{minipage}

Projection preserves the protected gradient's first-order descent rate.
When it applies with $\bm{a}\ne\bm{0}$, orthogonal decomposition of
$\bm{b}$ gives
\begin{equation}
\label{eq:audit-opd}
\inner{\bm{b}}{\tilde{\bm{b}}}=\norm{\tilde{\bm{b}}}^2
=\norm{\bm{b}}^2-\frac{\inner{\bm{a}}{\bm{b}}^2}{\norm{\bm{a}}^2}.
\end{equation}
This identity describes the surviving candidate alone. The full direction
also includes $\bm{a}$, so the unprotected loss need not decrease.
These properties concern the ideal combined gradient. Global norm clipping
rescales it, whereas AdamW's momentum, adaptive scaling, and weight decay
can change the update direction \citep{loshchilov2019decoupled}.
Appendix~\ref{app:actual-update} relates the first-order term to the actual
parameter displacement and finite-step loss change.

\subsection{How protection interacts with the mixing weight}
\label{sec:protection-scale}

The two protection directions share a conflict test but respond differently
to the coefficient. Write $\Gobj{1}=\nabla F_1$ and
$\Gobj{2}=\nabla F_2$, and fix a conflicting pair with
$c=\inner{\Gobj{1}}{\Gobj{2}}<0$. Equation~\ref{eq:gfinal} becomes
\begin{align}
\bm{g}_{\mathrm{protect}\,1}
&=\Gobj{1}+\lambda\left(\Gobj{2}-
  \frac{c}{\norm{\Gobj{1}}^2}\Gobj{1}\right),\label{eq:protect-scale-one}\\
\bm{g}_{\mathrm{protect}\,2}
&=\lambda\Gobj{2}+\left(\Gobj{1}-
  \frac{c}{\norm{\Gobj{2}}^2}\Gobj{2}\right).\label{eq:protect-scale-two}
\end{align}
In the first case, $\lambda$ scales the candidate and its correction
together. In the second, it scales the protected gradient, while the
correction to $\Gobj{1}$ is independent of $\lambda$. If $\phi$ is the
angle between the unweighted gradients, the corresponding correction
norms are $\lambda\norm{\Gobj{2}}|\cos\phi|$ and
$\norm{\Gobj{1}}|\cos\phi|$. Swapping protection therefore changes which gradient component
$\lambda$ controls.

For a fixed pair and $\lambda>0$, scaling does not change the conflict
decision. Conflict-only projection retains aligned candidates. The always
gate removes their parallel component as well. The coefficient controls
scale, while the gate controls whether a helpful component is retained.
Appendix~\ref{app:task-projection-intervention} isolates these effects at
fixed gradients, while RQ2 compares the resulting trained policies.

\subsection{Accumulation and agent training}
\label{sec:minibatch}

For an aggregation rule $F_j=\sum_m\alpha_m F_j^{(m)}$, with fixed
weights $\alpha_m$, define $\gobj{j}{m}=\nabla F_j^{(m)}$ and
$\Gobj{j}=\sum_m\alpha_m\gobj{j}{m}$.
CORD acts once on these accumulated mini-batch gradients. Projection is
nonlinear in the gradient pair, so projecting each micro-batch separately
usually produces a different update. Accumulating first defines the
projection at the optimizer's update scale. The loss aggregation rule sets
the weights before accumulation. Keeping those weights fixed makes
repartitioning a computational choice, whereas changing from micro-batch to
global-token or trajectory averaging changes the gradients themselves.
Appendix~\ref{app:optimization-sensitivity} compares these choices.

\begin{algorithm}[t]
\DontPrintSemicolon
\KwIn{mini-batch $B$; losses $F_1,F_2$ with specified aggregation;
coefficient $\lambda$; protection direction; gate}
Accumulate $\Gobj{1}=\nabla F_1$ and $\Gobj{2}=\nabla F_2$ separately over $B$\;
$\bm{u}\leftarrow\Gobj{1}$,\quad $\bm{v}\leftarrow\lambda\Gobj{2}$\;
Select $(\bm{a},\bm{b})\leftarrow(\bm{u},\bm{v})$ or $(\bm{v},\bm{u})$\;
Reduce the inner-product and norm statistics across gradient shards\;
$\tilde{\bm{b}}\leftarrow\bm{b}$\;
\If{$\norm{\bm{a}}^2>0$ and (gate is always or
(gate is conflict-only and $\inner{\bm{a}}{\bm{b}}<0$))}{
$\tilde{\bm{b}}\leftarrow\bm{b}-
(\inner{\bm{a}}{\bm{b}}/\norm{\bm{a}}^2)\bm{a}$\;
}
Pass $\gfinal\leftarrow\bm{a}+\tilde{\bm{b}}$ to the optimizer\;
\caption{CORD gradient combination. At $\lambda=0$, use $\Gobj{1}$ directly.}
\label{alg:orthosd}
\end{algorithm}

For agents, Equation~\ref{eq:oneliner} becomes $\bm{u}=\Gr$ and
$\bm{v}=\lambda\Go$: protecting GRPO projects OPD, while protecting OPD
projects GRPO. Skill-SD \citep{wang2026skillsd} uses $k_3$. GRPO+OPSD and SDAR \citep{lu2026sdar} retain their host
$k_1$ objectives and weighting conventions. The loss analysis also examines
$\rho_{\mathrm d}$ from Equation~\ref{eq:prelim-opd}.
Appendix~\ref{app:baseline-details} specifies the recipes. Main host comparisons protect RL with conflict-only projection.
The direction and gate analysis tests both protection choices.

\paragraph{Computational complexity.}
Given the two objective gradients, CORD uses $O(d)$ arithmetic and $O(d)$
additional gradient storage for $d$ trainable parameters. With $K$ PPO
mini-batches per training iteration, the gradient-combination cost is
$O(Kd)$. When the objectives depend on the same student outputs, their
gradient evaluations can share the forward computation. They still require
distinct objective-gradient information, followed by inner products,
a projection, and vector addition. Shared computation and fused vector
operations can reduce redundant work without changing the projection rule.
For a fixed workload, actor overhead differs from full-iteration overhead. Rollout, teacher
scoring, and environment interaction are shared stages, so their share of
each iteration determines how additional gradient work affects total time.
The component model in Appendix~\ref{app:computational-cost} holds the
workload fixed and separates measured actor timings from estimated
iteration overhead.

\section{Experiments}
\label{sec:exp}

We test whether changing the gradient-combination rule adds value across
joint training methods, then use the analyses to guide its configuration.
The experiments answer three research questions:
\textbf{(RQ1) Effectiveness:} How much does CORD improve existing joint
training recipes?
\textbf{(RQ2) Projection choices:} How do the protected objective, projection
gate, and loss construction guide practical configuration?
\textbf{(RQ3) Agent behavior:} How do update choices appear in agent
interactions?

\subsection{Experimental setup}
\label{sec:exp-setup}

\paragraph{Tasks and training.}
We use the SDAR agent suite \citep{lu2026sdar} with
Qwen2.5-3B-Instruct \citep{qwen2024qwen25}.
ALFWorld \citep{shridhar2021alfworld} evaluates household tasks over at most
50 actions. We report success across six categories and the pooled average.
WebShop \citep{yao2022webshop} evaluates instruction-following shopping.
We cap episodes at 15 actions and report graded task score and strict
success rate on a 0--100 scale.
Our shared trainer uses 150 training iterations, eight rollouts per prompt,
AdamW at $10^{-6}$, reference-KL coefficient $0.01$, and responses capped at
512 tokens on eight NVIDIA A100 40\,GB GPUs. Full settings are given in
Appendix~\ref{app:experimental-details}.

\paragraph{Methods.}
Table~\ref{tab:main} includes Vanilla, Skill-Prompt, and OPSD values reported
by \citet{lu2026sdar}. GRPO and RLSD results are our own runs.
RLSD incorporates the self-distillation signal into the policy optimization
objective \citep{yang2026rlvr}.
Following SDAR, the joint recipes share a predefined skill bank whose
contents remain fixed throughout training. Each Base--CORD pair uses the
same bank and skill-selection rule to isolate gradient combination.
Household tasks select task-specific manipulation skills, while shopping
instructions select relevant skills by keyword alongside general
strategies. This context guides teacher scoring of student-generated
actions and is absent from the student's rollout and evaluation inputs.
We apply CORD to three joint training recipes in this shared interface:
Skill-SD \citep{wang2026skillsd} ($k_3$, $\lambda=0.001$), GRPO+OPSD
(host $k_1$, $\lambda=0.005$), and SDAR \citep{lu2026sdar}
(confidence-weighted host $k_1$, $\lambda=0.01$).
Each recipe without CORD uses naive addition. $+$ CORD retains its coefficient,
protects the RL gradient, and projects only on conflict.
RQ2 then varies the protected objective and projection gate.

\paragraph{Evaluation.}
We fix validation tasks to track progress on the same problems during
training and hold the final evaluation tasks constant across paired methods.
For the main comparison and full-training analyses in RQ2, including
PCGrad \citep{yu2020gradient}, we evaluate the final policy four times
after 150 updates at temperature $0.4$. Each run uses one training seed,
and the displayed $\pm$ is the population standard deviation across
these four passes.
Our WebShop evaluation uses 500 tasks, whereas the published baseline
values in Table~\ref{tab:main} use the 128-task protocol of
\citet{lu2026sdar}.
Appendix~\ref{app:implementation-evaluation} details the protocol.

\subsection{RQ1: Incremental effectiveness}
\label{sec:exp-main}

\begin{table}[t]
\centering
\caption{\textbf{Final results on Qwen2.5-3B-Instruct.}
ALFWorld success and WebShop score/success are on a 0--100 scale.
$\dagger$: \citet{lu2026sdar}.
$*$: skills at inference. Our step-150 results use one training seed and
four evaluation passes. $\pm$ denotes decode standard deviation.
CORD protects RL with conflict-only
projection. $\Delta=\mathrm{CORD}-\mathrm{Base}$ uses displayed means:
success differences are percentage points and score differences are score
points, with red/green indicating increases/decreases. Deltas have no assigned
standard deviation and are excluded from ranking. Bold/underline mark the
largest/second-largest distinct method means. Ties share the mark.}
\label{tab:main}
\begingroup
\small
\setlength{\tabcolsep}{1.4pt}
\renewcommand{\arraystretch}{1.22}
\begin{tabularx}{\linewidth}{@{}l*{6}{>{\centering\arraybackslash}X}>{\columncolor{verylightgray}\centering\arraybackslash}X>{\centering\arraybackslash}X>{\columncolor{verylightgray}\centering\arraybackslash}X@{}}
\toprule
& \multicolumn{7}{c}{ALFWorld} & \multicolumn{2}{c}{WebShop}\\
\cmidrule(lr){2-8}\cmidrule(l){9-10}
Method & Pick & Look & Clean & Heat & Cool & Pick2 & \textbf{Avg.} & Score & \shortstack{Success\\Rate}\\
\midrule
Vanilla$^\dagger$ & $44.4$ & $11.1$ & $6.2$ & $15.4$ & $28.6$ & $12.5$ & $21.9$ & $6.7$ & $0.8$\\
Skill-Prompt$^{*\dagger}$ & $51.7$ & $66.7$ & $48.4$ & $0.0$ & $4.3$ & $10.0$ & $28.9$ & $0.2$ & $0.8$\\
OPSD$^\dagger$ & $48.8$ & $41.7$ & $16.7$ & $0.0$ & $15.8$ & $16.7$ & $28.1$ & $11.3$ & $3.1$\\
GRPO & $93.5{\scriptstyle\pm 3.1}$ & $72.9{\scriptstyle\pm 12.3}$ & $67.0{\scriptstyle\pm 3.0}$ & $62.5{\scriptstyle\pm 5.9}$ & $59.5{\scriptstyle\pm 2.4}$ & $66.3{\scriptstyle\pm 8.3}$ & $71.3{\scriptstyle\pm 1.9}$ & $71.1{\scriptstyle\pm 1.1}$ & $59.2{\scriptstyle\pm 1.5}$\\
RLSD & $95.4{\scriptstyle\pm 1.6}$ & $68.8{\scriptstyle\pm 3.6}$ & $83.0{\scriptstyle\pm 3.9}$ & $\mathbf{87.5}{\scriptstyle\pm 7.8}$ & $56.0{\scriptstyle\pm 3.9}$ & $53.8{\scriptstyle\pm 6.7}$ & $74.4{\scriptstyle\pm 2.2}$ & $69.8{\scriptstyle\pm 0.7}$ & $54.4{\scriptstyle\pm 1.0}$\\
\midrule
Skill-SD & $89.3{\scriptstyle\pm 2.4}$ & $\mathbf{78.8}{\scriptstyle\pm 6.4}$ & $85.2{\scriptstyle\pm 3.7}$ & $65.6{\scriptstyle\pm 7.0}$ & $60.0{\scriptstyle\pm 2.8}$ & $\mathbf{70.8}{\scriptstyle\pm 7.8}$ & $76.4{\scriptstyle\pm 1.7}$ & $72.1{\scriptstyle\pm 0.9}$ & $55.8{\scriptstyle\pm 0.9}$\\
\quad\textbf{+ CORD} & $91.4{\scriptstyle\pm 3.5}$ & $67.3{\scriptstyle\pm 6.4}$ & $87.0{\scriptstyle\pm 4.1}$ & $67.2{\scriptstyle\pm 5.2}$ & $\mathbf{78.0}{\scriptstyle\pm 2.0}$ & $60.4{\scriptstyle\pm 2.1}$ & $\underline{77.9}{\scriptstyle\pm 1.5}$ & $78.1{\scriptstyle\pm 0.2}$ & $59.4{\scriptstyle\pm 0.8}$\\
\quad$\Delta$ & \corddeltaup{2.1} & \corddeltadown{11.5} & \corddeltaup{1.8} & \corddeltaup{1.6} & \corddeltaup{18.0} & \corddeltadown{10.4} & \corddeltaup{1.5} & \corddeltaup{6.0} & \corddeltaup{3.6}\\
\midrule
GRPO+OPSD & $94.4{\scriptstyle\pm 6.1}$ & $60.4{\scriptstyle\pm 12.3}$ & $75.0{\scriptstyle\pm 2.5}$ & $58.9{\scriptstyle\pm 5.9}$ & $50.0{\scriptstyle\pm 5.3}$ & $41.3{\scriptstyle\pm 4.2}$ & $65.0{\scriptstyle\pm 2.6}$ & $\mathbf{80.5}{\scriptstyle\pm 0.6}$ & $\mathbf{64.8}{\scriptstyle\pm 0.2}$\\
\quad\textbf{+ CORD} & $\mathbf{98.6}{\scriptstyle\pm 1.4}$ & $75.0{\scriptstyle\pm 6.4}$ & $\underline{88.0}{\scriptstyle\pm 4.0}$ & $75.0{\scriptstyle\pm 6.3}$ & $66.0{\scriptstyle\pm 2.0}$ & $\underline{69.8}{\scriptstyle\pm 4.5}$ & $\mathbf{80.9}{\scriptstyle\pm 1.6}$ & $77.5{\scriptstyle\pm 0.2}$ & $64.5{\scriptstyle\pm 1.0}$\\
\quad$\Delta$ & \corddeltaup{4.2} & \corddeltaup{14.6} & \corddeltaup{13.0} & \corddeltaup{16.1} & \corddeltaup{16.0} & \corddeltaup{28.5} & \corddeltaup{15.9} & \corddeltadown{3.0} & \corddeltadown{0.3}\\
\midrule
SDAR & $\underline{96.3}{\scriptstyle\pm 2.6}$ & $64.6{\scriptstyle\pm 3.6}$ & $81.2{\scriptstyle\pm 3.0}$ & $67.9{\scriptstyle\pm 3.6}$ & $59.5{\scriptstyle\pm 7.9}$ & $61.5{\scriptstyle\pm 4.7}$ & $73.8{\scriptstyle\pm 1.4}$ & $76.6{\scriptstyle\pm 1.0}$ & $60.7{\scriptstyle\pm 0.7}$\\
\quad\textbf{+ CORD} & $88.6{\scriptstyle\pm 6.4}$ & $\underline{76.9}{\scriptstyle\pm 0.0}$ & $\mathbf{89.8}{\scriptstyle\pm 1.6}$ & $\underline{79.7}{\scriptstyle\pm 8.1}$ & $\underline{74.0}{\scriptstyle\pm 3.5}$ & $51.0{\scriptstyle\pm 3.5}$ & $77.7{\scriptstyle\pm 3.1}$ & $\underline{78.2}{\scriptstyle\pm 1.1}$ & $\underline{64.6}{\scriptstyle\pm 1.1}$\\
\quad$\Delta$ & \corddeltadown{7.7} & \corddeltaup{12.3} & \corddeltaup{8.6} & \corddeltaup{11.8} & \corddeltaup{14.5} & \corddeltadown{10.5} & \corddeltaup{3.9} & \corddeltaup{1.6} & \corddeltaup{3.9}\\

\bottomrule
\end{tabularx}
\endgroup
\end{table}

Table~\ref{tab:main} measures the increment from adding CORD to each joint
training recipe. On ALFWorld, CORD improves all three recipes:
pooled success increases from 76.4\% to 77.9\% for Skill-SD, from 65.0\%
to 80.9\% for GRPO+OPSD, and from 73.8\% to 77.7\% for SDAR.
These are gains of 1.5, 15.9, and 3.9 percentage points, respectively.
The largest gain occurs for GRPO+OPSD, which reaches the highest pooled
success in the main comparison.

On WebShop, CORD improves Skill-SD success by 3.6 points
(55.8\% to 59.4\%) and SDAR success by 3.9 points (60.7\% to 64.6\%).
Their task scores also increase, from 72.1 to 78.1 and from 76.6 to 78.2,
respectively. For GRPO+OPSD, success changes from 64.8\% to 64.5\% and task
score from 80.5 to 77.5. Thus, the same projection configuration improves
all three recipes on ALFWorld and two on WebShop.
Appendix~\ref{app:statistical-significance} reports the corresponding
statistical tests. Appendix~\ref{app:backbone-7b} adds 7B results with
gains on ALFWorld and mixed outcomes on WebShop. These differences motivate examining how protection and gating interact
with the loss, coefficient, and gradient aggregation.

\subsection{RQ2: The roles of projection and gradient choices}
\label{sec:direction-gate}

\begin{table}[htbp]
\centering
\caption{\textbf{Protection direction and gating on WebShop.}
Qwen2.5-3B-Instruct with $k_3$, reporting success (\%).
Values are means and population standard deviations across four evaluation
passes at temperature $0.4$ after 150 updates, with one training seed.
Protect RL projects the distillation gradient. Protect distillation
projects the RL gradient. PCGrad projects both on conflict.
Bold identifies the largest observed mean at each coefficient.}
\label{tab:webshop-projection-full500}
\small
\setlength{\tabcolsep}{5pt}
\begin{tabular*}{\linewidth}{@{\extracolsep{\fill}}llcc@{}}
\toprule
Projection rule & Gate & $\lambda=0.001$ & $\lambda=0.005$\\
\midrule
No projection & --- & $55.75$\stdv{0.91} & $56.75$\stdv{0.89}\\
Protect RL & Conflict-only & $59.35$\stdv{0.79} & $61.85$\stdv{1.59}\\
Protect RL & Always & $62.65$\stdv{0.79} & $61.25$\stdv{0.92}\\
Protect distillation & Conflict-only & $58.35$\stdv{0.74} & $61.90$\stdv{1.08}\\
Protect distillation & Always & $38.90$\stdv{0.95} & \best{$62.10$}\stdv{1.54}\\
PCGrad \citep{yu2020gradient} & Conflict-only & \best{$63.10$}\stdv{0.91} & $58.85$\stdv{0.95}\\
\bottomrule
\end{tabular*}
\end{table}

\paragraph{Choose protection direction and trigger together.}
The protected objective determines which gradient remains intact, while the
gate determines whether projection also removes aligned components.
Table~\ref{tab:webshop-projection-full500} compares these choices on WebShop.
At $\lambda=0.001$, conflict-only projection improves success over no
projection for both RL and distillation protection, from 55.75\% to
59.35\% and 58.35\%, respectively.
Switching to always projection raises success to 62.65\% when protecting RL
but lowers it to 38.90\% when protecting distillation.
At $\lambda=0.005$, all four CORD configurations instead lie between
61.25\% and 62.10\%, above the 56.75\% unprojected result.
Direction and gating therefore interact with the coefficient: removing the
aligned component has a different effect depending on which objective is
protected and the relative gradient scales.

CORD keeps one gradient intact, whereas the two-objective PCGrad
adaptation \citep{yu2020gradient} modifies both when they conflict
(Appendix~\ref{app:pcgrad-adaptation}).
At $\lambda=0.001$, PCGrad has the largest observed mean of 63.10\%,
but at $\lambda=0.005$ its 58.85\% is below all four CORD configurations.

The fixed-gradient analysis in
Appendix~\ref{app:task-projection-intervention} clarifies how CORD's
protection direction interacts with $\lambda$. Protecting RL changes the weighted distillation gradient,
so the correction scales with $\lambda$. Protecting distillation changes
the RL gradient, whose correction is unchanged when only $\lambda$
varies. \emph{A small auxiliary coefficient therefore does not make
distillation protection a small intervention.} Because conflict-only
projection retains aligned contributions, the effect of always projection
depends on the selected protection direction.

\paragraph{Tune the loss and its coefficient together.}
\label{sec:loss-weight}
The token objective determines the distillation gradient supplied to CORD.
Table~\ref{tab:loss-type-success} compares $k_3$, the importance-weighted
log-ratio variant $\rho_{\mathrm d}$, and the host $k_1$ implementation
at $\lambda=0.005$, with RL protection and conflict-only projection.
The relative performance of the objectives changes across environments:
$\rho_{\mathrm d}$ exceeds $k_3$ on WebShop (58.95\% versus 55.55\%),
while $k_3$ exceeds $\rho_{\mathrm d}$ on ALFWorld (83.75\% versus 79.11\%).
The host $k_1$ reaches 59.95\% and 74.11\%, respectively.
The two importance-weighted losses supply different token gradients:
away from numerical caps, their derivatives with respect to the student
log-probability are $\rho_{\mathrm{on},t}\ell_t$ for $k_3$ and
$\rho_{\mathrm{on},t}(\ell_t+1)$ for $\rho_{\mathrm d}$
(Equation~\ref{eq:opd-token-derivatives}).
Changing the loss therefore changes what CORD projects, even at the same
coefficient. On WebShop, $k_3$ also has the highest graded score despite
lower strict success, so loss selection must follow the intended metric.

\begin{table}[t]
\centering
\caption{\textbf{Token objectives at $\lambda=0.005$.}
ALFWorld success and WebShop score/success are on a 0--100 scale.
CORD protects RL with conflict-only projection. Values are mean $\pm$
standard deviation across four decoding passes of step-150
Qwen2.5-3B-Instruct checkpoints at temperature $0.4$, with one training seed. Bold marks the largest mean per column.}
\label{tab:loss-type-success}
\begingroup
\small
\setlength{\tabcolsep}{1.6pt}
\renewcommand{\arraystretch}{1.17}
\begin{tabularx}{\linewidth}{@{}l*{9}{>{\centering\arraybackslash}X}@{}}
\toprule
& \multicolumn{7}{c}{ALFWorld} & \multicolumn{2}{c}{WebShop}\\
\cmidrule(lr){2-8}\cmidrule(l){9-10}
Objective & Pick & Look & Clean & Heat & Cool & Pick2 & \textbf{Avg.} & Score & Success\\
\midrule
$k_1$ & \makecell[c]{$\mathbf{95.71}$\\[0pt]${\scriptstyle\pm 4.74}$} & \makecell[c]{$71.15$\\[0pt]${\scriptstyle\pm 6.38}$} & \makecell[c]{$71.30$\\[0pt]${\scriptstyle\pm 4.81}$} & \makecell[c]{$56.25$\\[0pt]${\scriptstyle\pm 4.42}$} & \makecell[c]{$67.00$\\[0pt]${\scriptstyle\pm 9.54}$} & \makecell[c]{$66.67$\\[0pt]${\scriptstyle\pm 5.10}$} & \makecell[c]{$74.11$\\[0pt]${\scriptstyle\pm 3.00}$} & \makecell[c]{$73.68$\\[0pt]${\scriptstyle\pm 0.79}$} & \makecell[c]{$\mathbf{59.95}$\\[0pt]${\scriptstyle\pm 0.77}$}\\
$\rho_{\mathrm{d}}$ & \makecell[c]{$92.14$\\[0pt]${\scriptstyle\pm 3.11}$} & \makecell[c]{$65.38$\\[0pt]${\scriptstyle\pm 3.85}$} & \makecell[c]{$85.19$\\[0pt]${\scriptstyle\pm 2.62}$} & \makecell[c]{$78.13$\\[0pt]${\scriptstyle\pm 3.13}$} & \makecell[c]{$70.00$\\[0pt]${\scriptstyle\pm 2.00}$} & \makecell[c]{$\mathbf{70.83}$\\[0pt]${\scriptstyle\pm 2.95}$} & \makecell[c]{$79.11$\\[0pt]${\scriptstyle\pm 1.28}$} & \makecell[c]{$69.13$\\[0pt]${\scriptstyle\pm 0.73}$} & \makecell[c]{$58.95$\\[0pt]${\scriptstyle\pm 0.38}$}\\
$k_3$ & \makecell[c]{$93.57$\\[0pt]${\scriptstyle\pm 1.24}$} & \makecell[c]{$\mathbf{80.77}$\\[0pt]${\scriptstyle\pm 8.60}$} & \makecell[c]{$\mathbf{91.67}$\\[0pt]${\scriptstyle\pm 3.07}$} & \makecell[c]{$\mathbf{79.69}$\\[0pt]${\scriptstyle\pm 6.81}$} & \makecell[c]{$\mathbf{85.00}$\\[0pt]${\scriptstyle\pm 1.73}$} & \makecell[c]{$63.54$\\[0pt]${\scriptstyle\pm 1.80}$} & \makecell[c]{$\mathbf{83.75}$\\[0pt]${\scriptstyle\pm 1.06}$} & \makecell[c]{$\mathbf{74.66}$\\[0pt]${\scriptstyle\pm 0.44}$} & \makecell[c]{$55.55$\\[0pt]${\scriptstyle\pm 0.79}$}\\
\bottomrule
\end{tabularx}
\endgroup

\end{table}

\begin{figure}[!htbp]
\centering
\includegraphics[width=0.58\linewidth]{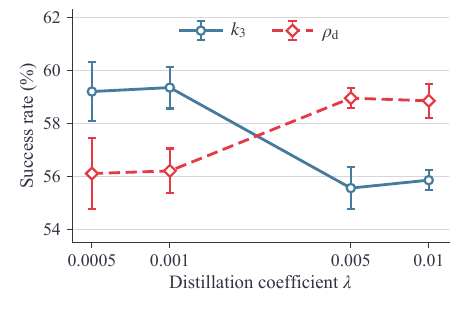}
\caption{\textbf{Effect of $\lambda$ on WebShop success.}
Error bars show standard deviation across decoding passes.}
\label{fig:lambda-success}
\end{figure}

Figure~\ref{fig:lambda-success} shows how the effect of $\lambda$ depends
on the token objective. Increasing it from 0.001 to 0.005 lowers $k_3$ success from
59.35\% to 55.55\%, but raises $\rho_{\mathrm d}$ success from 56.20\% to
58.95\%. \emph{Projection therefore does not make a coefficient transferable
between losses.} These loss and coefficient analyses use a separate
checkpoint collection from the direction and gate analysis.

\paragraph{Fix aggregation before comparing projection choices.}
\label{sec:aggregation-behavior}
The mini-batch size sets how many response sequences contribute to each
gradient pair. Five-round ALFWorld continuations from the same checkpoint
and AdamW state vary this size from 32 to 128 at a fixed episode budget
(Table~\ref{tab:optimization-sensitivity-aggregation}). Success under RL
protection varies by 0.78 percentage points, while distillation protection
peaks at size 64. Larger mini-batches change both gradient estimation and
the number of optimizer updates. These continuations update the policy
and collect new trajectories, testing training consequences beyond the
geometry of a fixed gradient pair.

Aggregation also determines how tokens and trajectories contribute to
each gradient. At size 64, the protection directions nearly tie under
microbatch weighting. Distillation protection has the higher mean under
global-token weighting, whereas RL protection has the higher mean under
trajectory weighting. All three paired intervals include zero, so these
orderings remain descriptive. Appendix~\ref{app:optimization-sensitivity}
defines the weights and protocol. Repartitioning fixed weights changes
the computation, whereas changing the weights changes the objectives
being protected.

\begin{table}[!htbp]
\centering
\caption{\textbf{Effect of gradient aggregation and mini-batch size.}
ALFWorld task-macro success (\%) after five continuation rounds from the
same Qwen2.5-3B step-90 checkpoint and AdamW state.
$B$ counts response rows per optimizer mini-batch.
Appendix~\ref{app:optimization-sensitivity} specifies the single-seed
evaluation protocol and aggregation weights. Bold marks row maxima.}
\label{tab:optimization-sensitivity-aggregation}
\small
\setlength{\tabcolsep}{5pt}
\renewcommand{\arraystretch}{1.08}
\begin{tabular*}{\linewidth}{@{\extracolsep{\fill}}lcrrr@{}}
\toprule
Aggregation & $B$ & No projection & Protect RL & Protect distillation \\
\midrule
Microbatch & 64 & 53.98 & \textbf{57.55} & 57.42 \\
Global token & 32 & 52.52 & \textbf{57.46} & 49.22 \\
Global token & 64 & 59.11 & 57.33 & \textbf{61.03} \\
Global token & 128 & \textbf{61.81} & 58.10 & 55.86 \\
Trajectory & 64 & 60.99 & \textbf{61.77} & 57.42 \\
\bottomrule
\end{tabular*}
\end{table}

\FloatBarrier

\subsection{RQ3: Case analysis of agent behavior}
\label{sec:case-study}

Success rates do not show which attributes an agent missed or how many
actions it repeated. Figure~\ref{fig:case-study} compares
three policies that open the same
product to buy a youth garment in dark heather, size 4t.
The unprojected policy selects the size and age group but omits the color,
receiving a partial score of 0.90. Both projected policies complete all
requested options and receive 1.00. With RL protection, the agent selects
each option once and purchases in six steps. With distillation protection,
it repeats four selections and purchases in ten.
\emph{Attribute completion and action efficiency are thus separate aspects of
behavior}: both directions succeed here, but their repetitions differ.
Further examples are provided in
Appendices~\ref{app:additional-cases} and~\ref{app:token-case-study}.

\begin{figure}[!htbp]
\centering
\includegraphics[width=0.95\linewidth]{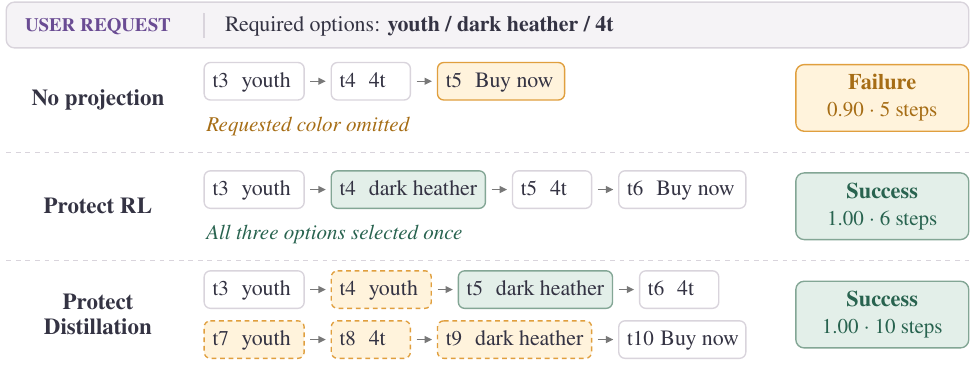}
\caption{\textbf{Option completion and repetition in WebShop.}
Qwen2.5-3B-Instruct, $k_3$, $\lambda=0.005$, step 150,
conflict-only projection, and greedy decoding.
Boxes show selected options and the purchase action, with original step indices.
Green highlights selection of the requested color. Dashed amber boxes mark
repeated selections. Scores are on $[0,1]$. Success requires $1.00$.}
\label{fig:case-study}
\end{figure}

\FloatBarrier
\section{Conclusion}
\label{sec:conclusion}

Joint RL and distillation gives agents complementary feedback, but opposing
gradient components can reduce its local optimization benefit. CORD applies
a simple projection rule to this interaction, allowing either objective to
be protected and making the trigger explicit. The first-order analysis
explains the correction on fixed surrogates before optimizer transformations, and the agent
experiments establish incremental gains across three host recipes on
ALFWorld and two of three on WebShop. The analyses also yield concrete
configuration guidance: choose protection and triggering together, tune
the mixing weight for the selected distillation loss, and specify
aggregation before comparing gradient directions. CORD thus makes gradient
combination a reusable design choice.

\clearpage
\bibliography{references}
\bibliographystyle{plainnat}

\clearpage
\appendix
\raggedbottom
\startcontents[appendices]
\cordappendixcontents
\clearpage
\section{Additional Related Work}
\label{app:related}

\subsection{On-policy and skill-conditioned distillation}
\label{app:rw-distill}

On-policy distillation trains a student on its own generated sequences while
a teacher supplies token-level targets, reducing the distribution mismatch
between training and inference. \citet{agarwal2024opd} develop generalized
knowledge distillation with a separate teacher. On-policy self-distillation
instead uses different contextual views of a model:
\citet{zhao2026opd} condition the teacher on privileged information, such as
verified reasoning traces, while the student sees the original question.
For multi-turn agents, Skill-SD \citep{wang2026skillsd} derives dynamic skills
from completed trajectories and uses an importance-weighted reverse-KL
objective with a periodically synchronized teacher. These methods supply dense supervision on student-generated trajectories.

SDAR \citep{lu2026sdar} combines GRPO with skill-conditioned distillation and
modulates token-level supervision using a detached sigmoid gate. Its
experimental protocol draws on the SkillBank of SkillRL
\citep{xia2026skillrl}, which learns a hierarchical skill library and evolves
it with the policy. SDAR also evaluates retrieved-skill adaptations of
Skill-SD and OPSD, together with prompt-based skill injection. Our three host
recipes---Skill-SD, GRPO+OPSD, and SDAR---follow the agent-training
implementations described in Section~\ref{sec:method} and
Appendix~\ref{app:experimental-details}. In particular, the retrieved-skill
$k_3$ host follows the adaptation in \citet{lu2026sdar}, which differs from
the original Skill-SD dynamic skill-generation protocol.
A token-level distillation gate changes the input gradient, whereas CORD's
projection gate decides when to modify the relationship between the two
accumulated gradients. The two operations can therefore be composed.

Another approach incorporates teacher information directly into the
advantage or policy-loss signal. RLSD \citep{yang2026rlvr} uses detached
teacher--student log-probability differences to reweight token advantages while
preserving their signs. This route modifies the policy update internally, without a separate
distillation loss to combine.

\subsection{Evolving skills and reflection memory}
\label{app:rw-evolving-skills}

The guidance available to an agent can evolve with its experience.
SkillRL expands and refines a hierarchical skill library using failure
analysis \citep{xia2026skillrl}. RetroAgent stores retrospective lessons
from successful and failed attempts in a growing reflection memory, and
combines retrieved lessons with progress-based intrinsic rewards
\citep{zhang2026retroagent}. Its RL-trained variant also jointly optimizes
decision making and self-reflection. These approaches improve what the
agent retains from its interactions and retrieves for later tasks.

Evolving guidance need not reside in a persistent library. SEED trains the
policy to generate hindsight skills from its current trajectories and
internalizes their guidance through joint RL and on-policy distillation
\citep{wu2026seed}. OPID constructs episode- and step-level skills whose
induced log-probability shifts enter the policy update
\citep{yang2026opid}. GRSD uses reflections from successful and failed
rollouts to form group-level guidance and modulate advantages
\citep{zheng2026grsd}.
CORD operates at the subsequent combination of objective gradients, so its
use is orthogonal to whether guidance comes from a fixed bank or evolves
with the policy. Potential applications include combining reward and
distillation gradients in SEED, or decision-making and reflection
gradients in the RL-trained RetroAgent variant, while preserving their
respective guidance mechanisms.

\subsection{Conflicting gradients and projection}
\label{app:rw-conflict}

Multi-task loss balancing adjusts objective scales using gradient magnitudes
or learned uncertainty \citep{chen2018gradnorm,kendall2018uncertainty}.
Gradient projection instead changes directions and has a long history in
multi-task and continual learning.
For two conflicting task gradients, PCGrad \citep{yu2020gradient} removes each
gradient's component along the other before combining them. A-GEM
\citep{chaudhry2019efficient} uses an asymmetric projection: it modifies the
current-task gradient to satisfy a single half-space constraint defined by
the gradient estimated from a sampled episodic-memory batch. This constraint prevents an
increase in that sampled memory loss under a first-order approximation.
Orthogonal gradient descent \citep{farajtabar2020orthogonal} projects updates
away from a stored subspace of gradients of previous network outputs.
CAGrad \citep{liu2021cagrad} balances progress on the mean task loss with the
worst local improvement among tasks.

CORD reuses the single-half-space projection underlying A-GEM
(Lemma~\ref{lem:proj}). The protected reference in CORD comes from one of the
two jointly trained objectives, and the projected candidate is added to that
reference. Either objective can be protected. Conflict-only projection uses
a half-space constraint, while the always-project variant imposes an
orthogonality constraint even when the gradients agree.
The agent experiments study how protection, triggering, loss construction,
and aggregation affect learning. The main host comparisons protect GRPO
with conflict-only projection, while the direction and gate analyses test
both choices.

\subsection{Agent RL and GRPO-style policy optimization}
\label{app:rw-agentrl}

GRPO \citep{shao2024deepseekmath} estimates advantages relative to groups of
sampled responses, avoiding a separate learned critic. PPO
\citep{schulman2017proximal} supplies the clipped policy surrogate used by
these host implementations. The truncated importance weighting in our $k_3$ and
$\rho_{\mathrm d}$ distillation objectives is related
to the correction used in V-trace/IMPALA \citep{espeholt2018impala}, although
the distillation weight does not implement the full V-trace return
estimator. GiGPO \citep{feng2025gigpo} improves multi-turn credit assignment
by combining episode-level advantages with step-level comparisons among
actions grouped by shared anchor states. These choices determine the
reward-derived gradient. CORD operates after the two objective gradients
have been accumulated for a mini-batch, so credit assignment and gradient
combination address different parts of the update.

\section{Experimental Details and Additional Results}
\label{app:experimental-details}

This section provides the task definitions, baseline configurations, and
evaluation procedures used to study CORD. Our implementation extends the
SDAR codebase \citep{lu2026sdar}, which builds on the multi-turn training
framework \emph{verl-agent} released with GiGPO \citep{feng2025gigpo}.
We retain this environment and skill interface for the three host recipes.

\subsection{Datasets and Tasks}
\label{app:datasets-tasks}

Table~\ref{tab:app-environments} summarizes the two environments. ALFWorld
and WebShop require an agent to execute actions whose consequences determine
which observations become available next. Both tasks connect token-level
distillation to rewards obtained after several decisions.

\begin{table}[!htbp]
\centering
\caption{Task and environment configuration. A turn denotes one generated
response followed by an environment interaction. Evaluation sizes describe the
source task pools or episode assignments. Section~\ref{app:implementation-evaluation} specifies the fixed evaluation protocol. The WebShop catalogue
size counts products, rather than training instructions.}
\label{tab:app-environments}
\begingroup
\small
\setlength{\tabcolsep}{4pt}
\renewcommand{\arraystretch}{1.14}
\begin{tabularx}{\textwidth}{@{}lXXcl@{}}
\toprule
Environment & Training source & Evaluation source & Turns & Metric \\
\midrule
ALFWorld & $3{,}553$ training games & $140$ in-distribution games & $50$ & Success rate \\
WebShop & Goal pool from a $1{,}000$-product catalogue & $500$ worker--goal assignments & $15$ & Score; success \\
\bottomrule
\end{tabularx}
\endgroup
\end{table}

\paragraph{ALFWorld.}
ALFWorld \citep{shridhar2021alfworld} provides text-based household tasks in
which the agent navigates rooms, locates objects, and performs the operations
specified by a goal. We use the TextWorld interface and the training split
shared by SDAR and GiGPO. Evaluation uses the in-distribution split, whose games contain environments
seen during training.
The six task categories are Pick and Place (Pick), Examine in Light (Look),
Clean and Place (Clean), Heat and Place (Heat), Cool and Place (Cool), and
Pick Two and Place (Pick2). Success requires completing the full goal within
$50$ turns. We report the success rate for each category and the pooled success
rate over evaluated episodes, so the overall score weights categories by their
number of evaluated tasks.

\paragraph{WebShop.}
WebShop \citep{yao2022webshop} asks the agent to find and purchase a product
that satisfies an instruction, including constraints such as attributes,
options, and price. We use the text interface and the small catalogue selected
by the inherited environment configuration. The agent searches the catalogue,
opens product pages, selects options, and purchases an item within $15$ turns.
The environment constructs goals from this catalogue: the wrapper reserves
goal indices $0$--$499$ for evaluation and samples training goals from the
remaining indices. This index partition is interpreted within each worker's
seeded goal ordering. We retain both the graded task score, which measures
how closely the purchase satisfies the instruction, and strict success, which
requires a task score of $1$. Both are reported on a $0$--$100$ scale.

\subsection{Baselines}
\label{app:baseline-details}

We instantiate Skill-SD, GRPO+OPSD, and SDAR in a common trainer. Each recipe
combines the GRPO objective with a token-level auxiliary loss, and uses the
same skill bank and teacher-input construction. Their differences can be
expressed through three choices: the distillation coefficient $\lambda$,
whether teacher-confidence weighting is enabled, and the OPD token objective.
Table~\ref{tab:app-baseline-configs} states these choices explicitly.

\paragraph{Token objectives and confidence weights.}
For a sampled response token $y_t$ and its interaction history $x_t$, define
\begin{equation}
\label{eq:opd-ell}
s_t=\log\pi_\theta(y_t\mid x_t),\qquad
q_t=\operatorname{sg}\!\left[\log\pi_\theta(y_t\mid x_t\oplus S(x_t))\right],
\qquad \ell_t=s_t-q_t.
\end{equation}
These log-probabilities follow Equation~\ref{eq:prelim-logratio}.
The teacher receives skill context $S(x_t)$. Its scores $q_t$ and the rollout
log-probabilities $s_t^{\mathrm{old}}$ are detached. The capped on-policy ratio is
\begin{equation}
\label{eq:app-rho-on}
\rho_{\mathrm{on},t}=
\min\!\left\{\exp\!\left(\min\{s_t-s_t^{\mathrm{old}},10\}\right),2\right\}.
\end{equation}
The token objectives used by the host recipes and loss analysis are
\begin{equation}
\label{eq:opd-loss}
\psi_t=
\begin{cases}
\rho_{\mathrm{on},t}\big(\exp(\min\{-\ell_t,20\})-1+\ell_t\big), & k_3,\\
\rho_{\mathrm{on},t}\ell_t, & \rho_{\mathrm d},\\
-\ell_t, & \text{SDAR-style }k_1.
\end{cases}
\end{equation}
The $k_3$ and $\rho_{\mathrm d}$ objectives retain the student-side gradient
through $\rho_{\mathrm{on},t}$. For fixed sampled tokens, with neither ratio nor
exponential cap active,
\begin{equation}
\label{eq:opd-token-derivatives}
\frac{\partial\psi_t^{k_3}}{\partial s_t}
=\rho_{\mathrm{on},t}\ell_t,\qquad
\frac{\partial\psi_t^{\rho_{\mathrm d}}}{\partial s_t}
=\rho_{\mathrm{on},t}(\ell_t+1).
\end{equation}
Thus $k_3$ has zero token derivative at matched teacher and student
log-probabilities, while $\rho_{\mathrm d}$ retains an offset.
The objectives can change the direction of the aggregate gradient, rather
than only its scale.
The loss and coefficient analyses examine the $\rho_{\mathrm d}$ log-ratio
objective as an alternative to the host objectives in
Table~\ref{tab:app-baseline-configs}, which remain unchanged. The $k_1$ term follows the SDAR and
GRPO+OPSD implementations. With fixed sampled tokens and detached $q_t$,
its gradient is $-\nabla s_t$. Teacher scores affect this host variant
through confidence weights when those weights are enabled.
Confidence weights are $w_t=1$ when disabled and
$w_t=\operatorname{sg}[\sigma(-\beta\ell_t)]$ when enabled, where $\sigma$
is the sigmoid function. These weights act before gradient accumulation,
whereas the projection gate acts on the resulting gradient pair.

\paragraph{Loss aggregation.}
Let $m_t$ mask padding and inactive response tokens. The standard host path
forms a valid-token mean within each physical micro-batch $B^m$ and applies
its accumulation and distributed scaling $\alpha_m$:
\begin{equation}
\label{eq:app-auxiliary}
\Lopd(B)=\mathcal A_B[w_t\psi_t]=\sum_m\alpha_m
\frac{\sum_{t\in B^m}m_t w_t \psi_t}{\sum_{t\in B^m}m_t}.
\end{equation}
The sum is over nonempty micro-batches. For a complete mini-batch with
$M$ equally weighted micro-batches, $\alpha_m=1/M$. Thus a token's effective weight
also depends on its micro-batch's valid-token count. A global-token mean
instead uses one denominator $\sum_{t\in B}m_t$ for the complete optimizer
mini-batch. Appendix~\ref{app:optimization-sensitivity} defines and compares
micro-batch, global-token, and trajectory aggregation. The coefficient
$\lambda$ multiplies the aggregated loss separately from these weights.

Importance weighting of the host $k_1$ term would give
$\psi_t=-\rho_{\mathrm{on},t}\ell_t$, which differs in sign from
the $\rho_{\mathrm d}$ objective and is not used by the Base recipes in
Table~\ref{tab:app-baseline-configs}.

\begin{table}[!htbp]
\centering
\caption{Base recipes in the shared implementation. The OPD column
specifies the token term, including importance weighting.
Every Base configuration disables projection.}
\label{tab:app-baseline-configs}
\begingroup
\small
\setlength{\tabcolsep}{6pt}
\renewcommand{\arraystretch}{1.16}
\begin{tabular*}{\linewidth}{@{\extracolsep{\fill}}lccc@{}}
\toprule
Base recipe & $\lambda$ & Confidence weighting $w_t$ & OPD token term $\psi_t$ \\
\midrule
Skill-SD & $10^{-3}$ & $1$ (off) & Importance-weighted $k_3$ \\
GRPO+OPSD & $5\!\times\!10^{-3}$ & $1$ (off) & $k_1$, without $\rho_{\mathrm{on},t}$ \\
SDAR & $10^{-2}$ & $\operatorname{sg}[\sigma(-5\ell_t)]$ & $k_1$, without $\rho_{\mathrm{on},t}$ \\
\bottomrule
\end{tabular*}
\endgroup
\end{table}

\paragraph{Skill-SD.}
Our shared implementation of the Skill-SD recipe uses the importance-weighted
$k_3$ surrogate and a coefficient of $10^{-3}$. It retains the SDAR skill
interface. In this implementation, the student-to-rollout importance
ratio is capped at $2$, and the exponential argument in the $k_3$ term is
upper-clamped at $20$ for numerical stability. These safeguards are applied
identically to the Base and $+$ CORD configurations. The auxiliary loss uses Equation~\ref{eq:app-auxiliary} without
teacher-confidence weighting.

\paragraph{GRPO+OPSD.}
GRPO+OPSD uses the plain teacher-minus-student log-probability term and uniform
token weights. Its coefficient in our unified implementation is
$\lambda=0.005$. This matches the effective scale of the SDAR implementation's
GRPO+OPSD configuration: the latter sets the outer coefficient to $0.01$ and
the sigmoid sharpness to zero, yielding a constant weight
$\sigma(0)=0.5$. Our gate-off branch uses $w_t=1$, so absorbing this constant
into $\lambda$ preserves the auxiliary loss and its gradient. The numerical
difference in coefficients therefore reflects the gate convention rather than
a separate hyperparameter search.

\paragraph{SDAR.}
SDAR uses the same $k_1$ token term with $\lambda=0.01$ and $\beta=5$.
Its detached sigmoid weight assigns greater weight to sampled tokens for which
the skill-conditioned teacher has a higher log-probability than the student.
This confidence weighting acts before differentiation. CORD subsequently
acts on the accumulated auxiliary gradient, using its alignment with the
GRPO actor gradient to decide whether to project. This separates token
weighting from the subsequent choice of gradient combination.

\paragraph{Additional baselines.}
Our GRPO runs disable the distillation contribution. RLSD uses the
skill-conditioned self-teacher with coefficient $0.5$, a 100-step warmdown,
and clipping parameter $0.2$. We evaluate its final policies after 150 updates
in four passes using the ordered ALFWorld assignments
and the full WebShop task set described in
Appendix~\ref{app:implementation-evaluation}. The reference Vanilla,
Skill-Prompt, and OPSD entries in Table~\ref{tab:main} are reported values from
the SDAR paper.

\paragraph{Fixed skill bank and teacher context.}
\label{app:fixed-skill-bank}
Following SDAR, we use the predefined SkillBank adopted from SkillRL
\citep{lu2026sdar,xia2026skillrl}. The bank contains general strategies and
task-specific skills distilled from successful experience and failure
analysis. We keep its contents fixed throughout training. Thus, we use
SkillRL's bank as a source of guidance without its recursive bank-evolution
procedure. In ALFWorld, the household task determines the task-specific
skills. In WebShop, instruction keywords select relevant skills alongside
the general strategies. Table~\ref{tab:skill-bank-examples} illustrates
representative content.

The teacher is a skill-conditioned view of the current policy, not a
separately trained model. It scores the student's generated tokens with
gradients disabled. The student receives no additional skill context during
rollout or evaluation. Each Base--CORD pair shares the bank, skill-selection
rule, and auxiliary coefficient, so its comparison isolates the
gradient-combination rule. CORD does not prescribe how the skill bank is selected or constructed.

\begin{table}[!htbp]
\centering
\caption{Representative content from the fixed skill bank, paraphrased for
brevity. These entries guide the teacher without entering the student's
rollout or evaluation input.}
\label{tab:skill-bank-examples}
\begingroup
\small
\setlength{\tabcolsep}{5pt}
\renewcommand{\arraystretch}{1.16}
\begin{tabularx}{\linewidth}{@{}p{0.24\linewidth}X@{}}
\toprule
Skill context & Guidance \\
\midrule
ALFWorld: cleaning & Take the object, clean it at a sink, move to the target location, and place it there. \\
WebShop: apparel & Select the required size and color, then check the price of the chosen variant before purchasing. \\
WebShop: search recovery & After two or three unproductive variant changes, return to search instead of continuing the same loop. \\
\bottomrule
\end{tabularx}
\endgroup
\end{table}

\paragraph{CORD configuration.}
The main CORD configuration protects the GRPO actor gradient,
uses the Euclidean inner product, and projects once per complete mini-batch
only when the two gradients conflict. The coefficient is constant, and the
projection reference is the current protected gradient without EMA smoothing.

\subsection{Implementation and Evaluation Details}
\label{app:implementation-evaluation}

\paragraph{Training configuration.}
Table~\ref{tab:app-training-config} summarizes the common settings. We use
Qwen2.5-3B-Instruct for the comparisons currently reported in
Table~\ref{tab:main}, with FSDP \citep{zhao2023fsdp} for model training and vLLM \citep{kwon2023pagedattention} for rollout generation. Each iteration samples $16$ task groups with eight trajectories per group. Group members share
the initial task and environment state, while their actions are sampled.
The resulting trajectories are converted into step-level prompt--response
records for optimization. Accordingly, the PPO mini-batch sizes below count
these records rather than complete multi-turn episodes.

\begin{table}[!htbp]
\centering
\caption{Common training settings for the three base recipes and their CORD
variants. Task groups and trajectories are counted before trajectories are
expanded into step-level training records. Learning rates are constant after
the indicated warm-up.}
\label{tab:app-training-config}
\begingroup
\small
\setlength{\tabcolsep}{6pt}
\renewcommand{\arraystretch}{1.1}
\begin{tabular*}{\linewidth}{@{\extracolsep{\fill}}lcc@{}}
\toprule
Setting & ALFWorld & WebShop \\
\midrule
Training iterations & $150$ & $150$ \\
Task groups per rollout & $16$ & $16$ \\
Trajectories per group & $8$ & $8$ \\
Maximum prompt tokens & $2{,}048$ & $4{,}096$ \\
Maximum response tokens per turn & $512$ & $512$ \\
History retained (previous turns) & $2$ & $2$ \\
PPO mini-batch size (step records) & $256$ & $64$ \\
PPO epochs per rollout batch & $1$ & $1$ \\
Learning rate & $10^{-6}$ & $10^{-6}$ \\
Learning-rate warm-up fraction & $0$ & $0$ \\
Reference-policy KL coefficient & $0.01$ & $0.01$ \\
Invalid-action penalty & $0.1$ & $0.1$ \\
Training rollout temperature & $1.0$ & $1.0$ \\
\bottomrule
\end{tabular*}
\endgroup
\end{table}

The group-relative advantage in Section~\ref{sec:preliminaries} uses the
host's response-record returns, with group statistics pooled across the
records from all attempts at the same task. Each record's normalized
advantage is broadcast over its valid response tokens.
The policy-loss implementation additionally uses dual clipping for negative
advantages: if $c_t=-\min(r_t\widehat A_t,
\operatorname{clip}(r_t,1-\varepsilon,1+\varepsilon)\widehat A_t)$
is the token loss in Equation~\ref{eq:prelim-grpo}, it replaces $c_t$ by
$\min(c_t,-3\widehat A_t)$ when $\widehat A_t<0$.
The full minimized actor loss is
$\Lgrpo=\mathcal L_{\mathrm{pg}}+\beta_{\mathrm{ref}}\mathcal K_{\mathrm{ref}}
-\beta_{\mathrm H}\mathcal H$, where $\mathcal L_{\mathrm{pg}}$ includes this
safeguard and the regularization terms use the same valid-token aggregation.
The frozen rollout probabilities and advantages carry no gradient.

We optimize with AdamW, using momentum coefficients $(0.9,0.999)$, weight
decay $0.01$, a PPO clip ratio of $0.2$, and gradient-norm clipping at $1$.
The frozen initial policy supplies the reference for the low-variance KL
regularizer, with the same coefficient across recipes. The GRPO leg accumulated
by CORD differentiates the complete host actor loss: the clipped policy
surrogate plus reference-policy KL, and the entropy term when enabled.
Protection therefore refers to this actor surrogate. The skill-conditioned
auxiliary loss is computed separately.
The interactive environments assign reward $10$ for completing the goal and
$0$ otherwise.
Invalid actions receive the additional penalties in
Table~\ref{tab:app-training-config}. The graded WebShop task score is retained
as an evaluation metric even though its training reward uses strict success.

\paragraph{Multi-turn inputs and gradient accumulation.}
Following \emph{verl-agent}, each decision uses the current observation,
the task description, and a bounded history of previous interactions.
ALFWorld and WebShop retain the two most recent turns. Losses are computed on
the generated responses. Environment observations and padding do not provide
target tokens. We accumulate the GRPO actor and auxiliary gradients over
the full PPO mini-batch with the selected aggregation weights before computing
their inner product and applying the projection. In the fixed-coefficient
implementation, the auxiliary backward already includes $\lambda$, so its
buffer stores $\lambda\Go$.
After distributed gradient reduction and accumulation, a packed reduction
supplies inner products and norms across parameter shards. The reported
CORD configurations use Euclidean statistics. Numerical
protection skips projection when the squared reference norm is at most
$\epsilon$ and otherwise divides by $\norm{\bm{a}}^2+\epsilon$, with
$\epsilon=10^{-12}$. The ideal projection identities in
Proposition~\ref{prop:firstorder} omit this numerical adjustment.
Writing $A=\norm{\bm{a}}^2$ and $c=\inner{\bm{a}}{\bm{b}}$, the
regularized combined gradient $\bm{g}_{\epsilon}$ satisfies
$\inner{\bm{a}}{\bm{g}_{\epsilon}}=A+c\epsilon/(A+\epsilon)$
before floating-point error. For $c<0$, this is below $A$ and need not be
positive. Skipping projection instead gives $A+c$.
The combined gradient is clipped before the AdamW optimizer step. Task sampling and data order use recorded seeds, and the
rollout implementation separately exposes the generation seed so that sampled
decoding can be reproduced with a fixed worker configuration.

\paragraph{PCGrad two-objective adaptation.}
\label{app:pcgrad-adaptation}
We apply PCGrad \citep{yu2020gradient} to the host's fully accumulated actor gradient
$\bm{u}$ and $\lambda$-weighted distillation gradient $\bm{v}$.
When they conflict, each is projected against the other's original gradient,
and the two projected gradients are summed. Aligned pairs are added unchanged.
Each projection adds $\epsilon=10^{-12}$ to its squared-reference-norm
denominator and is skipped when that squared norm is at most $\epsilon$.
The runs in Table~\ref{tab:webshop-projection-full500} use $k_3$ and the
final-evaluation protocol below.

\paragraph{From the combined gradient to the parameter step.}
\label{app:actual-update}
At an anchor $\theta$, fix the sampled inputs and responses, masks, rollout
and reference-policy probabilities,
advantages, teacher scores, detached confidence values, aggregation weights,
the stochastic realization, and $\lambda$. These define a local surrogate $F_a$ with
$\bm{a}=\nabla F_a(\theta)$.
For the actual optimizer displacement $\Delta$, write
$p=-\bm{a}^\top\Delta$.

If $F_a$ has an $L$-Lipschitz gradient on a
neighborhood of the segment from $\theta$ to $\theta+\Delta$, then
\begin{equation}
\label{eq:actual-step-bound}
F_a(\theta+\Delta)-F_a(\theta)
\le -p+\frac{L}{2}\|\Delta\|^2.
\end{equation}
Indeed, the fundamental theorem of calculus expresses the remainder as
$\int_0^1[\nabla F_a(\theta+t\Delta)-\bm{a}]^\top\Delta\,dt$.
Cauchy--Schwarz and the Lipschitz condition bound its absolute value by
$L\|\Delta\|^2/2$.
For ideal SGD, $\Delta=-\eta\gfinal$, strict descent follows when
$L>0$, $\bm{a}^\top\gfinal>0$, and
\begin{equation}
\label{eq:ideal-step-sufficient}
0<\eta<\frac{2\bm{a}^\top\gfinal}{L\|\gfinal\|^2}.
\end{equation}
This is a sufficient local condition, without an estimate of $L$.
PPO or ratio-cap kinks can invalidate the smoothness assumption.
Scalar gradient clipping preserves the sign of $\bm{a}^\top\gfinal$ but
changes the effective step size. AdamW's momentum, coordinate scaling, and
weight decay change $\Delta$, so the Euclidean projection identity alone
does not establish $p>0$. These statements concern the fixed surrogate and imply neither monotonic
return improvement nor training convergence.

\paragraph{Fixed validation and final-evaluation tasks.}
We fix the validation tasks so that training progress is measured on the
same problems at each evaluation. In the inherited ALFWorld protocol, we
preserve an ordered list of $128$ episode assignments drawn from the
in-distribution evaluation pool. The list contains $91$ distinct games.
Retaining repeated assignments preserves their evaluation weights.
Analyses identified as full-game evaluations use all $140$ in-distribution
ALFWorld games with fixed assignments.
For WebShop, final evaluation uses $500$ predetermined task assignments.
Each comparison uses the same tasks, ordering, and decoding settings for
Base and CORD, separating changes in the policy from changes in task
composition.

\paragraph{Monitoring and final evaluation.}
We separate training-time monitoring from the measurements reported as final
performance. When enabled, monitoring uses greedy decoding on a fixed set
every five training iterations. Final evaluation uses the policy after
150 updates, without selecting the checkpoint with the highest intermediate
validation score. This follows
the \emph{verl-agent} authors' clarification that final scores are obtained
by repeatedly evaluating the step-$150$ checkpoint and averaging the
results.\footnote{\href{https://github.com/langfengQ/verl-agent/issues/246\#issuecomment-4641967418}{\emph{verl-agent}, issue \#246, author clarification (June 7, 2026).}}

The final-evaluation protocol performs four evaluation passes, generating one
trajectory per task in each pass. Sampled evaluation uses temperature $0.4$
and top-$p=1$, without top-$k$ truncation. If $s_r$ is a reported metric in
pass $r$, we report
\begin{equation}
\label{eq:app-eval-spread}
\bar{s}=\frac{1}{K}\sum_{r=1}^{K}s_r,
\qquad
\sigma_{\mathrm{eval}}=
\sqrt{\frac{1}{K}\sum_{r=1}^{K}(s_r-\bar{s})^2},
\qquad K=4.
\end{equation}
Thus, $\bar{s}\pm\sigma_{\mathrm{eval}}$ reports the mean and population standard
deviation across four evaluation passes of the final policy. It does not
estimate variation across training seeds, and averaging four passes does not change the metric into a
best-of-four success rate. The final-evaluation comparisons use one training seed.

\paragraph{Relation to the reported SDAR scores.}
SDAR reports WebShop results on $128$ fixed validation tasks
\citep{lu2026sdar}. The paper does not specify the number of independent
final-evaluation passes or report variability across such passes.
We use $500$ fixed assignments and report the mean and population standard
deviation over four independently sampled passes of the final policy.
These protocol differences limit direct comparison with the published
scores, identified separately in Table~\ref{tab:main}. Our Base--CORD
comparisons share task coverage and evaluation settings.

\FloatBarrier
\subsection{Additional Experiments Across Model Scales}
\label{app:backbone-7b}

Table~\ref{tab:backbone-7b} provides supplementary results with
Qwen2.5-7B-Instruct for GRPO, RLSD, and the Skill-SD and SDAR host pairs.
On ALFWorld, CORD raises Skill-SD success from $69.5\%$ to $77.0\%$ and
SDAR from $68.9\%$ to $76.8\%$, with displayed-mean gains of $7.5$ and
$7.9$ percentage points. WebShop outcomes depend on the host: SDAR gains
$4.7$ score points and $2.9$ success percentage points, whereas Skill-SD
loses $2.9$ score points and $1.8$ success percentage points. These
single-seed 7B results supplement the 3B comparison and show that the
benefit depends on both task and host.

\begin{table}[H]
\centering
\caption{\textbf{Final results on Qwen2.5-7B-Instruct.}
ALFWorld success and WebShop score/success are on a 0--100 scale.
Results use one training seed and four evaluation passes at temperature $0.4$
after update 150. ALFWorld uses 128 fixed assignments and WebShop uses 500 fixed tasks.
$\pm$ denotes population standard deviation across the four decoding passes. CORD protects RL with conflict-only projection.
$\Delta=\mathrm{CORD}-\mathrm{Base}$ uses displayed means, in percentage points
for success and score points for score. Red/green indicate increases/decreases.
Bold/underline mark the largest/second-largest method means.}
\label{tab:backbone-7b}
\begingroup
\small
\setlength{\tabcolsep}{1.4pt}
\renewcommand{\arraystretch}{1.22}
\begin{tabularx}{\linewidth}{@{}l*{6}{>{\centering\arraybackslash}X}>{\columncolor{verylightgray}\centering\arraybackslash}X>{\centering\arraybackslash}X>{\columncolor{verylightgray}\centering\arraybackslash}X@{}}
\toprule
& \multicolumn{7}{c}{ALFWorld} & \multicolumn{2}{c}{WebShop}\\
\cmidrule(lr){2-8}\cmidrule(l){9-10}
Method & Pick & Look & Clean & Heat & Cool & Pick2 & \textbf{Avg.} & Score & \shortstack{Success\\Rate}\\
\midrule

GRPO & $95.4{\scriptstyle\pm 3.1}$ & $64.6{\scriptstyle\pm 3.6}$ & $\mathbf{83.0}{\scriptstyle\pm 5.3}$ & $71.4{\scriptstyle\pm 5.1}$ & $45.2{\scriptstyle\pm 5.3}$ & $65.4{\scriptstyle\pm 2.7}$ & $72.9{\scriptstyle\pm 1.2}$ & $80.1{\scriptstyle\pm 0.8}$ & $69.3{\scriptstyle\pm 1.1}$\\
RLSD & $95.4{\scriptstyle\pm 3.1}$ & $72.9{\scriptstyle\pm 9.1}$ & $67.0{\scriptstyle\pm 3.9}$ & $42.9{\scriptstyle\pm 5.1}$ & $32.1{\scriptstyle\pm 6.2}$ & $69.2{\scriptstyle\pm 3.8}$ & $65.6{\scriptstyle\pm 1.8}$ & $\mathbf{84.4}{\scriptstyle\pm 0.3}$ & $\underline{71.0}{\scriptstyle\pm 0.6}$\\
\midrule
Skill-SD & $\underline{96.3}{\scriptstyle\pm 2.6}$ & $68.8{\scriptstyle\pm 9.1}$ & $\underline{80.4}{\scriptstyle\pm 4.0}$ & $\mathbf{75.0}{\scriptstyle\pm 6.2}$ & $\underline{48.8}{\scriptstyle\pm 3.9}$ & $44.2{\scriptstyle\pm 8.4}$ & $69.5{\scriptstyle\pm 3.3}$ & $\underline{84.2}{\scriptstyle\pm 0.5}$ & $\mathbf{71.6}{\scriptstyle\pm 0.7}$\\
\quad\textbf{+ CORD} & $\mathbf{99.1}{\scriptstyle\pm 1.6}$ & $\underline{83.3}{\scriptstyle\pm 5.9}$ & $76.8{\scriptstyle\pm 4.0}$ & $\underline{73.2}{\scriptstyle\pm 5.9}$ & $\mathbf{51.2}{\scriptstyle\pm 3.9}$ & $74.0{\scriptstyle\pm 3.2}$ & $\mathbf{77.0}{\scriptstyle\pm 1.3}$ & $81.3{\scriptstyle\pm 1.3}$ & $69.8{\scriptstyle\pm 1.5}$\\
\quad$\Delta$ & \corddeltaup{2.8} & \corddeltaup{14.5} & \corddeltadown{3.6} & \corddeltadown{1.8} & \corddeltaup{2.4} & \corddeltaup{29.8} & \corddeltaup{7.5} & \corddeltadown{2.9} & \corddeltadown{1.8}\\
\midrule
SDAR & $90.7{\scriptstyle\pm 5.6}$ & $75.0{\scriptstyle\pm 0.0}$ & $59.8{\scriptstyle\pm 4.6}$ & $58.9{\scriptstyle\pm 11.7}$ & $45.2{\scriptstyle\pm 5.3}$ & $\mathbf{77.9}{\scriptstyle\pm 5.7}$ & $68.9{\scriptstyle\pm 3.5}$ & $79.0{\scriptstyle\pm 1.0}$ & $65.8{\scriptstyle\pm 0.4}$\\
\quad\textbf{+ CORD} & $93.5{\scriptstyle\pm 4.8}$ & $\mathbf{87.5}{\scriptstyle\pm 9.3}$ & $\underline{80.4}{\scriptstyle\pm 3.1}$ & $71.4{\scriptstyle\pm 5.1}$ & $\underline{48.8}{\scriptstyle\pm 3.9}$ & $\underline{76.0}{\scriptstyle\pm 6.9}$ & $\underline{76.8}{\scriptstyle\pm 3.5}$ & $83.7{\scriptstyle\pm 0.5}$ & $68.7{\scriptstyle\pm 0.4}$\\
\quad$\Delta$ & \corddeltaup{2.8} & \corddeltaup{12.5} & \corddeltaup{20.6} & \corddeltaup{12.5} & \corddeltaup{3.6} & \corddeltadown{1.9} & \corddeltaup{7.9} & \corddeltaup{4.7} & \corddeltaup{2.9}\\
\bottomrule
\end{tabularx}
\endgroup
\end{table}

We omit GRPO+OPSD from this 7B comparison because, for fixed sampled
tokens, its $k_1$ auxiliary term uses constant token weights and a detached
teacher. Its auxiliary gradient therefore does not vary with teacher
log-probabilities and cannot express teacher--student mismatch in this
implementation. SDAR's detached confidence weights still depend on this
mismatch.

\section{Supplementary Analyses}
\label{app:supplementary-analyses}

\subsection{Motivation Diagnostics}
\label{app:gradient-diagnostics}

\paragraph{Gradient diagnostics.}
Figure~\ref{fig:motivation-triptych} uses 2,741 WebShop/Qwen3-1.7B
mini-batch gradient records collected with projection disabled and
$\lambda=0.001$. Negative RL--distillation cosine occurs in 1,322 records
(48.2\%). The left panel shows the 5th, 25th, 50th, 75th, and 95th
percentiles of this subset, using linear interpolation. It compares the
original cosines with those after analytic conflict-only projection of the
same pairs before the optimizer. Aligned records are omitted from this panel.
The middle panel reweights distillation to $\lambda=0.005$ and compares
DirectSum with conflict-only RL protection. Negative and nonnegative
cosines define the conflict and aligned groups. RL-surrogate progress is
normalized by the RL-gradient-alone value. The plotted coordinates are
recovered from the archived vector figure, with tails clipped to the
displayed range.

\paragraph{Training comparison.}
The right panel uses WebShop/Qwen2.5-3B-Instruct, $k_3$, $\lambda=0.001$,
and conflict-only distillation protection. It shows success on on-policy
training trajectories over iterations 75--148 after 21-step smoothing,
with limits $0.35$--$0.65$. Curves average training seeds 0 and 1, and bands
span the two seed curves. Two unprojected seed-0 runs are averaged before
combining seeds. GRPO-only is trained separately at seeds 0 and 1.

\subsection{Token Objectives and Mixing Coefficients}
\label{app:loss-coefficient-analysis}

Table~\ref{tab:loss-type-success} reports the category-level results for the
token-objective comparison in RQ2. Figure~\ref{fig:lambda-score}
adds the WebShop task score to the success-rate curves in
Figure~\ref{fig:lambda-success}. These experiments use RL protection and
conflict-only projection. Each point summarizes four evaluations of the
step-150 checkpoint at temperature $0.4$. The ALFWorld comparison uses the
full-game protocol described in Appendix~\ref{app:implementation-evaluation}.

\begin{figure}[!htbp]
\centering
\includegraphics[width=0.55\linewidth]{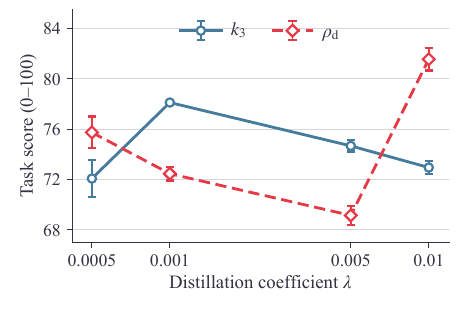}
\caption{\textbf{Effect of the mixing coefficient on WebShop score.}
The checkpoints and evaluation protocol are the same as in
Figure~\ref{fig:lambda-success}. Error bars show standard deviation across decoding passes.}
\label{fig:lambda-score}
\end{figure}

\FloatBarrier
\subsection{Task-level Projection Intervention}
\label{app:task-projection-intervention}

\paragraph{Measurement protocol.}
Figure~\ref{fig:task-projection-intervention} compares the two protection
directions under local coefficient reweighting.
For this diagnostic, $\Gr$ and $\Go$ denote the two objective gradients
aggregated over eight complete attempts at one task, before the outer
coefficient $\lambda$ is applied. We measure
\[
100\,\frac{\norm{\dcord(\lambda)-\dsum(\lambda)}}{\norm{\dsum(\lambda)}},
\qquad
\dsum(\lambda)=-(\bm{u}+\bm{v}(\lambda)),\quad
\dcord(\lambda)=-\gfinal.
\]
Here $\bm{u}=\Gr$ and $\bm{v}(\lambda)=\lambda\Go$, matching the weighting
in Equation~\ref{eq:oneliner}. The bold $\bm{d}$ symbols denote descent
directions, hence the minus signs. The ratio requires a nonzero direct-sum
direction. Within each environment, both columns
reuse the same checkpoint, trajectories, and unweighted gradients, with
$\bm{v}(\lambda)=(\lambda/\lambda_0)\bm{v}(\lambda_0)$.
ALFWorld uses checkpoint 90, confidence-weighted $k_1$ with $\beta=5$,
and native $\lambda_0=0.01$. WebShop uses checkpoint 75, $k_3$,
and $\lambda_0=0.001$. We use the eight ALFWorld source tasks whose direct
gradients were available, out of 32 source tasks, and all 16 WebShop source
tasks. The ALFWorld subset follows a failed task-to-block gradient
reconstruction check. Task indices retain the same order in both columns.
These interventions compare changes to fixed gradient pairs, without
training a new policy at each coefficient.

\begin{figure}[H]
\centering
\includegraphics[width=\linewidth]{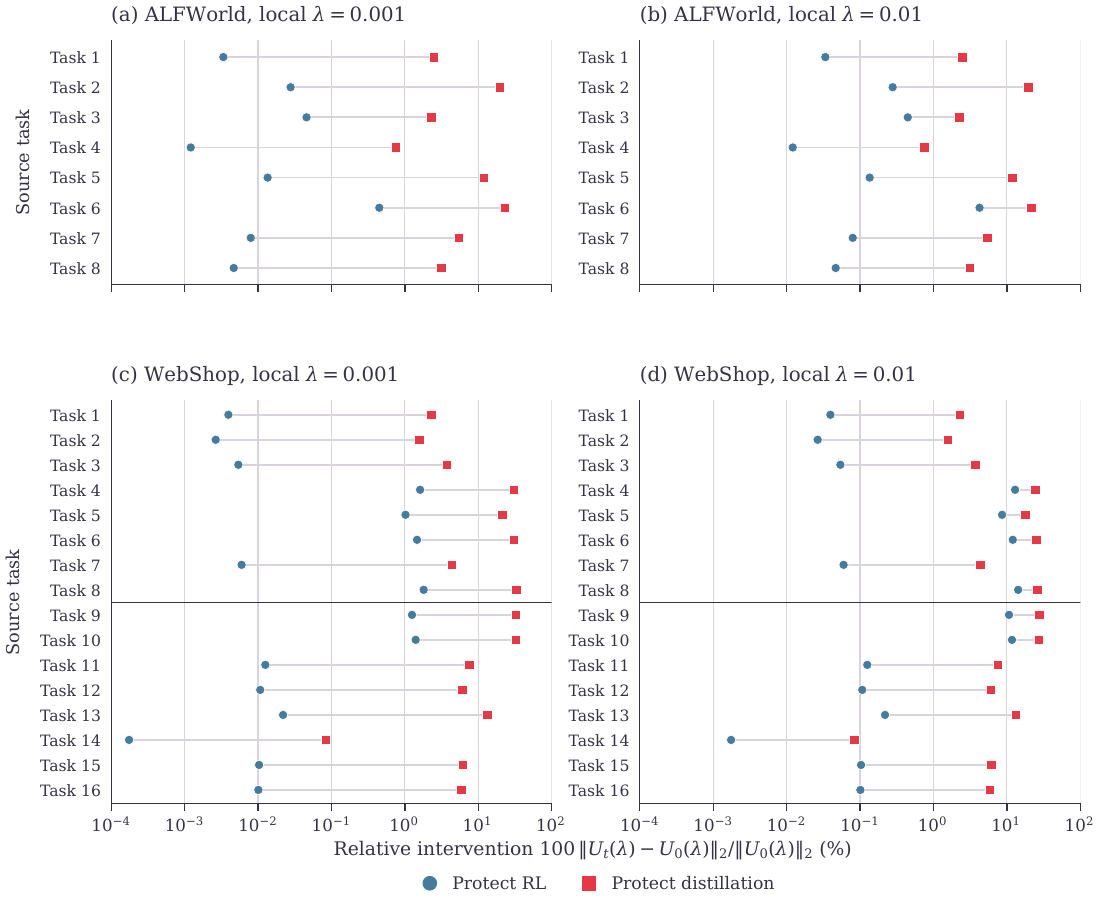}
\caption{\textbf{Task-level projection intervention under local coefficient reweighting.}
Each line connects the two protection directions for one source task under
always projection. Columns reuse the same gradients at different coefficients.
The horizontal axis reports the relative change in the combined update.
Its labels $U_0$ and $U_t$ denote $\dsum$ and $\dcord$, respectively.}
\label{fig:task-projection-intervention}
\end{figure}

\paragraph{Dependence on the coefficient.}
For fixed nonzero loss gradients $\Gr$ and $\Go$, write $\bm{v}=\lambda\Go$ with
$\lambda>0$, and let $\phi$ be the angle between $\Gr$ and $\Go$.
Under the always gate used here, the projection correction has norm $\lambda\lVert\Go\rVert|\cos\phi|$ when protecting
RL and $\lVert\Gr\rVert|\cos\phi|$ when protecting distillation, by
Equation~\ref{eq:gfinal}. Thus the second correction is independent of
$\lambda$ at this fixed gradient pair, whereas the kept distillation term still
scales with $\lambda$. These identities also hold on the conflicting branch
of the conflict-only gate, whose aligned branch has zero correction.

\subsection{Sensitivity to Optimization Choices}
\label{app:optimization-sensitivity}

Table~\ref{tab:optimization-sensitivity-aggregation} reports the endpoints of
the aggregation and mini-batch study in Section~\ref{sec:aggregation-behavior}.
Here we specify the continuation protocol and token weights, and examine
the corresponding uncertainty in Figure~\ref{fig:optimization-sensitivity}.

\paragraph{Continuation protocol.}
Every configuration starts from the same ALFWorld/Qwen2.5-3B checkpoint at
training step~90 and its AdamW state. We continue for five rounds with one
training seed, gated $k_1$, $\lambda=.01$, $\beta=5$, and conflict-only
projection. Each run uses 640 episodes, including the same initial 128
episodes. Four response rows per GPU on eight GPUs give $1/2/4$ accumulation
steps for optimizer mini-batch sizes $B=32/64/128$.
Evaluation uses 512 attempts over 91 in-distribution tasks at temperature
$0.4$, with success averaged across tasks. The 95\% task-bootstrap intervals
describe evaluation uncertainty for this seed. Paired contrasts are
descriptive. The figure labels \emph{Protect GRPO} and \emph{Protect OPD}
denote RL and distillation protection, respectively.

\paragraph{Aggregation weights.}
Let $S$ be an optimizer mini-batch containing $b$ valid response rows, where
$B$ is the configured mini-batch size. For objective $j$ and token $t$ in response row $i$, let $f_{it}^{(j)}$ include its token loss and any
confidence weight: in particular, $f_{it}^{(2)}=w_{it}\psi_{it}$ for OPD.
Then $F_j(S)=\sum_{(i,t)\in S}\omega_{it}f_{it}^{(j)}$ before projection.
The three rules assign token weights
\[
\omega_{it}=
\begin{cases}
1/(M T_m), & \text{microbatch averaging }(L_1),\\
1/T, & \text{global-token averaging }(L_2),\\
R/(N b L_{\tau(i)}), & \text{trajectory averaging }(L_3).
\end{cases}
\]
Here $T_m$ and $T$ count valid tokens in microbatch $m$ and $S$, and $M$
counts nonempty microbatches across GPUs. For trajectory averaging, $R$ and
$N$ count response rows and complete trajectories in the rollout.
Let $\tau(i)$ identify the complete trajectory containing row $i$, and let
$L_k$ count valid response tokens across all rows of trajectory $k$.
When $b=R$, the trajectory-averaged objective is
\[
F_{j,3}=\frac1N\sum_{k=1}^{N}\frac1{L_k}
\sum_{(i,t):\,\tau(i)=k} f_{it}^{(j)},
\]
so every complete trajectory has equal total weight. Mini-batches of response rows retain
these complete-trajectory length weights.

\begin{figure}[!htbp]
\centering
\includegraphics[width=\linewidth]{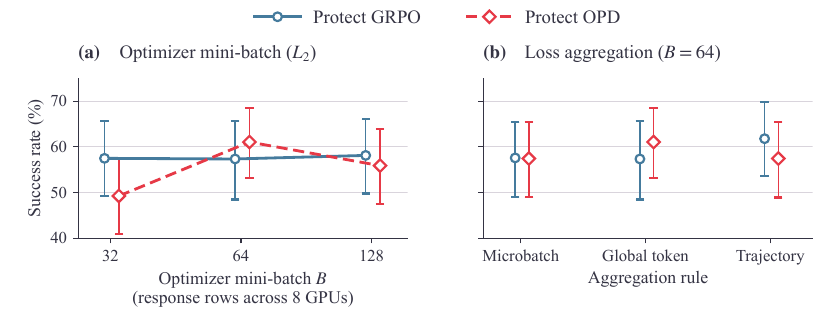}
\caption{\textbf{ALFWorld continuation results with task-bootstrap uncertainty.}
Panel (a) varies mini-batch size under global-token averaging ($L_2$),
while panel (b) varies aggregation at $B=64$.
Error bars show 95\% intervals over tasks under the protocol above.}
\label{fig:optimization-sensitivity}
\end{figure}

\paragraph{Optimizer mini-batch size.}
At fixed $L_2$, RL-protected success spans $0.78$ percentage points over
$B\in\{32,64,128\}$, whereas distillation protection peaks at $B=64$
(Figure~\ref{fig:optimization-sensitivity}a).
At $B=128$, distillation protection is $5.95$ percentage points below no
projection, with a paired 95\% interval of $[-11.13,-0.96]$.
Using $\bm{u}=\Gr$ and $\bm{v}=\lambda\Go$ as in Equation~\ref{eq:oneliner}, during conflict,
correction norms are $\lambda\|\Go\||\cos\phi|$ under RL protection and
$\|\Gr\||\cos\phi|$ under distillation protection
(Appendix~\ref{app:task-projection-intervention}), where $\phi$ is the angle between $\Gr$ and $\Go$. Small $\lambda$ thus
leaves the reverse correction unattenuated at fixed gradients.
Logged median ratios $\|\bm{v}\|/\|\bm{u}\|$ of $0.009$--$0.013$ describe the scale
regime in which this contrast occurs.
At fixed episode budget and learning rate $10^{-6}$, increasing $B$
reduces the number of updates from $616$ to $300$ to $150$ under
distillation protection. This sweep therefore changes both gradient
estimation and the number of optimizer updates.

\paragraph{Loss aggregation.}
At $B=64$, the directions nearly tie under microbatch averaging.
Distillation protection leads by $3.71$ points under global-token averaging,
and RL protection leads by $4.35$ under trajectory averaging
(Figure~\ref{fig:optimization-sensitivity}b). All three paired intervals
include zero, so this ordering remains descriptive.
Aggregation reweights gradients before projection, changing the conflict
test and protected objective. Global-token averaging ($L_2$) weights each
valid token equally, while trajectory averaging ($L_3$) gives each complete
trajectory equal total weight when all rollout rows are included.

\subsection{Training Dynamics}
\label{app:training-dynamics}

We examine how gradient combination affects rollout success, interaction
length, and policy statistics during training. Each Base/CORD pair below
shares its code revision, active training configuration, and training seed,
with conflict-only RL protection enabled only for CORD. The plots include
3B settings and an additional 7B training setting. All measurements describe
single-seed on-policy training trajectories and are separate from the
final evaluations.

\begin{figure}[t]
\centering
\includegraphics[width=\linewidth]{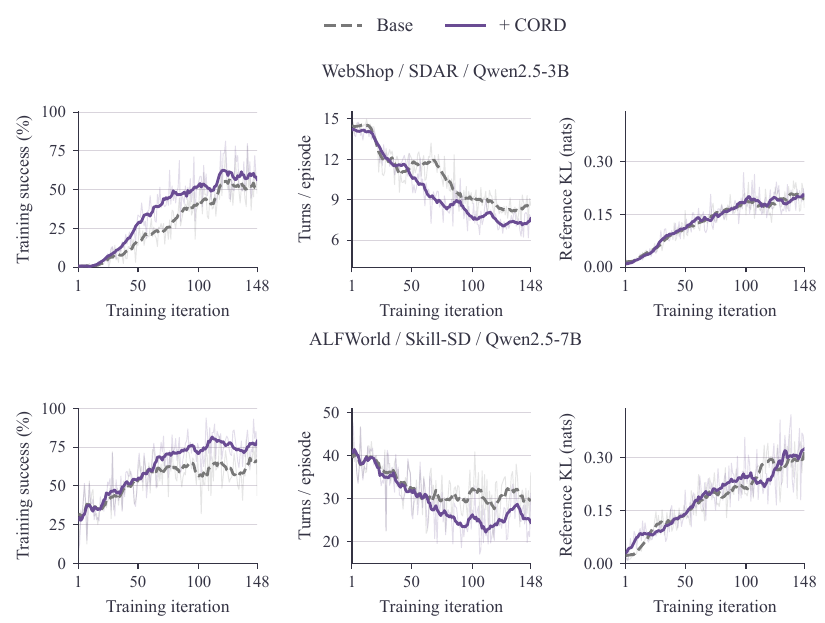}
\caption{\textbf{Training success, interaction length, and reference KL.}
Top: SDAR on WebShop with Qwen2.5-3B ($\lambda=0.01$).
Bottom: Skill-SD on ALFWorld with Qwen2.5-7B ($\lambda=0.001$).
Columns show rollout success, mean environment turns per episode, and the
logged sample-based KL estimate to the frozen reference policy.
Faint curves show raw observations, while darker curves show
11-iteration centered moving averages with shorter windows at the boundaries.
Each method has one training seed, shown over the common range of
iterations 1--148.}
\label{fig:reward-dynamics}
\end{figure}

\paragraph{Reward progress and interaction length.}
Both settings show higher late-training success and shorter episodes with
CORD (Figure~\ref{fig:reward-dynamics}). Using unsmoothed means over
iterations 101--148, WebShop success increases from $50.0\%$ to $57.1\%$
while episode length decreases from $8.58$ to $7.48$ turns. In the 7B
ALFWorld setting, success increases from $61.0\%$ to $76.0\%$ and episode
length decreases from $30.67$ to $25.34$ turns. Mean reference-KL estimates
are close within each pair: $0.188$ versus $0.189$ on WebShop and $0.274$
versus $0.271$ on ALFWorld. Thus, the shorter episodes accompany higher
training success at a similar logged reference-KL level.
Within these runs, the success and episode-length differences retain
their signs over the alternative windows 76--148 and 124--148.

\begin{figure}[t]
\centering
\includegraphics[width=\linewidth]{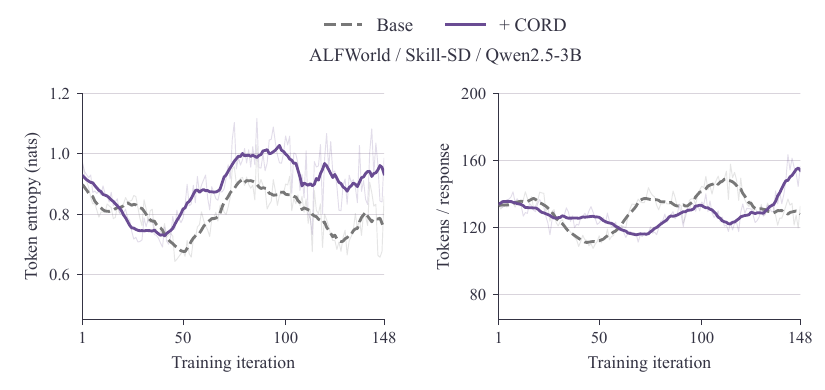}
\caption{\textbf{Token entropy and response length on ALFWorld.}
Skill-SD with Qwen2.5-3B at $\lambda=0.001$.
Left: token-mean policy entropy on the rollout responses,
measured before the actor update. Right: mean response-token count,
rather than tokens per complete episode. Each method uses one training seed.
The raw overlays, 11-iteration smoothing, and common range of iterations
1--148 follow Figure~\ref{fig:reward-dynamics}.}
\label{fig:entropy-dynamics}
\end{figure}

\paragraph{Token entropy and response length.}
CORD increases token entropy while slightly reducing mean response length
(Figure~\ref{fig:entropy-dynamics}). Over iterations 101--148, entropy
increases from $0.780$ to $0.925$ nats, while mean response length
decreases from $137.1$ to $133.3$ tokens. The entropy difference remains
positive in both alternative windows. These observations distinguish
changes in the token distribution from changes in response verbosity.
Entropy is measured on each policy's own visited states, so it does not
by itself measure exploration quality or establish a final-evaluation gain.

\FloatBarrier
\subsection{Computational cost}
\label{app:computational-cost}

Let $d$ be the number of trainable parameters and let a training iteration
contain $K$ PPO mini-batches, each partitioned into $M$ micro-batches.
CORD accumulates the two objective gradients over each complete PPO
mini-batch before applying the projection. Computing their inner products,
norms, and combined update requires $O(d)$ arithmetic per mini-batch,
or $O(Kd)$ per iteration. The additional gradient buffers require $O(d)$
storage, which is distributed across workers when gradients are sharded.
The scalar inner products and norms can be communicated in a packed
reduction whose message size does not grow with $d$.

For gradient evaluation, write $F$ for the cost of a student forward
computation and $B_1$, $B_2$ for the reverse-mode computations of the two
objectives on a fixed micro-batch. If both objectives use the same student
outputs, a shared-forward schedule has actor-update cost
\[
T_{\mathrm{CORD}}^{\mathrm{actor}}
= KM(F+B_1+B_2)+O(Kd)+T_{\mathrm{opt}},
\]
where $T_{\mathrm{opt}}$ includes gradient clipping and optimizer updates.
A fused-loss Base instead has cost
\[
T_{\mathrm{Base}}^{\mathrm{actor}}
= KM(F+B_{\mathrm{sum}})+T_{\mathrm{opt}},
\]
with $B_{\mathrm{sum}}$ denoting the reverse-mode computation of the
weighted loss sum. These expressions use equal-sized micro-batches to simplify
notation. For variable workloads, the corresponding costs are summed over
micro-batches. Parameter and gradient communication belong to the gradient
evaluation costs, separately from the packed reduction for projection.

Sharing the forward computation avoids evaluating identical student
outputs twice. It does not make two objective-gradient evaluations
equivalent to one backward pass on their summed loss: the latter supplies
only the combined gradient, which generally does not determine the inner
products needed by CORD. Both schedules have the same asymptotic order
when the reverse-mode costs are comparable, while their constant factors
depend on graph reuse, gradient-buffer handling, sharding, and hardware.
Forward sharing, fused vector operations, and selective diagnostic
collection are implementation opportunities rather than a fixed
wall-clock guarantee. Teacher scores already computed for the auxiliary
objective are reused.

\paragraph{Cost of a complete training iteration.}
Let $A$ denote the actor-update time of an ordinary fused-loss Base, and
let $S$ include rollout and environment interaction, teacher, reference,
and old-policy scoring, reward and advantage computation, and the
remaining work within a training iteration. For a fixed workload, if these
shared stages remain unchanged and the actor-update multiplier is
$r_{\mathrm{actor}}$, then
\[
R_{\mathrm{iter}}
=\frac{S+r_{\mathrm{actor}}A}{S+A}
=1+(r_{\mathrm{actor}}-1)f_{\mathrm{actor}},
\qquad
f_{\mathrm{actor}}=\frac{A}{S+A}.
\]
Thus, the actor-update multiplier is weighted by its share of Base
iteration time. This comparison excludes initialization, periodic
evaluation, checkpoint saving, and time outside the training iteration.

\paragraph{Component-based estimate.}
We evaluate an isolated shared-forward implementation on Qwen2.5-3B
with eight A100-40GB GPUs, FSDP, activation checkpointing, a global PPO
mini-batch size of 64, and a per-GPU micro-batch size of four.
We replay the first, middle, and last complete mini-batches saved from
one ALFWorld iteration. All replays use the same checkpoint, optimizer,
and random states, with fixed $\lambda=0.01$ and conflict-only projection
with GRPO protection. After two warmup calls, each setting
has seven timed repeats per mini-batch. Timings include the student
forward and backward computations, communication, gradient combination,
clipping, and AdamW, with GPU synchronization and the maximum elapsed
time across workers. The ratio of summed per-mini-batch median times is
$1.78\times$ relative to an ordinary Base that backpropagates the summed
loss once.

To estimate the iteration-level impact, we use two recorded WebShop
workload profiles with the same model, GPU type and count, and actor
configuration. Both profiles cover the same 131 iterations between
iterations 11 and 148 after excluding evaluation and checkpoint-saving
iterations. The projection-off profile already collects both objective
gradients for diagnostics. We therefore reconstruct ordinary Base actor
time by dividing the logged actor time by its separately measured
calibration ratio: $2.00$ for the diagnostic path and $2.01$ for the
original CORD path. Applying the shared-forward multiplier to these
reconstructed actor times gives Table~\ref{tab:iteration-cost-estimate}.
The estimated Base actor share is 25--27\%, yielding an estimated
iteration multiplier of $1.20$--$1.21\times$. Applying the same calculation
within the three consecutive windows 11--50, 51--100, and 101--148 gives
$1.17$--$1.22\times$.

\begin{table}[ht]
\centering
\caption{Component-based estimates for two WebShop workload profiles.
Each row holds the recorded workload fixed for both methods.
Shared-stage times are observed means, whereas actor times and iteration
ratios are estimates for the ordinary Base and the shared-forward CORD
implementation.}
\label{tab:iteration-cost-estimate}
\small
\setlength{\tabcolsep}{3pt}
\begin{tabular*}{\linewidth}{@{\extracolsep{\fill}}lrrrr@{}}
\toprule
Recorded workload &
\shortstack{Shared\\stages (s)} &
\shortstack{Base actor\\(s, est.)} &
\shortstack{CORD actor\\(s, est.)} &
\shortstack{Iteration\\ratio (est.)} \\
\midrule
Projection-off run & 150.2 & 54.8 & 97.2 & $1.21\times$ \\
CORD run           & 147.7 & 50.0 & 88.8 & $1.20\times$ \\
\bottomrule
\end{tabular*}
\end{table}

These estimates assume that the actor ratios measured on the ALFWorld
replay inputs transfer to the WebShop workloads and approximate the
logged actor stage, which also includes worker and data-transfer
overhead. The historical profiles and replay use different physical
machines of the same GPU type. The estimates describe the recorded
workloads with shared-stage costs held fixed and exclude policy-dependent
changes in subsequent trajectories. They are not measurements of a complete
training run with the optimized implementation.

\FloatBarrier
\subsection{Statistical Significance}
\label{app:statistical-significance}

The fixed-task, repeated-evaluation protocol in
Appendix~\ref{app:implementation-evaluation} separates decoding variability
from changes in task composition. We compare Base and CORD using two-sided
Welch tests on the four
evaluation-pass means for each checkpoint (Table~\ref{tab:app-significance}).
These tests summarize variation across decoding passes at one training seed.
Holm correction is applied jointly to all 27 ALFWorld and WebShop comparisons.

\begin{table}[H]
\centering
\caption{\textbf{Significance tests for Table~\ref{tab:main}.}
Raw two-sided Welch $p$-values are reported to three significant digits.
Bold $\dagger$ entries have Holm-adjusted $p<0.05$.}
\label{tab:app-significance}
\begingroup
\small
\setlength{\tabcolsep}{9pt}
\renewcommand{\arraystretch}{1.18}
\begin{tabular*}{\linewidth}{@{\extracolsep{\fill}}lccc@{}}
\toprule
Metric & Skill-SD & GRPO+OPSD & SDAR \\
\midrule
\multicolumn{4}{@{}l}{\textbf{ALFWorld}} \\
Pick & $4.18\!\times\!10^{-1}$ & $3.32\!\times\!10^{-1}$ & $1.25\!\times\!10^{-1}$ \\
Look & $6.86\!\times\!10^{-2}$ & $1.35\!\times\!10^{-1}$ & $9.62\!\times\!10^{-3}$ \\
Clean & $5.85\!\times\!10^{-1}$ & $5.16\!\times\!10^{-3}$ & $8.39\!\times\!10^{-3}$ \\
Heat & $7.67\!\times\!10^{-1}$ & $1.79\!\times\!10^{-2}$ & $8.01\!\times\!10^{-2}$ \\
Cool & $\mathbf{1.88\!\times\!10^{-4}}^{\dagger}$ & $9.15\!\times\!10^{-3}$ & $4.23\!\times\!10^{-2}$ \\
Pick2 & $1.00\!\times\!10^{-1}$ & $\mathbf{2.14\!\times\!10^{-4}}^{\dagger}$ & $2.33\!\times\!10^{-2}$ \\
\textbf{Average} & $3.16\!\times\!10^{-1}$ & $\mathbf{2.88\!\times\!10^{-4}}^{\dagger}$ & $1.18\!\times\!10^{-1}$ \\
\addlinespace[3pt]
\multicolumn{4}{@{}l}{\textbf{WebShop}} \\
Score & $\mathbf{9.09\!\times\!10^{-4}}^{\dagger}$ & $\mathbf{1.05\!\times\!10^{-3}}^{\dagger}$ & $1.28\!\times\!10^{-1}$ \\
Success rate & $\mathbf{2.20\!\times\!10^{-3}}^{\dagger}$ & $6.95\!\times\!10^{-1}$ & $2.84\!\times\!10^{-3}$ \\
\bottomrule
\end{tabular*}
\endgroup
\end{table}

\FloatBarrier
\begin{samepage}
\subsection{Additional Case Analyses}
\label{app:additional-cases}

These four cases extend Section~\ref{sec:case-study} using
Qwen2.5-3B-Instruct, $k_3$, step-150 checkpoints and greedy decoding,
with $\lambda=0.005$ for WebShop and $0.001$ for ALFWorld.
The unprojected variant adds the gradients directly, while both protected
variants use conflict-only projection. The ALFWorld tasks and decoding settings match,
but training used eight GPUs for distillation protection and sixteen for
the other checkpoints.
\par
\end{samepage}

\begin{samepage}
\paragraph{WebShop: repetition and attribute completion.}
In the rug case (Figure~\ref{fig:app-webshop-cases}a), all three policies buy
the same product and options with score $1.00$, but the policy with RL
protection makes eight repeated option clicks and uses 14 steps, compared
with six under
distillation protection. This reverses the ordering by total steps in
Figure~\ref{fig:case-study}. For the garment
(Figure~\ref{fig:app-webshop-cases}b), all three finish in five steps.
The policy with RL protection omits the requested size, scoring $0.80$
compared with $1.00$ for the other two. Thus, reaching a purchase promptly and completing its required
attributes are separate aspects of behavior.
\par
\end{samepage}

\begin{figure}[!htbp]
\centering
{\small\bfseries (a) Same purchase, different repetition\par}
\smallskip
\includegraphics[width=\linewidth]{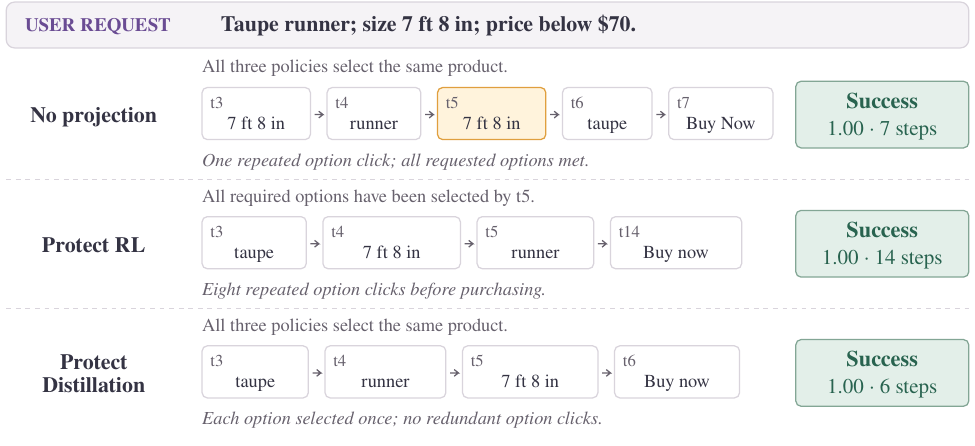}
\medskip
{\small\bfseries (b) Equal step counts, different attribute completion\par}
\smallskip
\includegraphics[width=\linewidth]{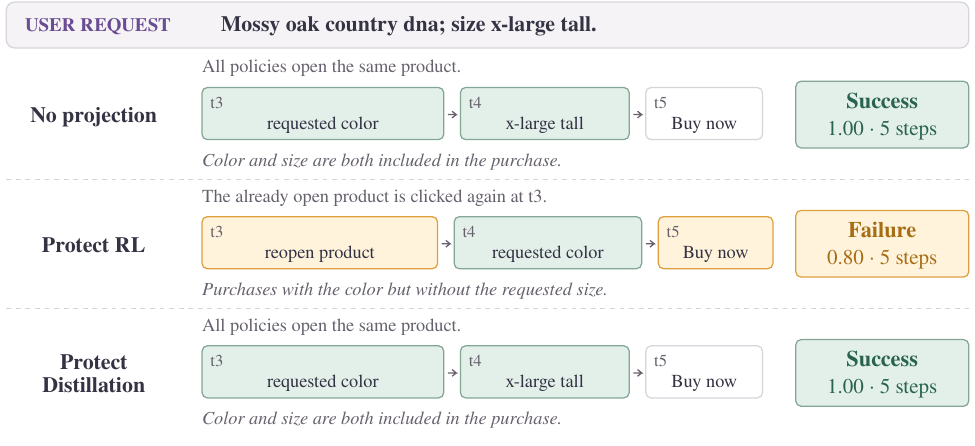}
\caption{\textbf{Additional WebShop trajectories.}
Trajectories retain the original step indices and use context lines to
summarize omitted actions.
Green marks requested attributes and amber marks repeated or incomplete
actions. Outcomes report the final score and total number of steps,
with score $1.00$ required for success.}
\label{fig:app-webshop-cases}
\end{figure}
\FloatBarrier

\begin{samepage}
\paragraph{ALFWorld: object state and completed subgoals.}
In Figure~\ref{fig:app-alfworld-cases}a, the unprojected policy repeatedly
cools eggs and never heats one within 50 steps. Both projected policies
finish with the same four-action heating-and-placement sequence, taking
eight steps under RL protection and sixteen under distillation protection.
The extra actions occur before this shared suffix, including five
ineffective actions under distillation protection.
Figure~\ref{fig:app-alfworld-cases}b reverses the length ordering.
With distillation protection, the agent keeps the first clock on the
dresser and completes the second delivery in 11 steps, compared with
25 steps after backtracking under RL protection. The unprojected policy attempts placement with the wrong
held object or at the wrong location and ends with only one delivered clock.
The shorter successful trajectory occurs under different protection
directions in these two cases.
\par
\end{samepage}

\begin{figure}[!htbp]
\centering
\includegraphics[width=\linewidth]{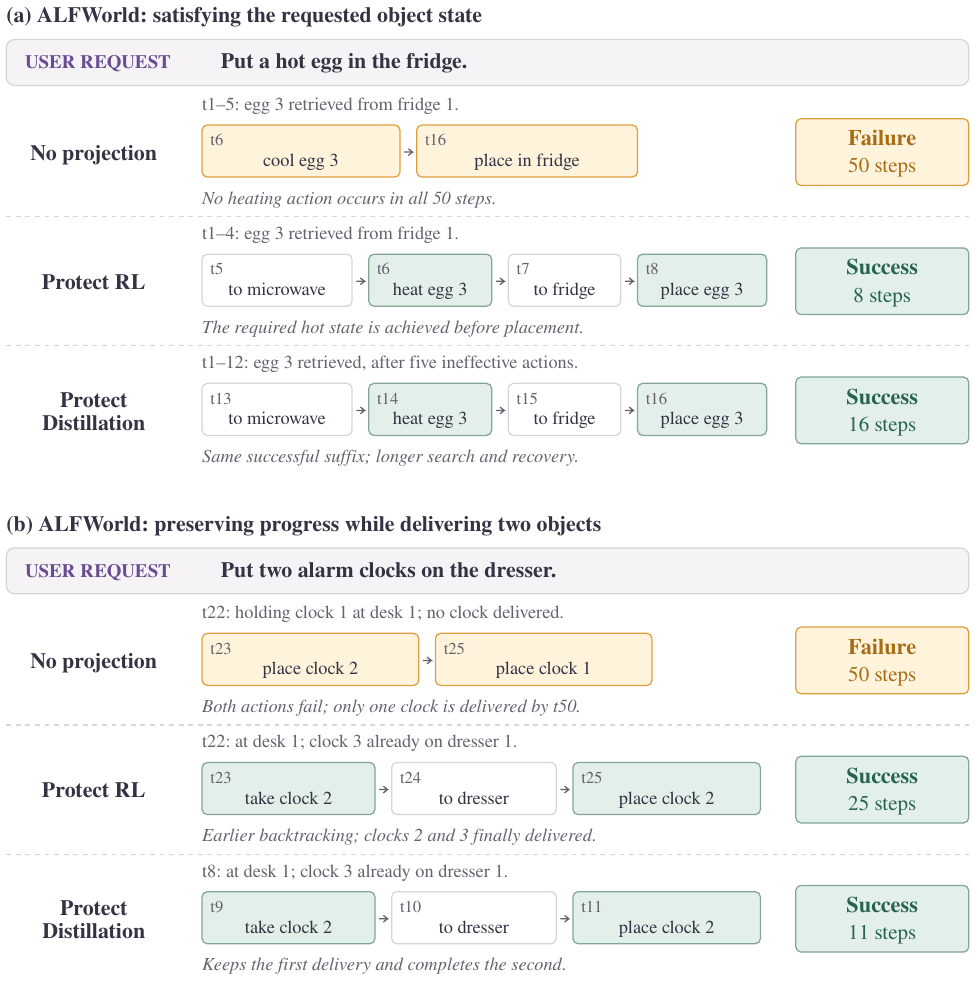}
\caption{\textbf{Additional ALFWorld trajectories.}
(a) Achieving the requested object state. (b) Completing two deliveries.
Action labels are abbreviated, with original step indices and context for
omitted intervals. ``Clock'' denotes an alarm clock and ``place'' denotes
the environment's move action. Green marks goal-relevant manipulation,
and amber marks incorrect or ineffective actions.}
\label{fig:app-alfworld-cases}
\end{figure}

\FloatBarrier
\begin{samepage}
\subsection{Token-Level Effects of Projection}
\label{app:token-case-study}

Figure~\ref{fig:token-case-trajectory} follows a complete 12-step WebShop
trajectory. From the Qwen2.5-3B step-75 checkpoint and identical AdamW
states, we apply one update per rule on the same conflicting batch
(OPD weight $\lambda=0.001$), then rescore tokens at their saved prefixes.
In the query, CORD with OPD protection retains more probability on the token
\texttt{he} in \texttt{heather} and the fit term \texttt{men} than DirectSum.
For \texttt{3}, it reduces suppression
at the first size selection (step~5), but reverses DirectSum's increase
when the choice repeats (step~7). This parameter-space projection produces
context-dependent token changes, including an adverse increase for the
unlisted \texttt{[next} token.
\par
\end{samepage}

\begin{figure}[!htbp]
\centering
\includegraphics[width=\linewidth]{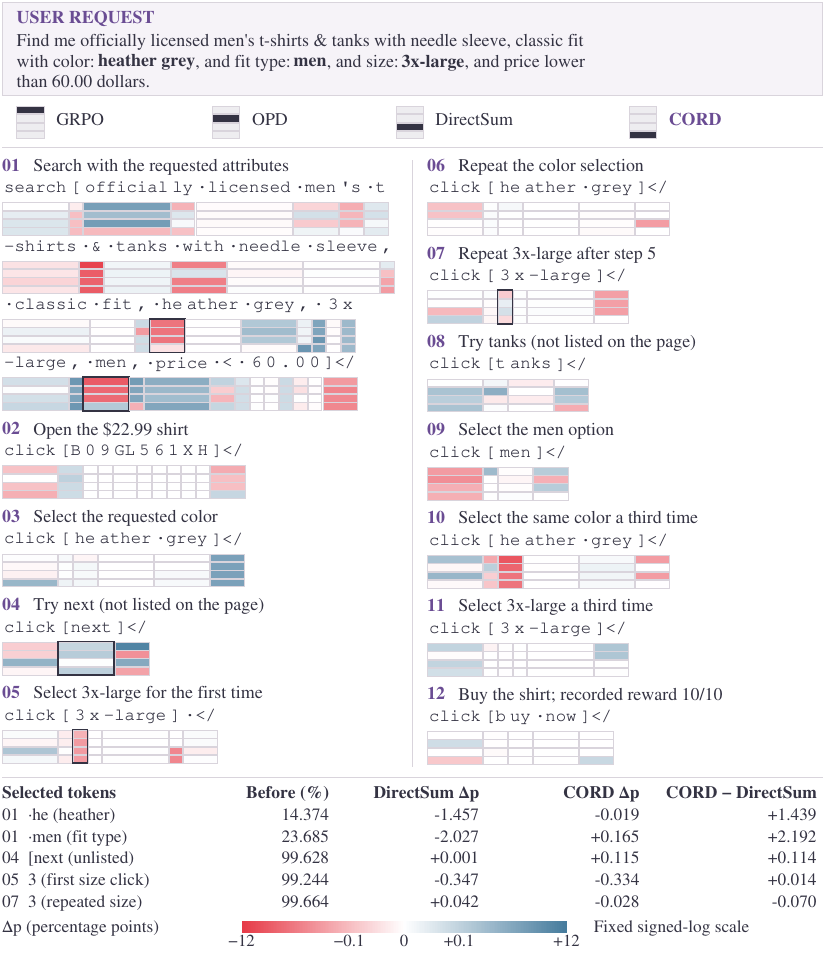}
\caption{\textbf{Token probability changes along a complete trajectory.}
Legend keys mark row positions, while outlines identify tabulated tokens,
not update quality. CORD protects OPD. GRPO protection matches DirectSum
at measured precision and is omitted. Dots denote spaces, and tag-boundary
pieces are retained. The recorded outcome predates these updates.}
\label{fig:token-case-trajectory}
\end{figure}

\end{document}